\documentclass[10pt]{article}
\usepackage[accepted]{tmlr}
\usepackage{amsmath,amsfonts,bm}

\def\eqref#1{equation~\ref{#1}}
\def\1{\bm{1}}

\DeclareMathAlphabet{\mathsfit}{\encodingdefault}{\sfdefault}{m}{sl}
\SetMathAlphabet{\mathsfit}{bold}{\encodingdefault}{\sfdefault}{bx}{n}

\newcommand{\E}{\mathbb{E}}

\newcommand{\R}{\mathbb{R}}

\DeclareMathOperator*{\argmin}{arg\,min}

\usepackage{hyperref}
\usepackage{url}

\usepackage{soul}
\usepackage{microtype}
\usepackage{graphicx}
\usepackage{subfigure}
\usepackage{booktabs} % for professional tables
\usepackage{makecell}
\usepackage{fancybox}

\usepackage{graphicx} 
\usepackage{amsmath}
\usepackage{booktabs}
\usepackage{algorithm}
\usepackage{algorithmic}
\usepackage{todonotes}
\usepackage{capt-of}
\usepackage{amsfonts} 
\usepackage{verbatim} 
\usepackage{amssymb}
\usepackage{verbatim}
\usepackage{enumitem}
\usepackage{sidecap,color} 
\usepackage{amsthm}

\usepackage{multirow}
\usepackage{braket}
\usepackage{booktabs}       % professional-quality tables
\usepackage{footnotehyper}
\usepackage{tablefootnote}
\def \X {\mathcal{X}}

\def \R {\mathbb{R}}

\def \w {\mathbf{w}}
\def \W {\mathcal{W}}

\def \x {\mathbf{x}}

\def \E {\mathrm{E}}

\def \x {\mathbf{x}}

\def \p {\mathbf{p}}

\def \a {\mathbf{a}}

\def \O {\mathcal{O}}

\def \b {\mathbf{b}}
\def \q {\mathbf{q}}

\def \d {\mathbf{d}}
\def \z {\mathbf{z}}

\def \y {\mathbf{y}}

\def \R {\mathbb{R}}

\def \D {\mathcal{D}}

\DeclareMathOperator*{\dist}{dist}

\def \E {\mathbb{E}}

\def \v {\mathbf{v}}

\newtheorem{theorem}{Theorem}

\newtheorem{lemma}{Lemma}
\newtheorem{corollary}{Corollary}
\newtheorem{definition}{Definition}
\newtheorem{assumption}{Assumption}

\newcommand*\samethanks[1][\value{footnote}]{\footnotemark[#1]}

\title{Stochastic Constrained DRO with a Complexity \\ Independent of Sample Size}

\author{\name{Qi  Qi\thanks{Equal Contribution.}}\email{qi-qi@uiowa.edu}\\
 \addr Department of Computer Science\\
 The University of Iowa, Iowa City, IA 52242, USA \\
\name{Jiameng Lyu\samethanks}\email{lvjm21@mails.tsinghua.edu.cn}\\
   \addr Department of Mathematical Sciences\\
   Tsinghua University, Beijing, 100084, China\\
\name{Kung-Sik Chan}\email{kung-sik-chan@uiowa.edu}\\
\addr Department of Statistics and Actuarial Science\\
The University of Iowa, Iowa City, IA 52242, USA \\
\name{Er-Wei Bai}\email{ er-wei-bai@uiowa.edu}\\
\addr Department of Electrical and Computer Engineering\\
The University of Iowa, Iowa City, IA 52242, USA \\
\name{Tianbao Yang}\email{tianbao-yang@tamu.edu}\\
\addr Department of Computer Science \& Engineering\\
Texas A\&M University, College Station, TX 77843, USA
}

\def\month{06}  % Insert correct month for camera-ready version
\def\year{2023} % Insert correct year for camera-ready version
\def\openreview{\url{https://openreview.net/forum?id=VpaXrBFYZ9}} % Insert correct link to OpenReview for camera-ready version

\begin{document}
\maketitle




\begin{center}
First Version\footnote{Compared with first version, we added more baselines and more refined comparisons with other related works.}: 11 Oct, 2022
\end{center}

\begin{abstract}
Distributionally Robust Optimization (DRO), as a popular method to train robust models against distribution shift between training and test sets, has received tremendous attention in recent years. In this paper, we propose and analyze stochastic algorithms that apply to both non-convex and convex losses for solving Kullback–Leibler divergence constrained DRO problem. Compared with existing methods solving this problem, our stochastic algorithms not only enjoy competitive if not better complexity independent of sample size but also just require a constant batch size at every iteration, which is more practical for broad applications. We establish a nearly optimal complexity bound for finding an $\epsilon$-stationary solution for non-convex losses and an optimal complexity for finding an $\epsilon$-optimal solution for convex losses. Empirical studies demonstrate the effectiveness of the proposed algorithms for solving non-convex and convex constrained DRO problems. 
\end{abstract}

\setlength{\textfloatsep}{2pt}% Remove \textfloatsep
%\abovedisplayskip=5pt
%\abovedisplayshortskip=5pt
%\belowisplayskip=5pt
%\belowdisplayshortskip=5pt
\setlength{\abovedisplayskip}{2pt}
\setlength{\belowdisplayskip}{2pt}

\section{Introduction}
Large-scale optimization of DRO has recently garnered increasing attention due to its promising performance on handling noisy labels,  imbalanced data and adversarial data~\citep{namkoong2017variance,zhu2019robust,qi2020attentional,chen2018robust}.
%1.The classic empirical risk minimization (ERM) minimizes the empirical risk over a training set and aims to achieve good performance on a test set. However, ERM often doesn't perform well in real applications especially for adversarial data and imbalanced data since the gap between the distribution of the training and test sets can't be ignored in these cases. 
%2.Distributionally Robust Optimization (DRO), as a popular method to train robust models against distribution shift between training and test sets, has received tremendous attention in recent years. And DRO promises to get reliable models handling noisy data, adversarial data and imbalanced data \citep{namkoong2017variance,zhu2019robust,qi2020attentional,chen2018robust}. However, a lack of large-scale optimization methods of DRO problems especially constrained DRO problems have hindered its application in practice.
%3.The surging need to train robust models in changing environments makes Distributionally Robust Optimization (DRO) receive tremendous attention since DRO shows promise to get the reliable models to handle the noisy data, adversarial data, and imbalanced data. 
%4.Unlike ERM whose objective is the emprical loss with the distribution of the training data, DRO proposes to minimize the worst-case loss over a ambiguity set . 
    Various primal-dual algorithms can be used for solving various DRO problems~\citep{rafique2021weakly,nemirovski2009robust}. However, primal-dual algorithms inevitably suffer from additional overhead for handling a $n$ dimensionality dual variable, where $n$ is the sample size. This is an undesirable feature for large-scale deep learning, where $n$ could be in the order of millions or even billions. Hence, a recent trend is to design dual-free algorithms for solving various DRO problems~\citep{qi2021online,jin2021non,levy2020large}. %For example, \citet{qi2021online} considered optimizing a Kullback–Leibler divergence regularized DRO problem by a duality-free algorithm. They formulated the problem into an equivalent compositional function and proposed efficient algorithms for optimizing the compositional function with improved complexities.  

In this paper, we provide efficient dual-free algorithms solving the following constrained DRO problem, which are still lacking in the literature,  
\begin{equation}
\label{eq:pdro}
\min_{\w \in \mathcal W}\max_{\{\p\in\Delta_n:D(\p, \mathbf 1/n)\leq  \rho\} }\sum_{i=1}^np_i\ell_i(\w) - \lambda_0 D(\p, \mathbf 1/n),
\end{equation}
where $\w$ denotes the model parameter, $\mathcal W$ is closed convex set, $\Delta_n=\{\p\in\R^n: \sum_{i=1}^n p_i =1, p_i\geq 0\}$ denotes a $n$-dimensional simplex, $\ell_i(\w)$ denotes a loss function on the $i$-th data, $D(\p, \mathbf 1/n)=\sum_{i=1}^np_i\log(p_in)$ represents the Kullback–Leibler (KL) divergence  measure between $\p$ and uniform probabilities $\mathbf 1/n\in\R^{n}$, and $\rho$ is the constraint parameter, and $\lambda_0>0$ is a small constant. A small KL regularization on $\p$ is added to ensure the objective in terms of $\w$ is smooth for deriving fast convergence. 

% \co{The motivation of the introduce of the 
% regularized term can be summarized from two aspects: 1)
% As we will show in Lemma~\ref{lem:Fsmooth}, the introduce of the KL-divergence
% regularization term to the constrained DRO can smooth the overall objective for easier and more
% efficient optimization. 2) It is a more general problem with both regularization term and constraints. \cite{levy2020large}studied the same problem, and
% \cite{qi2021online} verified the supreme performance of DRO
% with KL-divergence penalty without constraint on handling data
% imbalance problem.}

There are several reasons for considering the above constrained DRO problem. First, existing dual-free algorithms are not satisfactory~\citep{qi2021online,jin2021non,levy2020large,NEURIPS2021_b986700c}. They are either restricted to problems with no additional constraints on the dual variable $\p$ except for the simplex constraint ~\citep{qi2021online,jin2021non}, or restricted to convex analysis or have a requirement on the batch size that depends on accuracy level~\citep{levy2020large,NEURIPS2021_b986700c}. Second, the Kullback–Leibler divergence measure is a more natural metric for measuring the distance between two distributions than other divergence measures, e.g., Euclidean distance. %\co{And as we will show in Lemma~\ref{lem:Fsmooth}, the introduce of the KL-divergence regularization term to the KL-constrained DRO problem can smooth the overall objective for easier and more efficient optimization.} 
Third,
compared with the KL-regularized DRO problem {without constraints}, the above KL-constrained DRO formulation allows it to automatically decide a proper regularization effect that depends on the optimal solution by tuning the constraint upper bound $\rho$.
In other words, solving the constrained DRO with $\rho$ offers the capability of optimizing the temperature parameter $\lambda$ in Eq. (2),  which corresponds to  the log-sum-exponential form with a temperature parameter $\lambda$ is widely used in many ML/AI methods, e.g., constrastive self-supervised learning ~\citep{yuan2022provable, qiu2023not}. Empirical studies have demonstrated that selecting an appropriate value for $\lambda$ is crucial for achieving good performance \citep{goel2022cyclip, li2021align, radford2021learning}. Therefore, solving the constrained distributionally robust optimization problem provides the added advantage of identifying an optimal temperature during the training process.

The question to be addressed is the following: 

 %It is also notable that tuning $\rho$  in $[0,\log n]$ with $\lambda_0$ fixed to a small value is more manageable than tuning $\lambda_0\in(0,\infty)$ without the constraint $D(\p, \mathbf 1/n)\leq  \rho$. 
 %\textcolor{red}{Unlike Empirical Risk Minimization (ERM) whose objective is the empirical average loss over the training data distribution, DRO aims to minimize the worst-case losses over an ambiguity set ${\U_{\p}:=\{\p:D(\p, \mathbf 1/n)\leq \rho \}}$, which is beneficial to improve the robustness of the models. To optimize DRO~(\ref{eq:pdro}), several primal-dual methods \citet{rafique2021weakly,nemirovski2009robust} can be directly applied. The memory cost could be as high as $\mathcal{O}(n)$ per-iteration in the primal-dual methods for updating the dual variables, which hinders their application on large-scale problems. Recently, \citep{levy2020large,duchi2021learning} proposed mini-batch stochastic gradient algorithms. To finding a $\epsilon$-stationary point for the non-convex loss or achieving the $\epsilon$-objective gap for the convex loss, it requires the batch size in the level of $\mathcal O(1/\epsilon)$ to control the stochastic estimator per-iteration variance to guarantee the algorithm convergence, which is detrimental for heterogeneous training data on deep learning~\citep{mccandlish2018empirical}. To overcome the above deficiencies, we consider only KL divergence in the paper and raise a research question:}
 
\shadowbox{\begin{minipage}[t]{0.95\columnwidth}%
\it Can we develop stochastic algorithms whose oracle complexity is optimal for both convex and non-convex losses, and its per-iteration complexity is independent of sample size $n$ without imposing any requirements on the (large) batch size in the meantime?%
\end{minipage}}

%An oracle complexity lower bound scaling as $\Omega(n\epsilon^{-2/3})$ has been proved by \citep{carmon2021thinking}, showing dependence on $n$ in oracle complexity is necessary when solving the optimization problem: $\min_{\w \in \mathbb{R}^d} \max_{\p\in\Delta_n} \sum_{i=1}^n  p_i\ell_i(\w).$   
%Therefore, we relax the original constrained distributionally robust optimization problem by adding a small regularization term to achieve our goal. Concretely, we consider the following DRO formulation:
%\begin{align}\label{eq:cdro2}
%\min_{\w \in \mathcal W} \max_{\{\p\in\Delta_n:D(\p, \mathbf 1/n)\leq  \rho\} } L(\w,\p) - \lambda_0 D(\p, 1/n),
%\end{align}
%where  $\lambda_0> 0$ is a small constant. $D(\p, \mathbf 1/n)$ represents the KL divergence in the following paper.

We address the above question by (i) deriving an equivalent primal-only formulation that is of a compositional form; (ii) designing two algorithms for non-convex losses and extending them for convex losses; (iii) establishing an optimal complexity for both convex and non-convex losses. In particular, for a non-convex and smooth loss function $\ell_i(\w)$, we achieve an oracle complexity of $\widetilde\O(1/\epsilon^3)$\footnote{$\widetilde\O$ omits a logarithmic dependence over $\epsilon$.} for finding an $\epsilon$-stationary solution; and for a convex and smooth loss function, we achieve an oracle complexity of $\O(1/\epsilon^2)$ for finding an $\epsilon$-optimal solution. We would like to emphasize that these results are on par with the best complexities that can be achieved by primal-dual algorithms~\citep{huang2020accelerated,namkoong2016stochastic}. But our algorithms have a per-iteration complexity of $\O(d)$, which is independent of the sample size $n$. The convergence comparison of different methods for solving~(\ref{eq:pdro}) is shown in Table~\ref{table:algs}. 

%To make it easier to optimize, we first invoke dual variable $\lambda$ for the inner maximization robust loss to
% relax the the constrains on $\p$ and obtain the dual problem as  follows:
%\begin{align}\label{eq:lag}
%  L(\w) = \max_{\p\in\Delta_n, D(\p, \mathbf 1/n)\leq  \rho}\sum_{i=1}^np_i \ell_i(\w)  - \lambda_0 D(\p, 1/n)
% \end{align}
%%  \begin{align}\label{eq:lag}
%  \min_{\lambda\geq 0}\max_{\p\in\Delta_n}L(\w,\p)  - (\lambda+\lambda_0) D(\p, 1/n) + \lambda \rho.
% \end{align}
  %To transform this primal problem to an unconstrained optimization problem,
 %Thanks to special structure of KL divergence, the inner maximization robust loss function has a close form in terms of $\p$ for a fixed 
%\begin{equation}\label{eq:robustloss}
%    \max_{\{\p\in\Delta_n:D(\p, \mathbf 1/n)\leq  \rho\}}L(\w,\p) -  \lambda_0 D(\p, 1/n).
%\end{equation}

%Then by maximizing over $\p$ exactly and substituting $\lambda+\lambda_0$ with $\lambda$ (see Section~\ref{sec:derivation} in Appendix for the detailed derivation), 

To achieve these results, we first convert the problem~(\ref{eq:pdro}) into an equivalent problem:
 \begin{align}\label{eq:cdro3}
 \min_{\w \in \mathcal W}\min\limits_{\lambda \geq \lambda_0}F(\w, \lambda):=\lambda \log\left(\frac{1}{n}\sum_{i=1}^n  \exp\left(\frac{\ell_i(\w)}{\lambda}\right)\right) + {(\lambda-\lambda_0)\rho}.
 \end{align}
%   \begin{align}\label{eq:cdro3}
%  \min_{\w \in \mathcal W}\min\limits_{\lambda \geq 0}{(\lambda+\lambda_0) \log\frac{1}{n}\sum_{i=1}^n  \exp(\frac{\ell_i(\w)}{\lambda+\lambda_0}) +(\lambda+\lambda_0) \rho}.
%  \end{align}
By considering $\x=(\w^{\top}, \lambda)^{\top}\in\R^{d+1}$ as a single variable to be optimized, the objective function is a compositional function of $\x$ in the form of $f(g(\x))$, where $g(\x) =\left[\lambda, \frac{1}{n}\sum_{i=1}^n  \exp\left(\frac{\ell_i(\w)}{\lambda}\right)\right]\in\R^{2}$ and $f(g)=g_1\log(g_2) + g_1\rho$.  However, there are several challenges to be addressed for achieving optimal complexities for both convex and non-convex loss functions $\ell_i(\w)$. First, the problem $F(\x)$ is non-smooth in terms of $\x$ given the domain constraint $\w\in\W$ and $\lambda\geq \lambda_0$. Second, the outer function $f(g)$'s gradient is non-Lipschtiz continuous in terms of the second coordinate $g_2$ if $\lambda$ is unbounded, which is essential for all existing stochastic compositional optimization algorithms. Third, to the best of our knowledge, no optimal complexity in the order of $\O(1/\epsilon^2)$ has been achieved for a convex compositional function except for~\citet{Zhang2020OptimalAF}, which assumes $f$ is convex and component-wisely non-decreasing  and hence is not applicable to~(\ref{eq:cdro3}).

To address the first two challenges, we derive an upper bound for the optimal $\lambda$ assuming that $\ell_i(\w)$ is bounded for $\w\in\W$, i.e., $\lambda\in[\lambda_0, \tilde\lambda]$, which allows us to establish the smoothness condition of $F(\x)$ and $f(g)$. Then we consider optimizing $\bar F(\x)=F(\x) + \delta_{\X}(\x)$, where $\delta_{\X}(\x)=0$ if $\x\in\mathcal X=\{\x=(\w^{\top}, \lambda)^{\top}: \w\in\W, \lambda\in[\lambda_0, \tilde\lambda]\}$. By leveraging the smoothness conditions of $F$ and $f$, we design stochastic algorithms by utilizing a recursive variance-reduction technique to compute a stochastic estimator of the gradient of $F(\x)$, which allows us to achieve a complexity of $\widetilde\O(1/\epsilon^3)$ for finding a solution $\bar\x$ such that $\E[\text{dist}(0, \hat\partial\bar F(\bar\x))]\leq \epsilon$. To address the third challenge, we consider optimizing $\bar F_\mu(\x)=\bar F(\x) + \mu\|\x\|^2/2$ for a small $\mu$. We prove that $\bar F_\mu(\x)$ satisfies a Kurdyka-\L ojasiewicz inequality, which allows us to boost the convergence of the aforementioned algorithm to enjoy an optimal complexity of $\O(1/\epsilon^2)$ for finding an $\epsilon$-optimal solution to $\bar F(\x)$. Besides the optimal algorithms,  we also present simpler algorithms with worse complexity, which are more practical for deep learning applications without requiring two backpropagations at two different points per iteration as in the optimal algorithms.  

In the existing analysis of compositional optimization algorithms, either (i) the problem is assumed to be unconstrained, e.g., \cite{qi2020attentional,qi2021online}, or (ii) the complexity is sub-optimal, e.g., \cite{ghadimi2020single}, or (iii) the problem is restricted, e.g., the outer function $f$ is convex and non-decreasing as assumed in~\citep{Zhang2020OptimalAF}. To the best of our knowledge, this is {\bf the first result } for stochastic compositional optimization with a domain constraint that enjoys the optimal complexities for both convex and non-convex objectives.

\section{Related Work}
DRO springs from the robust optimization literature~\citep{bertsimas2018data,ben2013robust}
and has been extensively studied in machine learning and statistics~\citep{ahmadi2012entropic,namkoong2017variance,duchi2016statistics,staib2019distributionally,deng2020distributionally,qi2020simple,duchi2021learning}, and operations research~\citep{rahimian2019distributionally,delage2010distributionally}. Depending on how to constrain or regularize the uncertain variables, there are constrained DRO formulations that specify a constraint set for the uncertain variables, and regularized DRO formulations that use a regularization term in the objective for regularizing the uncertain variables~\citep{levy2020large}.   \citet{duchi2016statistics} showed that minimizing constrained DRO with $f$-divergence including a $\chi^2$-divergence constraint and a KL-divergence constraint, is equivalent to adding variance regularization for the Empirical Risk Minimization (ERM) objective, which is able to reduce the uncertainty and improve the generalization performance of the model.  %Zhu et al. (2019)~\citep{zhu2019robust}, Duchi et al. (2016)~\citep{duchi2016statistics} also verify the effectiveness of DRO over ERM on data imbalance problems.
%variance regularization ~\citep{namkoong2017variance}, \citep{duchi2016statistics}
%In~\citep{namkoong2017variance}the authors proved that minimizing the DRO formulation with a quadratic regularization in a constraint form is equivalent to minimizing the sum of the empirical loss and a variance regularization defined on itself. Variance regularization can enjoy better generalization error compared with the empirical loss minimization~\citep{namkoong2017variance}, and was also observed to be effective for imbalanced data~\citep{namkoong2017variance,zhu2019robust}. Recently, \citep{duchi2016statistics} also establishes this equivalence for a broader family of regularization function $h(\p, \mathbf 1/n)$ including the KL divergence.
%\vspace{-0.2in}
\begin{savenotes}
\begin{table*}[t]
\centering
\caption{Summary of algorithms solving KL-constrained DRO problem. Complexity represents the oracle complexity for achieving $\E[\dist (0, \hat\partial\bar F({\x}))]\leq\epsilon$  or other first-order stationarity convergence for the non-convex setting and $\E[F(\x)-F(\x_*)]\leq\epsilon$ for the convex setting. Per Iter Cost denotes the per-iteration computational complexity. The algorithm styles include primal-dual (PD), primal only (P), and compositional (COM). ``-" means not available in the original paper. }%\footnote{A Multi-level Monte Carlo (MLMC) gradient estimator has been proposed in~\citep{levy2020large} to improve the convergence rate of Dual SGM and FastDRO up to  $O(1\epsilon^2\log(1/\epsilon)$}

\label{table:algs}
\resizebox{1\linewidth}{!}
{\begin{tabular}{c|c|c|c|c|c|c}
\hline
Setting                     & Algorithms         & Reference                                 &  Complexity                      & Batch Size               &Per Iter Cost  & Style         \\ \hline
\multirow{4}{*}{Non-convex}& PG-SMD2\footnote{{PG-SMD2 refers to PG-SMD algorithm under Assumption D2 in \citet{rafique2021weakly}. }
 }          &  \citep{rafique2021weakly} & $\mathcal O(1/\epsilon^4)$             & $\mathcal O(1)$     &     $\mathcal O(n+d) $   & PD   \\   \cline{2-6} 
                            & AccMDA                & \citep{huang2020accelerated}                                    & $\mathcal O(1/\epsilon^3)$            & $\mathcal O(1)$ & $\mathcal O(n+d)$  & PD 
                            \\\cline{2-6} 
                  &   Dual SGM &  \citep{levy2020large}                                       &  -              & $\mathcal O(1)$  &         $\mathcal O(d)$  & P  \\       
                            \cline{2-6} 
                            & SCDRO              & \multirow{2}{*}{This work}                & $\mathcal O(1/\epsilon^4)$             & $\mathcal O(1)$     &   $\mathcal O(d)$      & COM \\ \cline{2-2} \cline{4-6} 
                            & ASCDRO             &                                           & $\widetilde{\mathcal O}(1/\epsilon^3)$ & $\mathcal O(1)$       &   $\mathcal O(d)$    & COM \\ \hline
\multirow{4}{*}{Convex}     & FastDRO\footnote{FastDRO is name of the GitHub repository of \citet{levy2020large}, and we use the name ``FastDRO'' to refer to the algorithm based on mini-batch gradient estimator in \citet{levy2020large}.}          & \citep{levy2020large}     & $\mathcal O(1/\epsilon^3)$             & $\mathcal O(1/\epsilon)$   &$\O(\frac{d}{\epsilon})$ & P   \\ \cline{2-6} 
                            & SPD &  \citep{namkoong2016stochastic}                                      & $\mathcal O(1/\epsilon^2)$             & $\mathcal O(1)$  &         $\mathcal O(n+d)$  & PD   \\ \cline{2-6} 
                            & Dual SGM &  \citep{levy2020large}                                       & $\mathcal O(1/\epsilon^2)$             & $\mathcal O(1)$  &         $\mathcal O(d)$  & P  \\ \cline{2-6} 
                            & RSCDRO            & \multirow{2}{*}{This work}                & $\mathcal O(1/\epsilon^3)$             & $\mathcal O(1)$   &  $\mathcal O(d)$       & COM \\ \cline{2-2} \cline{4-6} 
                            & RASCDRO           &                                           & $\mathcal O(1/\epsilon^2)$             & $\mathcal O(1)$ &  $\mathcal O(d)$         & COM \\ \hline
\end{tabular}}
\end{table*}
\end{savenotes}
%\vspace{-0.05in}

\noindent
\textbf{Primal-Dual Algorithms.} 
% \footnotetext[2]{{PG-SMD2 refers to PG-SMD algorithm under Assumption D2 in \citet{rafique2021weakly}. }}
% \footnotetext[3]{{FastDRO is name of the GitHub repository of \citet{levy2020large}, and we use the name ``FastDRO'' to refer to the algorithm based on mini-batch gradient estimator in \citet{levy2020large}.}  }
Many primal-dual algorithms designed for the min-max problems~\citep{nemirovski2009robust, juditsky2011solving, yan2019stochastic,namkoong2016stochastic, yan2020sharp,song2021variance,alacaoglu2022complexity} are applicable to solving~(\ref{eq:pdro}) when $\ell$ is a convex function. For non-convex loss functions,  recently, \citet{rafique2021weakly} and  \citet{yan2020sharp} proposed non-convex stochastic algorithms for solving non-convex strongly convex min-max problems, which are applicable to solving~(\ref{eq:pdro}) when $\ell$ is a weakly convex function or smooth. Many primal-dual stochastic algorithms have been proposed for solving non-convex strongly concave problems with a state of the art oracle complexity of $\O(1/\epsilon^3)$ for finding a stationary solution~\citep{huang2020accelerated,luo2020stochastic,tran2020hybrid}. However, the primal-dual algorithms require maintaining and updating an $\mathcal O(n)$ dimensional vector for updating the dual variable. %Thus the per-iteration cost can be as high as $\mathcal O(n)$ in the worst case, which is detrimental for the large-scale deep learning applications.

\noindent
\textbf{Constrained DRO.}
\citet{wang2021sinkhorn} studies the Sinkhorn distance constraint DRO, a variant of Wasserstein distance based on entropic regularization. An efficient batch gradient descent with a bisection search algorithm has been proposed to obtain a near-optimal solution with an arbitrarily small sub-optimality gap. However, no non-asymptotic convergence results are established in their paper. \citet{duchi2021learning} developed a convex DRO framework with $f$-divergence constraints to improve model robustness. The author developed the finite-sample minimax upper and lower bounds and the non-asymptotic convergence rate of $\mathcal O(1/\sqrt{n})$, and provided the empirical studies on real distributional shifts tasks with existing interior point solver \citep{udell2014convex} and gradient descent with backtracking Armijo line-searches~\citep{boyd2004convex}. However, no stochastic algorithms that directly optimize the considered constrained DRO with non-asymptotic convergence rates are provided in their paper.

\noindent
Recently, \citet{levy2020large}  proposed  sample independent algorithms based on gradient estimators for solving a group of DRO problems in the convex setting.
 To be more specific, they achieved a convergence rate of  $\widetilde{\O}(1/\epsilon^2)$ for the $\chi^2$-constrained/regularized and CVaR-constrained convex DRO problems and the batch size of logarithmically dependent on the inverse accuracy level $\O(\log(1/\epsilon))$ with the help of multi-level Monte-Carlo (MLMC) gradient estimator. For the KL-constrained DRO objective and other more general setting,  they achieve a convergence rate of $\O(1/\epsilon^3)$ under a Lipschitz continuity assumption on the inverse CDF of the loss function and a mini-batch gradient estimator   with a batch size in the order $\O(1/\epsilon)$  (please refer to Table 3 in \citet{levy2020large}).
 In addition, \citet{levy2020large} also proposed a simple stochastic gradient method for solving the dual expression of the DRO formulation, which is called Dual SGM. In terms of convergence, they only discussed the convergence guarantee for the $\chi^2$-regularized and CVaR penalized convex DRO problems (cf. Claim 3 in their paper). However, there is still gap for proving the convergence rate of Dual SGM for non-convex KL-constrained DRO problems due to similar challenges mentioned in the previous section, in particular establishing the smoothness condition in terms of the primal variable and the Lagrangian multipliers (denoted as $\x, \nu, \eta$ respectively in their paper). This paper makes unique contributions for addressing these challenges by (i) removing $\eta$ in Dual SGM and  deriving the box constraint for our Lagrangian multiplier $\lambda$ for proving the smoothness condition; (ii) establishing an optimal complexity in the order of $\O(1/\epsilon^3)$ in the presence of non-smooth box constraints, which, to the best of our knowledge,  is the first time for solving a non-convex constrained compositional optimization problem. %   for the  we find Dual SGM could achieve convergence rate of $\O(1/\epsilon^2)$ in the convex setting and $\O(1/\epsilon^4)$ in the  non-convex setting by utilizing the upper bound of the optimal solution of the dual variable and the similar smoothness property derivation of the objective function in our paper.

Furthermore, it is noteworthy that the KL-constrained DRO formulation~(\ref{eq:cdro3}) offers a distinct advantage over the KL-regularized DRO problem without constraints. Specifically, the proposed algorithms enable automatic determination of an optimal regularization effect for the constrained DRO~(\ref{eq:cdro3}) upon the optimizing of $\lambda$, through the fine-tuning of the constraint upper bound $\rho$. This innovative approach has been empirically demonstrated to yield significant efficacy in the realm of contrastive learning, as substantiated by the findings of Qiu et al~\cite{qiu2023not}.

\noindent
\textbf{Regularized DRO.}
DRO with KL divergence regularization objective has shown superior performance for addressing data imbalanced problems \citep{qi2021online,qi2020attentional, li2020tilted,li2021tilted}. 
\citet{jin2021non} proposed a mini-batch
normalized gradient descent with momentum that can find a first-order $\epsilon$ stationary
point with an oracle complexity of $\mathcal O(1/\epsilon^4)$ for KL-regularized DRO and $\chi^2$ regularized DRO with a non-convex loss. They solve the challenge that the loss function could be unbounded. %However, a batch size in the level of $\mathcal O(1/\epsilon^2)$ is required.
\citet{qi2021online} proposed online stochastic compositional algorithms to solve KL-regularized DRO. They leveraged a recursive variance reduction technique (STORM~\citep{cutkosky2019momentum}) to compute a gradient estimator for the model parameter $\w$ only. They derived a complexity of $\widetilde\O(1/\epsilon^3)$ for a general non-convex problem and improved it to $\O(1/(\mu\epsilon))$ for a problem that satisfies an $\mu$-PL condition. \citet{qi2020attentional} reports a worse complexity for a simpler algorithm for solving KL-regularized DRO. 
%Compared with equation~(\ref{eq:cdro3}) in which $\lambda$ is a variable that is optimized during the training process, the regularized KL divergence limited the capacity of the DRO formulation and restricted the performance of the model.  However, the introduced dual variable $\lambda$ in equation~(\ref{eq:cdro3}) would increase the difficulty of the optimization process more than just one additional dimensional increase that can be trivially calculated with the model $\w$ at the same time. Due to the composition structure of equation~(\ref{eq:cdro3}), finer design and analysis for the dual variable $\lambda$ are necessary.
%Empirical studies in ~\citet{qi2020attentional} found that using a fixed $\lambda$ in KL-regularized DRO could harm the generalization performance due to its poor feature learning capability.  They introduced a two-stage decreasing $\lambda$ strategy to facilitate the deep learning of feature representations. With fixed $\lambda$, their algorithm suffers an oracle complexity of $\mathcal O(1/\epsilon^4)$ for finding an $\epsilon$-stationary points for non-convex losses. % but without introducing any additional optimization difficulty
 %The KL regularized DRO problem is a special case of problem~(\ref{eq:cdro3}) with a fixed pre-defined hyperparameter $\lambda$.
\citet{li2020tilted,li2021tilted} studied the effectiveness of KL regularized objective on different applications, such as enforcing
fairness between subgroups, and handling the class imbalance. %A  (hierarchy) mini-batch SGD stochastic algorithm has been proposed to optimize the DRO objective. %~\citet{qi2020attentional} proposed a robust SGD algorithm with momentum to address the regularized DRO with KL divergence problem with a theoretical guarantee of $\mathcal O(1/\epsilon^4)$ to find an $\epsilon$ stationary points for the non-convex losses. 
%\citet{qi2021online} proposed an online method based on the non-convex momentum variance reduction technique that achieves $\mathcal O(1/\epsilon^3)$ oracle complexity to find an $\epsilon$-stationary points in the non-convex setting, and optimal sample $\mathcal O(1/\mu\epsilon)$ complexity under KL condition to find an $\epsilon$-objective gap, $\mu$ denotes the conditional number. Empirical studies in the papers verify the effectiveness of the DRO objective on large-scale deep learning data imbalanced applications.
% mitigating the effect of outliers. 

%More related works are included in the appendix due to limit of space, which will not affect the discussion of results in this paper.

\noindent
 {{\bf Compositional Functions and DRO.} The connection between compositional functions and DRO formulations have been observed and leveraged in the literature.  \citet{RePEc:spr:aistmt:v:69:y:2017:i:4:d:10.1007_s10463-016-0559-8} studied the statistical estimation of compositional functionals with applications to estimating conditional-value-at-risk measures, which is closely related to the CVaR constrained DRO. However, they do not consider stochastic optimization algorithms. To the best of our knowledge, \citet{qi2021online} was the first to use stochastic compositional optimization algorithms to solve KL-regularized DRO problems.  Our work is different in that we solve KL-constrained DRO problems, which is more challenging than KL-regularized DRO problems. The benefits of using compositional optimization for solving DRO include (i) we do not need to maintain and update a high dimensional dual variable as in the primal-dual methods~\citep{rafique2021weakly}; (i) we do not need to worry about the batch size as in MLMC-based stochastic methods~\citep{levy2020large,NEURIPS2021_b986700c}.}

\noindent
{\bf Optimizing the Temperature Parameter.} Our formulation and algorithm can be applied to optimizing the temperature parameter in the temperature-scaled cross-entropy loss, which has wide applications in machine learning and artificial intelligence, e.g., knowledge distillation~\citet{HinVin15Distilling} and self-supervised learning~\citep{chen2020simple}. Recently, \cite{DBLP:journals/corr/abs-2305-11965} leveraged the optimization technique proposed in this paper for optimizing the individualized temperature parameter in the global contrastive loss of self-superivsed learning.

\section{Preliminaries}\label{sec:preliminaries}
% Since the robust loss function (\ref{eq:robustloss}) is concave in term of $\p$ given $\w$. With strong duality theorem, we know the value of the robust loss function is equal to
% \begin{align}\label{eq:dualform}
% \min_{\lambda\geq 0}\max_{\p\in\Delta_n}L(\w,\p)  -( \lambda+\lambda_0) D(\p, 1/n) + \lambda \rho.
% \end{align}

%In this section, we introduce notations, definitions and assumptions. We show that~(\ref{eq:pdro}) is equivalent to~(\ref{eq:cdro3}) in Section~\ref{sec:derivation} in Appendix. 

%: 
%\begin{align*}%\label{eq:cdro3}
% \min_{\w \in \mathcal W}\min\limits_{\lambda \geq \lambda_0}\underbrace{\lambda \log\left(\frac{1}{n}\sum_{i=1}^n  \exp\left(\frac{\ell_i(\w)}{\lambda}\right)\right) +\lambda \rho}_{F(\w,\lambda)}.
% \end{align*}
\noindent
\textbf{Notations:} Let $\|\cdot\|$  denotes the Euclidean norm of a vector or the spectral norm of a matrix. And $\x = (\w^\top,
\lambda)^\top\in \R^{d+1}$, $g_i(\x) = \exp(\frac{\ell_i(\w)}{\lambda})$  and $g(\x) = \E_{i\sim\D}[\exp(\frac{\ell_i(\w)}{\lambda})]$ where $\D$ denotes the training set and $i$ denotes the index of the sample randomly generated from $\D$. Let $f_\lambda(\cdot)=\lambda \log(\cdot) + \lambda\rho$, and $\nabla f_\lambda(g)=\frac{\lambda}{g}$ denotes the gradient of $f$ in terms of $g$. Let $\Pi_{\X}(\cdot)$ denote an Euclidean projection onto the domain $\X$. Let $[T]=\{1,\ldots, T\}$ and $\tau\sim[T]$ denotes a random selected index.  We make the following standard assumptions regarding to the problem~(\ref{eq:cdro3}).
\begin{assumption}
\label{ass:1}
There exists  $R$, $G$, $C$,  and $L$ such that

\begin{enumerate}
    \item[(a)] {The domain of model parameter $\mathcal{W}$ is bounded  such that there exists $R>0$ it holds $\|\mathbf w\|\leq R$  for any $\mathbf w\in\mathcal W$}
    %The domain of model parameter $\mathcal W$ is bounded by $R$, i.e., \textcolor{red}{for all $\w \in \W$, we have $\|\w\|\leq R$}.
  
    \item[(b)]  
$\ell_i(\w)$ is  $L$-smooth, i.e., $\|\nabla\ell_i(\w_1)  - \nabla \ell_i(\w_2) \|\leq L\|\w_1 -\w_2\|$, $\forall \w_1, \w_2 \in\mathcal W,  {i\in\D}$.
    \item[(c)] 
    $\ell_i(\w)$ is $G$-Lipschitz continuous function and bounded by $C$, i.e.,  $\| \nabla\ell_i(\w)\| \leq G$ and $|\ell_i(\w)|\leq C$  for all $\w \in \mathcal W$ and $i\in\D$.

     \item[(d)]  There exists a positive constant $\Delta< \infty$ and an initial solution $(\w_1, \lambda_1)$ such that $F(\w_1, \lambda_1) - \min\limits_{\w\in\W}\min\limits_{\lambda\geq\lambda_0} F(\w, \lambda)\leq \Delta$. % for a valid initial point $\x_1$. %.where $\X =\{\x | \w\in \mathcal W, \lambda_0\leq\lambda \leq \tilde\lambda\}$ and $\tilde\lambda=\lambda_0+C/\rho$.
\end{enumerate}
\end{assumption}

 \begin{assumption}
\label{ass:2}
 \noindent
 Let $\sigma_g$, $\sigma_{\nabla g}$ be positive constants and  $\sigma^2 = \max\{\sigma_g , \sigma_{\nabla g}\}$. For ${i\in\D}$, assume that
%\begin{equation*}
 %   \begin{aligned}
   $\E[\| g_i(\x) - g(\x)\|^2] \leq \sigma_g^2,\ \ \ \ \E[\| \nabla g_i(\x) - \nabla g(\x)\|^2] \leq \sigma_{\nabla_g}^2$.
  %  \end{aligned}
%\end{equation*}
\end{assumption}

%\begin{assumption}\label{ass:3}
%The domain of model parameter $\mathcal W$ is bounder by $R$, i.e., for all $\w \in \W$, we have $\|\w\|\leq R$.
%\end{assumption}
\noindent
\textbf{Remark: }Assumption~\ref{ass:1} (a), i.e., the boundness condition of $\W$ is also assumed in \citet{levy2020large}, which is mainly used for convex analysis.  Assumption~\ref{ass:1}(b), (c), i.e., the Lipstchiz continuity and smoothness of loss function, and the variance bounds for $g_i$ and its gradient in Assumption 2 can be derived from Assumption 1 (b), such that $\E[\| g_i(\x) - g(\x)\|^2] \leq \E[\|g_i(\x)\|^2]\leq\exp(\frac{2C}{\lambda_0}) $, and $\E[\| \nabla g_i(\x) - \nabla g(\x)\|^2] \leq \E[\| \nabla g_i(\x)\|^2]\leq\exp(\frac{2C}{\lambda_0})(G^2+\frac{C^2}{\lambda_0})$\footnote{We would like to point out that the variance bound and the smoothness constant $L_F$ are exponentially dependent on the problem parameters, so are these constants in some other stochastic methods solving constrained DRO, like Dual SGM in \citet{levy2020large}.}
%are standard assumptions in the literature for DRO problems~\citep{levy2020large,qi2021online,jin2021non}}.
 %The bounded domain $\X$ in Assumption~\ref{ass:1}(c) is derived from Lemma~\ref{lem:lambda_upper}. 
%in  is Assumption~\ref{ass:1}-\ref{ass:2} are standard assumptions for the Distributionally Robust Optimization problem \citep{levy2020large,qi2021online}.

However, $F(\w, \lambda)$ is not necessarily smooth in terms of $\x=(\w^{\top},\lambda)^{\top}$ if $\lambda$ is unbounded. To address this concern, we prove that optimal $\lambda$ is indeed bounded. 
\begin{lemma}
 \label{lem:lambda_upper}
 The optimal solution of the dual variable $\lambda^*$ to the problem~(\ref{eq:cdro3}) is upper bounded by $\tilde\lambda=\lambda_0+{C}/{\rho}$, where $C$ is the upper bound of the loss function and $\rho$ is the constraint parameter.
 \end{lemma}

%With Lemma~\ref{lem:lambda_upper}, it is easy to note that the optimal solution in problem~(\ref{eq:cdro3}) is bounded by $\tilde\lambda=\lambda_0+C/\rho$. 
Thus, we could constrain the domain of $\lambda$ in the DRO formulation~(\ref{eq:cdro3})  with the upper bound $\tilde \lambda$ , and obtain the following equivalent formulation:
% In Lemma~\ref{lem:lambda_upper}, we prove the optimal dual Lagrangian multiplier $\lambda^*$ is bounded by $\tilde\lambda=\lambda_0+C/\rho$. The upper bound $\tilde\lambda$ guarantees the smoothness of $F(\w,\lambda)$ and the boundness of stochastic compositional estimator in the proposed algorithms, which are critical to establish the convergence rates. Thus, combines with $\lambda\in[\lambda_0,\tilde\lambda]$, the DRO formulation~(\ref{eq:cdro3}) is equivalent to:
%We need to bound $\lambda/g(\x)$ with a constant to complete the  convergence analysis in the non-convex loss function case, and thus the bounded domain of $\lambda$ is needed. Moreover, the bounded domain will also  help us  obtain the  and establish the convergence rate in the convex loss function case. Therefore, we constrain $\lambda$ with a upper bound $\tilde\lambda$, where $\tilde\lambda=\lambda_0+C/\rho$, and consider the following optimization problem

\begin{equation}
\label{eqn:prob1}
    \begin{aligned}
      %\min_{\w \in \mathcal W}\min\limits_{\lambda_0 \leq \lambda\leq \tilde\lambda}\underbrace{\lambda %\log\frac{1}{n}\sum_{i=1}^n  \exp(\frac{\ell_i(\w)}{\lambda}) +\lambda \rho}_{F(\w,\lambda)}.
      \min_{\w \in \mathcal W}\min\limits_{\lambda_0 \leq \lambda\leq \tilde\lambda}\lambda \log\left(\frac{1}{n}\sum_{i=1}^n  \exp\left(\frac{\ell_i(\w)}{\lambda}\right)\right) +\lambda\rho.
    \end{aligned}
\end{equation}

The upper bound $\tilde\lambda$ guarantees the smoothness of $F(\w,\lambda)$ and the smoothness of $f_\lambda(\cdot)$, which are critical for the proposed algorithms to enjoy fast convergence rates. 
\begin{lemma}
\label{lem:Fsmooth}
%\label{lem:F_pty}
% $\widetilde{\lambda}L_g^2+L_g^2+\widetilde{\lambda}L_{\nabla_g} + 1)$
 $F(\w,\lambda)$ is $L_F$-smooth for any $\w\in\W$ and $\lambda\in[\lambda_0, \tilde\lambda]$, where $L_F=\tilde{\lambda}L_g^2+2L_g+\tilde{\lambda}L_{\nabla_g} + 1 +\tilde{\lambda}$. $L_g$ and $L_{\nabla_g}$ are constants independent of sample size $n$ and explicitly derived in Lemma~\ref{lem:g_pty} .
\end{lemma}

Below, we let $\X =\{\x | \w\in \mathcal W, \lambda_0\leq\lambda \leq \tilde\lambda\}$, 
$\delta_\X(\x) = 0$ if $\x\in \X$, and $\delta_\X(\x)=  \infty $ if  $\x\notin \X$. The problem (\ref{eqn:prob1}) is equivalent to :

\begin{equation}
\label{eqn:unconstrained_pd}
 \min\limits_{\x\in\R^{d+1}}  \bar F(\x):= F(\x) + \delta_{\X}(\x),
\end{equation}
Since $\bar F$ is non-smooth, we define the regular subgradient as follows.
\begin{definition}[Regular Subgradient]
 Consider a function $\Phi: \mathbb{R}^{n} \rightarrow \overline{\mathbb{R}}$ and $\Phi(\bar{\x})$ is finite at a point $\bar{\x}$. For a vector $\v \in \mathbb{R}^{n}$, $\v$ is a regular subgradient of $\Phi$ at $\bar{\x}$, written $\v \in \hat\partial\Phi(\bar{\x})$, if
$$
\liminf_{\x\rightarrow\bar{\x}}\frac{\Phi(\x) - \Phi(\bar{\x}) - \v^\top (\x -\bar{\x})}{\| \x - \bar{\x}\|} \geq 0.
$$
\end{definition}
Since $F(\x)$ is differentiable, we use $\hat \partial\bar F(\x)=\nabla F(\x)+\hat \partial\delta_{\X}(\x)$ (see Exercise 8.8 in \cite{rockafellar1998variational}) in the analysis. Recall the definition of subgradient of a convex function $\bar F$ which is denoted by $\partial \bar F$. When $\bar F(\x)$ is convex, we have $\hat \partial\bar F(\x)=\partial\bar F(\x)$ (see Proposition 8.2 in \cite{rockafellar1998variational}). The $\dist (0, \hat\partial\bar F({\x}))$ measures the distance between the origin and the regular subgradient set of $\bar F$ at $\x$. The oracle complexity is defined below:
\begin{definition}[Oracle Complexity]
 Let $\epsilon>0$ be a small constant, the oracle complexity is defined as the number of processing samples $\z$ in order to achieve $\E[\dist (0, \hat\partial\bar F({\x}))]\leq\epsilon$ for a non-convex loss function or $\E[F(\x)-F(\x_*)]\leq\epsilon$ for a convex loss function.
 \end{definition}

 \subsection{Equivalence Derivation}
 Before we move to the proposed algorithms in the next section, we derive the equivalence between equation between equations~(\ref{eq:pdro}),~(\ref{eq:cdro3}), and~(\ref{eqn:prob1}). Recall the original KL-constrained DRO  problem: \begin{align*}
\min_{\w \in \mathcal W} \max_{\{\p\in\Delta_n:D(\p, \mathbf 1/n)\leq  \rho\} }\sum_{i=1}^np_i\ell_i(\mathbf w) - \lambda_0 D(\p, 1/n),
\end{align*}
where $\Delta_n=\{\mathbf p\in\mathbb R^n: \sum_{i=1}^n p_i=1, 0\leq p_i\leq 1\}$, $D(\p, 1/n)$ is the KL divergence and $\lambda_0$ is a small positive constant.

In order to tackle this problem, let us first consider the robust loss 
\begin{align*}
 \max_{\{\p\in\Delta_n:D(\p, 1/n)\leq \rho\}}\sum_{i=1}^n p_i \ell_i(\w) -  \lambda_0 D(\p, 1/n).
 \end{align*}
And then we invoke the dual variable $\lambda$ to transform this primal problem to the following form
 \begin{align*}
 \max_{\p\in\Delta_n}\min_{\bar\lambda\geq 0}\sum_{i=1}^n p_i \ell_i(\w) - \bar\lambda (D(\p, 1/n) -\rho) - \lambda_0 D(\p, 1/n).
 \end{align*}
Since this problem is concave in term of $\p$ given $\w$, by strong duality theorem, we have
\begin{align*}
 &\max_{\p\in\Delta_n}\min_{\bar\lambda\geq 0}\sum_{i=1}^n p_i \ell_i(\w) - \bar\lambda (D(\p, 1/n) -\rho) - \lambda_0 D(\p, 1/n)\\
 &= \min_{\bar\lambda\geq 0}\max_{\p\in\Delta_n}\sum_{i=1}^n p_i \ell_i(\w) - \bar\lambda (D(\p, 1/n) -\rho) - \lambda_0 D(\p, 1/n).
 \end{align*}
 Let $\lambda=\bar\lambda+\lambda_0$, we have
 \begin{align*}
 &\min_{\bar\lambda\geq 0}\max_{\p\in\Delta_n}\sum_{i=1}^n p_i \ell_i(\w) - \bar\lambda (D(\p, 1/n) -\rho) - \lambda_0 D(\p, 1/n)\\
    &= \min_{\lambda\geq \lambda_0}\max_{\p\in\Delta_n}\sum_{i=1}^n p_i \ell_i(\w) - \lambda (D(\p, 1/n) -\rho) - \lambda_0 \rho.
 \end{align*}
Then the original problem is equivalent to the following problem 
\begin{equation*}
%\label{eqn:prob11}
    \begin{aligned}
     \min_{\w \in \mathcal W} \min\limits_{\lambda \geq \lambda_0}\max_{\p\in\Delta_n}\sum_{i=1}^n p_i \ell_i(\w) - \lambda (D(\p, 1/n) -\rho) - \lambda_0 \rho,
    \end{aligned}
\end{equation*}
Next we  fix $\x=(\w^\top,\lambda)^\top$  and derive an optimal solution $\p^*(\mathbf x)$ which depends on $\x$ and solves the inner maximization problem.  We consider the  following problem 
\begin{align*}
\min_{\p\in\Delta_n}-\sum_{i=1}^n p_i \ell_i(\w) + \lambda D(\p, 1/n).
\end{align*}
which has the same optimal solution $\p^*(\mathbf x)$ with our problem.

 There are three constraints to handle, i.e., $p_i\geq 0, \forall i$ and $p_i\leq 1, \forall i$ and $\sum_{i=1}^n p_i=1$.  Note that the constraint $p_i\geq 0$ is enforced by the term $p_i\log(p_i)$, otherwise the above objective will become infinity. As a result, the constraint $p_i<1$ is automatically satisfied due to $\sum_{i=1}^np_i=1$ and $p_i\geq0$. Hence, we only need to explicitly tackle the constraint $\sum_{i=1}^n p_i=1$. To this end, we define the following Lagrangian function
\begin{align*}
 L_{\mathbf x}(\mathbf p, \mu) = -\sum_{i=1}^n p_i\ell_i(\mathbf w)  +  \lambda\left(\log n + \sum_{i=1}^n  p_i \log (p_i)\right) + \mu(\sum_{i=1}^n  p_i - 1), 
\end{align*}
where $\mu$ is the Lagrangian multiplier for the constraint $\sum_{i=1}^n  p_i=1$. The optimal solutions satisfy the KKT conditions:  
\begin{align*}
    &- \ell_i(\mathbf w)  +  \lambda \left(\log (p^*_i(\mathbf x)) + 1\right) + \mu  = 0 \text{ and }\sum_{i=1}^n  p^*_i(\mathbf x)=1.
\end{align*}
From the first equation, we can derive $p^*_i(\mathbf x) \propto \exp(\ell_i(\mathbf w)/\lambda)$. Due to the second equation, we can conclude that $p^*_i(\mathbf x) = \frac{\exp(\ell_i(\mathbf w)/\lambda)}{\sum_{i=1}^n  \exp(\ell_i(\mathbf w)/\lambda)}$. Plugging this optimal $\p^*(\w)$ into the inner maximization problem, we have 
\begin{align*}
     \sum_{i=1}^np^*_i(\mathbf x)\ell_i(\mathbf w)  - \lambda\left(\log n + \sum_{i=1}^n  p_i^*(\mathbf w) \log (p_i^*(\mathbf w)) \right)  = \lambda \log \left(\frac{1}{n}\sum_{i=1}^n  \exp\left(\frac{\ell_i(\mathbf w)}{\lambda}\right)\right) , 
\end{align*}
% Therefore, combining Lemma~\ref{lem:lambda_upper} we get the following equivalent problem to the original problem
% \begin{equation*}
%     \begin{aligned}
%       %\min_{\w \in \mathcal W}\min\limits_{\lambda_0 \leq \lambda\leq \tilde\lambda}\underbrace{\lambda %\log\frac{1}{n}\sum_{i=1}^n  \exp(\frac{\ell_i(\w)}{\lambda}) +\lambda \rho}_{F(\w,\lambda)}.
%       \min_{\w \in \mathcal W}\min\limits_{\lambda_0 \leq \lambda\leq \tilde\lambda}\lambda \log\left(\frac{1}{n}\sum_{i=1}^n \exp\left(\frac{\ell_i(\w)}{\lambda}\right)\right) +\lambda \rho.
%     \end{aligned}
% \end{equation*}
% which is Eq.~(\ref{eqn:prob1}) in the paper. 
Therefore, we get the following equivalent problem to the original problem
\begin{equation*}
    \begin{aligned}
      %\min_{\w \in \mathcal W}\min\limits_{\lambda_0 \leq \lambda\leq \tilde\lambda}\underbrace{\lambda %\log\frac{1}{n}\sum_{i=1}^n  \exp(\frac{\ell_i(\w)}{\lambda}) +\lambda \rho}_{F(\w,\lambda)}.
      \min_{\w \in \mathcal W}\min\limits_{\lambda \geq \lambda_0}\lambda \log\left(\frac{1}{n}\sum_{i=1}^n \exp\left(\frac{\ell_i(\w)}{\lambda}\right)\right) +\lambda \rho.
    \end{aligned}
\end{equation*}
which is Eq.~(\ref{eq:cdro3}) in the paper. 
 
%  \begin{enumerate}
%     \item[(a)] $\dist (0, \hat\partial\bar F({\x}))\leq\epsilon$ in the non-convex loss function case where 
%     \item[(b)] $\E[F(\x)-F(\x_*)]\leq\epsilon$ in the convex loss function case.
% \end{enumerate}
% We denote by $\dist (\x, \X)$ the distance between the vector $\x$ and a set $\X$. Denote by $\hat{\partial}\bar F(\x)$ the Frechet subgradient and $\partial\bar F(\x)$ the limiting subgradient of a non-convex function $\bar F(\x): \R^d \rightarrow \R$, i.e.,
% \begin{equation}
%     \begin{aligned}
%       \bar\partial \bar F(\bar{\x}) &=\Big \{\v \in \R^d: \liminf_{\x\rightarrow\bar{\x}}\frac{\bar F(\x) - \bar F(\bar{\x}) - \v^\top (\x -\bar{\x})}{\| \x - \bar{\x}\|} \geq 0 \Big \} \\
%       \partial \bar F(\bar{\x}) &= \Big \{\v\in \R^d: \x_k\xrightarrow{\bar F}\bar{\x}, \v_k\in  \partial \bar F(\bar{\x}), \v_k \rightarrow \v\Big\}
%     \end{aligned}
% \end{equation}

\section{Stochastic Constrained DRO with Non-convex
Losses}\label{sec:basic}

In this section, we present two stochastic algorithms for solving~(\ref{eqn:unconstrained_pd}). The first algorithm is simpler yet practical for deep learning applications. The second algorithm is an accelerated one with a better complexity, which is more complex than the first algorithm. 
\subsection{Basic Algorithm: SCDRO}

A major concern of the algorithm design is to compute a stochastic gradient estimator of the gradient of $F(\x)$. At iteration $t$, the gradient of $F(\x_t)$ is given by 
\begin{equation}\label{eqn:wl}
\begin{aligned}
    &\nabla_{\w} F(\x_t)=\nabla f_{\lambda_t} (g(\x_t))\nabla_{\w} g(\x_t)\\
    &\nabla_{\lambda}F(\x_t)= \nabla f_{\lambda_t}  (g(\x_t))\nabla_{\lambda} g(\x_t) + \log(g(\x_t)) + \rho.
\end{aligned} 
\end{equation}
Both $\nabla_{\lambda} g(\x_t)$ and $\nabla_{\w} g(\x_t)$ can be estimated by unbiased estimator denoted by $\nabla g_i(\x_t)$.  The concern lies at how to estimate $g(\x_t)$ inside $\nabla f_{\lambda_t}(\cdot)$. The first algorithm SCDRO is applying existing techniques for two-level compositional function. In particular, we estimate $g(\x_t)$ by a sequence of $s_t$, which is updated by moving average 
%To form an unbiased estimator of $\nabla F(\x)$ and apply the Stochastic Gradient Descent we need to substitute $\partial_{\w} g(\x_t)$ with $\partial_{\w} g_i(\x_t)$ and substitute $\partial_{\lambda} g(\x_t)$ with $\partial_{\lambda} g_i(\x_t)$. However, in this way we still need to compute  $\nabla f_{\lambda_t} (g(\x_t))$ which has computation complexity $\mathcal{O}(n)$ for every iteration. 
%To achieve the minimal computational cost and even solve this optimization problem on the run we invoke the moving-average estimator of $g(\x_{t})$ as follows
%\vspace{-0.05in}
%\begin{align}\label{eqn:s}
    $s_{t} = (1-\beta )s_{t-1} + \beta g_i(\x_{t})$.
%\end{align}
Then we substitute $g(\x_t)$ in $\nabla_{\w} F(\x_t)$ and $\nabla_{\lambda} F(\x_t)$ with $s_t$,  and invoke the following moving average to obtain the gradient estimators  in terms of $\w_t$ and $\lambda_t$, respectively, 
\begin{align}
    \v_{t} & =(1-\beta )\v_{t-1} + \beta \nabla f_{\lambda_{t}}(s_{t})\nabla_{\w}g_i(\x_{t})\label{eqn:uv}\\
    u_{t} & = (1-\beta )u_{t-1} +\beta (\nabla f_{\lambda_{t}}(s_{t})\nabla_{\lambda} g_i(\x_{t})  +\log(s_{t}) + \rho)\notag.
\end{align}
   Finally we complete the update step of $\x_t$ by $\x_{t+1} = \Pi_\X(\x_t - \eta\z_t)$, where $\z_t =  (\v_t^\top, u_t)^\top$.

% in the update step of $\x_t$, i.e., $\x_{t+1} = \Pi_\X(\x_t - \eta\z_t)$, where $\z_t =  (\v_t^\top, u_t)^\top$.
%         $\v_{t} = \beta \nabla f_{\lambda_{t}}(s_{t})\partial_{\w}g_i(\x_{t})+(1-\beta )\v_{t-1}$
%   and $u_{t} = \beta (\nabla f_{\lambda_{t}}(s_{t})\partial_\lambda g_i(\x_{t})  +\log(s_{t}) + \rho)+(1-\beta )u_{t-1}$.
            % $\v_{t} = \beta \nabla f_{\lambda_{t}}(s_{t})\partial_{\w}g_i(\x_{t})+(1-\beta )\v_t$
   % and $u_{t+1} = \beta (\nabla f_{\lambda_{t+1}}(s_{t+1})\partial_\lambda g_i(\x_{t+1})  +\log(s_{t+1}) + \rho)+(1-\beta )u_t$.

\begin{minipage}[t]{0.5\textwidth}
\begin{algorithm}[H]
    \centering
\caption{SCDRO$(\x_{1}, \v_{1}, u_{1}, s_{1},\eta_{1},T_{1})$}\label{alg:SCCMA}
    \begin{algorithmic}[1]
        \STATE $\textbf{Input: } \w_{1}\in\W, \lambda_{1}\geq\lambda_0, \x_{1} = (\w_{1}^\top,\lambda_{1})^\top$
        \STATE \textbf{Initialization:} Draw a sample $\xi_1\sim \D$,  and calculate  $s_{1} = \exp(\ell_i(\w_{1})/\lambda_{1})$, 
        $$\v_{1} =\nabla f_{\lambda_{1}}(s_{1})\nabla_{\w} g_i(\x_{1}))\in \R^d$$ $$u_{1} = \nabla f_{\lambda_1}(s_{1})\nabla_{\lambda} g_i(\x_{1})  +\log(s_{1}) + \rho \in \R$$
        \FOR{$t=1,\cdots, T$}
         \STATE Update $\x_{t+1} = \Pi_\X(\x_t - \eta\z_t)$ %where $\z_t =  (\v_t^\top, u_t)^\top$
         \STATE Draw a sample $\xi_i\sim \D$
        \STATE Let $s_{t+1} = (1-\beta )s_t + \beta g_i(\x_{t+1})$
        \STATE Update  $\v_{t+1}, u_{t+1}$ according to~(\ref{eqn:uv}): \\   $\v_{t}  =(1-\beta )\v_{t-1} + \beta \nabla f_{\lambda_{t}}(s_{t})\nabla_{\w}g_i(\x_{t})\label{eqn:uv}$\\
   $u_{t}  = (1-\beta )u_{t-1}$ \\$\ \ \ +\beta (\nabla f_{\lambda_{t}}(s_{t})\nabla_{\lambda} g_i(\x_{t})  +\log(s_{t}) + \rho)\notag.$
        %\STATE Update $\v_{t+1}, u_{t+1}$ according to~(\ref{eqn:uv})% = \beta \nabla f_{\lambda_{t+1}}(s_{t+1})\partial_{\w} g_i(\w_{t+1},\lambda_{t+1})+(1-\beta )\v_t $
        %\STATE $u_{t+1} = \beta (\nabla f_{\lambda_{t+1}}(s_{t+1})\partial_\lambda g_i(\w_{t+1},\lambda_{t+1})  +\log(s_{t+1}) + \rho)+(1-\beta )u_t$
        % \STATE $\y_{t+1} =  \arg\min\limits_{\y_t \in \X} \Big\{ \langle \z_{t+1}, \y-\x_t\rangle + \frac{\eta }{2}\|\y-\x_t\|^2$ \Big\}
        % \STATE $\x_{t+1} = \x_t + \beta(\y_{t+1} - \x_t)$
        \ENDFOR 
        \RETURN  $(\x_\tau, \v_\tau, u_\tau, s_\tau)$, where $\tau\sim[T]$ %is uniformly sampled from $\{1,\cdots T\}$
    \end{algorithmic}
\end{algorithm}
\end{minipage}
\begin{minipage}[t]{0.5\textwidth}
\begin{algorithm}[H]
    \centering
    \caption{ASCDRO$(\x_{1}, \v_{1}, u_{1}, s_{1},\eta_{1}, T_{1})$}
    \begin{algorithmic}[1]
    \label{alg:STORM-CCMA}
        \STATE $\textbf{Input: } \w_1\in \W,  \lambda_1\geq\lambda_0$, $\x_1 = (\w_1^\top,\lambda_1)^\top$
        \STATE \textbf{Initialization:} Draw a sample $\xi_1\sim \D$, and calculate $s_1 = \exp(\ell_i(\w_1)/\lambda_1)$,
        $$\v_1 =\nabla_{\w} g_i(\x_{1})\in \R^d, \ u_1 =\nabla_\lambda g_i(\x_{1})\in \R$$
        \FOR{$t=1,\cdots, T$}
         \STATE Update $\x_{t+1} = \Pi_\X(\x_t - \eta\z_t)$, where $\z_t$ is given in~(\ref{eqn:stormz}): \\
        $\z_t=(\nabla f_{\lambda_t} (s_t)\v_t^\top, \nabla f_{\lambda_t}  (s_t)u_t + \log(s_t) + \rho)^\top$ %$\z_t=\nabla f_{\lambda_t}(s_t) \q_t + \q_{\lambda_t}$ and $\q_t =  (\v_t^\top, u_t)^\top, \q_{\lambda_t} = (\textbf{0}^\top,\log(s_{t})+\rho)^\top, \textbf{0}\in\R^{d}$
         \STATE Draw a sample $\xi_i\sim \D$
       % \STATE Let $s_{t+1} =  g_i(\x_{t+1}) + (1-\beta)(s_t - g_i(\x_{t}))$
       \STATE Update $s_{t+1}, \v_{t+1}, u_{t+1}$:\\
         $\v_{t} = \nabla_{\w} g_i(\x_{t})  + (1-\beta)(\v_{t-1} -\nabla_{\w} g_i(\x_{t-1}))$\\
    $u_{t} =  \nabla_{\lambda} g_i(\x_{t})  + (1-\beta)(u_{t-1} -\nabla_{\lambda} g_i(\x_{t-1}))$\\
    $s_{t} =  g_i(\x_{t}) + (1-\beta)(s_{t-1} - g_i(\x_{t-1})).$
    
        %\STATE Update $s_{t+1}, \v_{t+1}, u_{t+1}$ according to~(\ref{eqn:uv2})  %= \nabla_{\w} g_i(\x_{t+1})  + (1-\beta)(\v_t -\nabla_{\w} g_i(\x_{t})  )$
    %    \STATE $u_{t+1} =  \nabla_{\lambda} g_i(\x_{t+1})  + (1-\beta)(u_t -\nabla_{\lambda} g_i(\x_{t})  ) $
        \ENDFOR
    \RETURN $(\x_\tau, \v_\tau, u_\tau, s_\tau)$, where $\tau\sim[T]$
    \end{algorithmic}
\end{algorithm}
\end{minipage}

We would like to point out the moving average estimator for tracking the inner function $g(\w)$ is widely used for solving compositional optimization problems~\citep{wang2017stochastic,qi2021online,zhang2019stochastic,zhou2019momentum}. Using the moving average for computing a stochastic gradient estimator of a compositional function was first used in the NASA method proposed in~\citet{ghadimi2020single}. The proposed method SCDRO is presented in Algorithm~\ref{alg:SCCMA}. It is similar to NASA but with a simpler design on the update of $\x_{t+1}$. We directly use projection after an SGD-tyle update. In contrast, NASA uses two steps to update $\x_{t+1}$. As a consequence, NASA has two parameters for updating $\x_{t+1}$ while SCDRO only has one parameter $\eta$ for updating $\x_{t+1}$. %In addition, our convergence guarantee is directly in terms of $\dist (0, \hat{\partial}\bar F(\x_R))^2$, which is different from that used in~\citep{ghadimi2020single}. 
It is this simple change that allows us to extend SCDRO for convex problems in the next section.  
Below, we present the convergence rate of our basic algorithm SCDRO for a non-convex loss function.

\begin{theorem}
\label{thm:dist_main}
Suppose the Assumption~\ref{ass:1} and ~\ref{ass:2} hold, and set $\beta = \frac{1}{\sqrt{T}},\eta=\frac{\beta}{20L_F^2}$. Then after running Algorithm~\ref{alg:SCCMA} $T$ iterations, we have
%\vspace{-0.1in}
%\begin{equation*}
  %  \begin{aligned}
   $     \E[\dist (0, \hat{\partial}\bar F(\x_\tau))^2]\leq (624\sigma^2 +280\Delta)\frac{L_F^2}{\sqrt{T}} + \frac{20L_F^2\Delta}{T}$.
 %   \end{aligned}
%\end{equation*}
%\vspace{-0.15in}
\end{theorem}
\textbf{Remark: }Theorem~\ref{thm:dist_main} shows that SCDRO achieves a complexity of $\mathcal O(1/\epsilon^4)$ for finding an $\epsilon$-stationary point, i.e., $\E[\dist (0, \hat{\partial}\bar F(\x_R))]\leq\epsilon$ for a non-convex loss function. Note that NASA~\citep{ghadimi2020single} enjoys the same oracle complexity but for a different convergence measure, i.e., $\E[\|\y(\x, \z) - \x\|^2+\|\z - \nabla F(\x)\|^2]\leq \epsilon$  for a returned primal-dual pair $(\x, \z)$, where $\y(\x, \z) = \prod_{\mathcal X}[\x - \z]$. We can see that our convergence measure is more intuitive. In addition, we are able to leverage our convergence measure to establish the convergence for convex functions by using Kurdyka-\L ojasiewicz (KL) inequality and the restarting trick as shown in next section. In contrast, such convergence for NASA is missing in their paper. Compared with  stochastic primal-dual methods~\citep{rafique2021weakly,yan2020sharp} for the min-max formulation~(\ref{eq:pdro}), their algorithms are double looped and have the same oracle complexity for a different convergence measure, i.e., $\E[\dist (0, \hat{\partial}\bar F(\x_*))^2]\leq \gamma^2 \|\x-\x_*\|^2]\leq \epsilon$ for some returned solution $\x$, where $\x_*$ is a reference point that is not computable. Our convergence measure is stronger as we directly measure  $\E[\dist (0, \hat{\partial}\bar F(\x_\tau))^2]$ on a returned solution $\x_\tau$. This is due to that we leverage the smoothness of $F(\cdot)$.

\subsection{Accelerated Algorithm: ASCDRO}\label{sec:acceleration}
Our second algorithm presented in Algorithm~\ref{alg:STORM-CCMA} is inspired by \citet{qi2021online} for solving the KL-regularized DRO by leveraging a recursive variance reduced technique (i.e., STORM) to estimate $g(\w_t)$ and $\nabla g(\w_t)$ for computing $\nabla_{\w} F(\x_t)$ and $\nabla_{\lambda} F(\x_t)$ in~(\ref{eqn:wl}).  In particular, we use $\v_t$ for tracking $\nabla_{\w} g(\x_t)$, use $u_t$ for tracking $\nabla_{\lambda}g(\x_t)$, and use $s_t$ for tracking $g(\x_t)$, which are updated by:
\begin{equation}\label{eqn:uv2}
\begin{aligned}
    \v_{t} &= \nabla_{\w} g_i(\x_{t})  + (1-\beta)(\v_{t-1} -\nabla_{\w} g_i(\x_{t-1}))\\
    u_{t} &=  \nabla_{\lambda} g_i(\x_{t})  + (1-\beta)(u_{t-1} -\nabla_{\lambda} g_i(\x_{t-1}))\\
    s_{t} &=  g_i(\x_{t}) + (1-\beta)(s_{t-1} - g_i(\x_{t-1})).
\end{aligned}
\end{equation}
A similar update to $s_t$ has been used in~\citet{Chen_2021} for tracking the inner function values for two-level compositional optimization. However, they do not use similar updates for tracking the gradients as $\v_t, u_t$. Hence, their algorithm has a worse complexity.

Then we invoke these estimators into $\nabla_{\w} F(\x_t)$ and $\nabla_{\lambda} F(\x_t)$ to obtain the gradient estimator 
\begin{align}\label{eqn:stormz}
    \z_t=(\nabla f_{\lambda_t} (s_t)\v_t^\top, \nabla f_{\lambda_t}  (s_t)u_t + \log(s_t) + \rho)^\top.
\end{align}
%The detailed steps of the second algorithm referred to ASCDRO are presented in Algorithm~\ref{alg:STORM-CCMA}.
%For convenience, we express $\z_t$ as $\nabla f_{\lambda_t}(s_t) \q_t + \q_{\lambda_t}$ where $\q_t =  (\v_t^\top, u_t)^\top$, $\q_{\lambda_t} = (\textbf{0}^\top,\log(s_{t})+\rho)^\top$ and $\textbf{0}\in\R^{d}$. Finally we complete the update step of $\x_t$ by $\x_{t+1} = \Pi_\X(\x_t - \eta\z_t)$.
% \begin{equation}\label{eqn:uv3}
% \begin{aligned}
%   &\nabla f_{\lambda_t} (s_t)\v_t\\
%   &\nabla f_{\lambda_t}  (s_t)u_t + \log(s_t) + \rho
%   % \v_{t} &= \partial_{\w} g_i(\x_{t})  + (1-\beta)(\v_{t-1} -\partial_{\w} g_i(\x_{t-1}))\\
%   % u_{t} &=  \partial_{\lambda} g_i(\x_{t})  + (1-\beta)(u_{t-1} -\partial_{\lambda} g_i(\x_{t-1})).
% \end{aligned}
% \end{equation}
%This variant can provably reduce the variance of the gradients, and thus improve the convergence rate.
Below, we show  ASCDRO can achieve a better convergence rate in the non-convex loss function.

\begin{theorem}\label{thm:storm-ncx}
 Under Assumption~\ref{ass:1} and~\ref{ass:2}, for any $\alpha>1$, let $k = \frac{\alpha\sigma^{2/3}}{L_F}$, $w = \max(2\sigma^2, (16L^2_Fk)^3)$ and $c =\frac{\sigma^2}{14L_Fk^3}+130L_F^4$. Then after running Algorithm~\ref{alg:STORM-CCMA} for $T$ iterations with $\eta_t = \frac{k}{(w+t\sigma^2)^{1/3}}$ and $\beta_t = c\eta_t^2$, we have
% \vspace{-0.1in}
%\begin{align*}
 $  \E[\dist (0, \hat{\partial}\bar F(\x_\tau))^2] \leq\mathcal{O}\left(\frac{\log T}{T^{2/3}}\right)$. %, where $\tau\sim[T]$ is random selected. 
%\end{align*}
\end{theorem}

\noindent
\textbf{Remark:} Theorem~\ref{thm:storm-ncx} implies that with a polynomial decreasing step size, ASCDRO is able to find an $\epsilon$-stationary solution such that $\E[\dist (0, \hat{\partial}\bar F(\x_R))] \leq \epsilon$ with a near-optimal complexity $\widetilde{\mathcal O}(1/\epsilon^{3})$.  Note that the complexity
$\mathcal{\widetilde{O}}(1/\epsilon^{3})$ is optimal up to a logarithmic factor for solving non-convex smooth optimization problems~\citep{arjevani2019lower}. State-of-the-art primal-dual methods with variance-reduction for min-max problems~\citep{huang2020accelerated} have the same complexity but for a different convergence measure, i.e, $\E[\frac{1}{\gamma}\|\x - \prod_{\mathcal X}[\x - \gamma\nabla F(\x)]\|]\leq\epsilon$ for a returned solution $\x$.
%for making the gradient’s norm smaller than $\epsilon$ in expectation

\section{Stochastic Algorithms for Convex Problems}\label{sec:convex}
%\vspace{-0.1in}
In this section, we presented restarted algorithms for solving~(\ref{eqn:prob1}) with a convex loss function $\ell_i(\w)$. The key is to restart SCDRO and ASCDRO by using a stagewise step size scheme. We define a new objective $F_\mu(\x) = F(\x) + \mu\|\x\|^2/2$ and correspondingly $\bar F_\mu(\x) = F_\mu(\x) + \delta_{\X}(\x)$, where $\mu$ is a constant to be determined later.
With this new objective, we have the following lemma.
\begin{lemma}\label{le:kl}
Suppose that  $\ell_i(\w)$ is convex for all $i$, then for all $\x\in\X$, $\bar F_\mu(\x)$ satisfies the following Kurdyka-\L ojasiewicz (KL) inequality
%\begin{align*}
    $\dist (0,\partial \bar F_\mu(\x))^2\geq 2\mu(\bar F_\mu(\x) -\inf\limits_{\x\in\X}\bar F_\mu(\x))$.
%\end{align*}
\end{lemma}

Lemma~\ref{le:kl} allows us to obtain the convergence guarantee for convex losses. 
%According to Lemma~\ref{le:kl}, we could get the complexity of finding an $\epsilon$-optimal solution once we have the complexity of finding an $\epsilon$-stationary solution.
% \begin{align}
%     \v_{t} & =(1-\beta )\v_{t-1} + \beta \nabla f_{\lambda_{t}}(s_{t})\nabla_{\w}g_i(\x_{t})\label{eqn:uv}\\
%     u_{t} & = (1-\beta )u_{t-1} +\beta (\nabla f_{\lambda_{t}}(s_{t})\nabla_\lambda g_i(\x_{t})  +\log(s_{t}) + \rho)\notag.
% \end{align}
The idea of the restarted algorithm is to apply SCDRO and ASCDRO to the new objective $\bar F_\mu(\x)$ by adding $\mu\x_t$ to $(\nabla f_{\lambda_{t}}(s_{t})\nabla_{\w}g_i(\x_{t})^\top,\nabla f_{\lambda_{t}}(s_{t})\nabla_\lambda g_i(\x_{t})  +\log(s_{t}) + \rho)^\top$ in Eq.~(\ref{eqn:uv}) of  Algorithm~\ref{alg:SCCMA} and substituting $\z_t$ in~(\ref{eqn:stormz}) of Algorithm~\ref{alg:STORM-CCMA} by
% The idea of the boosted algorithm is to apply SCDRO and ASCDRO to the new objective $\bar F_\mu(\x)$ by substituting $\z_t =  (\v_t^\top, u_t)^\top$ with $\z_t =  (\v_t^\top, u_t)^\top+\mu\x_t$ in Algorithm~\ref{alg:SCCMA} and substituting $\z_t$ in~(\ref{eqn:stormz}) of Algorithm~\ref{alg:STORM-CCMA} by
%\begin{align*}
 $\z_t =    (\nabla f_{\lambda_t} (s_t)\v_t^\top, \nabla f_{\lambda_t}  (s_t)u_t + \log(s_t) + \rho)^\top + \mu \x_t$,
%\end{align*}
and restarting SCDRO or ASCDRO with a stagewise step size to enjoy the benefit of KL inequality of $\bar F_\mu(\x)$. It is notable that a stagewise step size is widely and commonly used in practice. The multi-stage restarted version of SCDRO and ASCDRO are shown Algorithm~\ref{alg:RSCCMA}, to which we refer as restarted-SCDRO (RSCDRO) and restarted-ASCDRO (RASCDRO). % which uses a geometrically decreasing step size in a stagewise manner.

\begin{algorithm}[t]
    \centering
    \caption{RSCDRO or RASCDRO}\label{alg:RSCCMA}%($\w_1,\epsilon_1,\mu ,c$)
    \begin{algorithmic}[1]
          \STATE $\textbf{Input: } \w_{1}\in \W, \lambda_{1}\in\R^+$, $\x_{1} = (\w_{1}^\top,\lambda_{1})^\top$
 \STATE \textbf{Initialization:} The same as in SCDRO or ASCDRO 
 \STATE Let $\Lambda_k=(\x_{k}, \v_{k}, u_{k}, s_{k})$%Draw a sample $\xi_i\sim \D$, and calculate $s_{1} = \exp(\frac{\ell_i(\w_{1})}{\lambda_{1}})$, $\v_{1} =\nabla f_{\lambda_{1}}(s_{1})\partial_{\w_{1}} g_i(\x_{1}))\in \R^d, u_{1} = \nabla f_{\lambda_{1}}(s_{1})\partial_\lambda g_i(\x_{1})  +\log(s_{1}) + \rho \in \R$
        \FOR{$k=1,\cdots, K$}
         \STATE $\Lambda_{k+1} = \text{SCDRO}(\Lambda_k,\eta_{k}, T_{k})$ or $\Lambda_{k+1} = \text{ASCDRO}(\Lambda_k,\eta_{k}, T_{k})$
         %\STATE (RASCDRO) $\Lambda_{k+1} = \text{ASCDRO}(\Lambda_k,\eta_{k}, T_{k})$
         \STATE Change $\eta_k, T_k$ according to Lemma~\ref{lem:stage_variance} or Lemma~\ref{lem:storm-var}
        \ENDFOR
    \RETURN $\x_K$
    \end{algorithmic}
\end{algorithm}

\subsection{Restarted SCDRO for Convex Problems }\label{sec:convex:sub:basic}

In this subsection, we present the convergence rate of  RSCDRO for convex losses. We first present a lemma that states $F_\mu(\x_k)$ is stagewisely decreasing. 
\begin{lemma}\label{lem:stage_variance}
Suppose Assumptions~\ref{ass:1} and~\ref{ass:2} hold, $\ell_i(\w)$ is convex for all $i$, and $F_{\mu}(\x_{1}) -\inf_{\x \in\X}F_{\mu}(\x)\leq \Delta_\mu< \infty$. Let $\epsilon_1 = \Delta_\mu$, $\epsilon_k=\epsilon_{k-1}/2$, $\beta_{k} =\min\{ \frac{\mu\epsilon_{k}}{c\sigma^2}, \frac{1}{c}\}, \eta_{k} =\min\{ \frac{\mu\epsilon_{k}}{12cL_F^2\sigma^2}, \frac{1}{12cL_F^2}\}$ and $T_{k} =\max\{\frac{384cL_F^2\sigma^2}{\mu^2\epsilon_{k}}, \frac{384cL_F^2}{\mu}\}$, where $c=384L^2_F$. Run RSCDRO, then we have
      $\E[F_{\mu}(\x_k) - \inf\limits_{\x \in\X}F_{\mu}(\x)]\leq\epsilon_k$ for each stage $k$.
\end{lemma}
%\vspace*{-0.1in}
The above lemma implies that the objective gap $\E[F_\mu(\x_k) -\inf_{\x \in\X}F_{\mu}(\x)]$ is decreased by a factor of $2$ after each stage. Based on the above lemma,  RSCDRO has the following convergence rate
% (Algorithm~\ref{alg:RSCCMA}).
%The above lemma implies that the variance of the stochastic estimator decreases between stages in Algorithm~\ref{alg:RSCCMA}. With the Lemma~\ref{le:kl} and Lemma~\ref{lem:stage_variance}, we can show that the objective gap $\E[F_\mu(\x_k) - F_\mu(\x_*)]$ is decreased by a factor of 2 after each stage. Hence we have the following convergence rate for the restarted version of SCDRO (Algorithm~\ref{alg:RSCCMA}):
\begin{theorem}
\label{thm:KL-RSCCMA}
Under the same assumptions and parameter settings as Lemma~\ref{lem:stage_variance}, after $K=\mathcal O(\log_2(\epsilon_1/\epsilon))$ stages, the output of RSCDRO satisfies $\E[F_\mu(\x_K) - \inf_{\x\in\X} F_\mu(\x)]\leq\epsilon$, and the oracle complexity is $\mathcal O(1/\mu^2\epsilon)$.
\end{theorem}

As $F_\mu(\x_K) - F_\mu(\x_*)\leq F_\mu(\x_K) - \inf_{\x\in\X} F_\mu(\x)$, where $\x_*=\argmin_{\x\in\X}F(\x)$. Therefore, if after $K$ stages it holds that $\E[F_\mu(\x_K) - \inf_{\x\in\X} F_\mu(\x)]\leq\epsilon/2$ with an oracle complexity of $\mathcal O(1/\mu^2\epsilon)$, we have  $\E[F_\mu(\x_K) - F_\mu(\x_*)]\leq\epsilon/2$ , i.e., $\E[F(\x_K) + \mu\| \x_K\|^2/2 - F(\x_*) - \mu\|\x_* \|^2/2]\leq \epsilon/2$. By Assumption~\ref{ass:1}(a) $\mathcal W$ is bounded by $R$, and then by setting $\mu = \epsilon/(2 (R^2+\tilde\lambda^2))$, with $\|\x\|^2\leq (R^2+\tilde\lambda^2)$ we have
\begin{align*}
 \E[F(\x_K) - F(\x_*)] \leq \frac{\epsilon}{2} + (2 (R^2+\tilde\lambda^2))\frac{\mu}{2} \leq \frac{\epsilon}{2} + \frac{\epsilon}{2} \leq \epsilon
\end{align*}
with an oracle complexity of $\mathcal O(1/\epsilon^3)$, i.e, the following corollary holds.
%The following corollary follows from the above theorem (please see Appendix~\ref{app:cor1} for proof).
\begin{corollary}
\label{cor:CX-RSCCMA} Let $\mu = \epsilon/(2 (R^2+\tilde\lambda^2))$. Then under the same assumptions and parameter settings as Lemma~\ref{lem:stage_variance}, after $K=\mathcal O(\log_2(\epsilon_1/\epsilon))$ stages, the output of RSCDRO satisfies $\E[F(\x_K) - \inf_{\x\in\X} F(\x)]\leq\epsilon$ and the oracle complexity is $\mathcal O(1/\epsilon^3)$.
\end{corollary}
%\vspace*{-0.05in}
\textbf{Remark: }Corollary~\ref{cor:CX-RSCCMA} shows that RSCDRO achieves an oracle complexity of $\mathcal O(1/\epsilon^3)$ for finding an $\epsilon$-optimal solution.  i.e., $\E[F(\x)-F(\x_*)]\leq\epsilon$ for the convex loss function with a geometrically decreasing step size in a stagewise manner.

\subsection{Restarted ASCDRO for Convex Problems }\label{sec:convex:sub:acceleration}
In this subsection, we establish a better convergence rate of RASCDRO for convex losses.

\begin{lemma}\label{lem:storm-var}
Suppose Assumptions~\ref{ass:1} and \ref{ass:2} hold, $\ell_i(\w)$ is convex for all $i$, and $F_{\mu}(\x_{1}) -\inf_{\x \in\X}F_{\mu}(\x)\leq \Delta_\mu< \infty$. Let $\epsilon_1 = \Delta_\mu$, $\epsilon_k=\epsilon_{k-1}/2$, $\beta_k =\min\{ \frac{\mu\epsilon_k}{c\sigma^2}, \frac{1}{c}\}, \eta_k =\min\{ \frac{\sqrt{\mu\epsilon_k}}{24cL_F\sigma^2}, \frac{1}{24cL_F^2}\}$ and $T_k =\max\{\frac{192cL_F\sigma}{\mu^{3/2}\sqrt{\epsilon_k}},\frac{192cL_F^2\sigma^2}{\mu\epsilon_k}, \frac{192c L_F^2}{\mu}\}$, where $c=768L^2_F$. Run RASCDRO, then we have
      $\E[F_{\mu}(\x_k) - \inf\nolimits_{\x \in\X}F_{\mu}(\x)]\leq\epsilon_k$  for each stage $k$.
    %\end{aligned}
%\end{equation*}
\end{lemma}
%\vspace{-0.1in}
The above lemma implies  that the objective gap $\E[F_\mu(\x_k) -\inf_{\x \in\X}F_{\mu}(\x)]$ is decreased by a factor of  $2$ after each stage. Hence we have the following convergence rate for the RASCDRO.
%The above lemma implies that the variance of the stochastic estimator decreases between stages in Algorithm~\ref{alg:RSCCMA}. With the Lemma~\ref{le:kl} and Lemma~\ref{lem:stage_variance}, we can show that the objective gap $\E[F_\mu(\x_k) - F_\mu(\x_*)]$ is decreased by a factor of 2 after each stage. Hence we have the following convergence rate for the restarted version of SCDRO (Algorithm~\ref{alg:RSCCMA}):
\begin{theorem}
\label{thm:storm-KL}
Under the same assumptions and parameter settings as Lemma~\ref{lem:storm-var}, after $K=\mathcal O(\log_2(\epsilon_1/\epsilon))$ stages, the output of RASCDRO satisfies $\E[F_\mu(\x_K) - \inf_{\x\in\X} F_\mu(\x)]\leq\epsilon$, and the oracle complexity is $\mathcal{O}\left(\max\left(1/\mu\epsilon,1/\mu^{3/2}\sqrt{\epsilon}\right)\right)$.
\end{theorem}
% It is easy to note that $F_\mu(\x_K) - F_\mu(\x_*)\leq F_\mu(\x_K) - \inf_{\x\in\X} F_\mu(\x)$, where $\x_*=\argmin_{\x\in\X}F(\x)$. Therefore, if after $K$ stages it holds that $\E[F_\mu(\x_K) - \inf_{\x\in\X} F_\mu(\x)]\leq\epsilon/2$ with a oracle complexity of $\mathcal O(1/\mu^2\epsilon)$, we have  $\E[F_\mu(\x_K) - F_\mu(\x_*)]\leq\epsilon/2$ , i.e., $\E[F(\x_K) + \mu\| \x_K\|^2/2 - F(\x_*) - \mu\|\x_* \|^2/2]\leq \epsilon/2$. Then by setting $\mu = \epsilon/(2 (R^2+\tilde\lambda^2))$, with $\|\x\|^2\leq (R^2+\tilde\lambda^2)$ we have
% \begin{align*}
%   \E[F(\x_K) - F(\x_*)] \leq \frac{\epsilon}{2} + (2 (R^2+\tilde\lambda^2))\frac{\mu}{2} \leq \frac{\epsilon}{2} + \frac{\epsilon}{2} \leq \epsilon
% \end{align*}
% with a oracle complexity of $\mathcal O(1/\epsilon^3)$. 
%\vspace*{-0.1in}
By the same method of derivation of Corollary~\ref{cor:CX-RSCCMA}, the following corollary of Theorem~\ref{thm:storm-KL} holds.
\begin{corollary}
\label{cor:RSTORM-SCCMA} Let $\mu = \epsilon/(2 (R^2+\tilde\lambda^2))$. Then under the same assumptions and parameter settings as Lemma~\ref{lem:storm-var}, after $K=\mathcal O(\log_2(\epsilon_1/\epsilon))$ stages, the output of RASCDRO satisfies $\E[F(\x_K) - \inf_{\x\in\X} F(\x)]\leq\epsilon$ and the oracle complexity is $\mathcal O(1/\epsilon^2)$.
\end{corollary}
\textbf{Remark: }Corollary~\ref{cor:RSTORM-SCCMA} shows that RASCDRO achieves the claimed oracle complexity $\mathcal O(1/\epsilon^2)$ for finding an $\epsilon$-optimal solution, % i.e., $\E[F(\x)-F(\x_*)]\leq\epsilon$ for a convex loss function. % with a geometrically decreasing step size in a stagewise manner. 
%Note that the complexity $\mathcal{O}(1/\epsilon^{2})$ 
which is optimal for solving convex smooth optimization problems~\citep{nemirovski1983problem}. Finally, we note that  a similar complexity was established in ~\citep{Zhang2020OptimalAF} for constrained convex compositional optimization problems. However, their analysis requires each level function to be convex, which does not apply to our case as the outer function $f_\lambda(\cdot)$ is non-convex.


\section{Experiments}\label{sec:exp}
In this section, we verify the effectiveness of the proposed algorithms in solving imbalanced classification problems. We show that the proposed methods outperform baselines under both the convex and non-convex settings in terms of convergence speed, and generalization performance.  In addition, we study the influence of $\rho$ to the robustness of different optimization methods in the supplement. All our results are conducted on Tesla V100.

\noindent\textbf{Baselines.} For the comparison of convergence speed, we compare with different algorithms for optimizing the same objective~(\ref{eq:pdro}), including, stochastic primal-dual algorithms, namely PG-SMD2~\citep{rafique2021weakly} for a non-convex loss, and SPD~\citep{namkoong2016stochastic} for a convex loss, Dual SGM~\citep{levy2020large} and mini-batch based SGD named FastDRO~\citep{levy2020large} for both convex and non-convex losses. For the comparison of generalization performance,
we compare with different methods for optimizing different objectives, including the traditional ERM with CE loss by SGD with momentum (SGDM), KL-regularized DRO solved by RECOVER~\citep{qi2021online}, ABSGD~\citep{qi2020attentional, li2021tilted} and CVaR-constrained, $\chi^2$-regularized/-constrained DRO optimized by FastDRO.

%We compare our baselines with PG-SMD2~\ref{} and Fast-DRO~\ref{} in both methods. PG-SMD2, XXX, FastDRO, XXX. RECOVER (XXXX) is designed for the regularzied $\lambda$, which is optimizing the same objective Equa- tion~\ref{eqn:unconstrained_pd} but with constant $\lambda$. and have been shown effective 359 for addressing the large-scale data imbalance problem. 

%We conduct experiments from two aspects. From the optimization perspective, we compare SCCMA with two constrained stochastic optimization algorithms, PG-SMD2~\ref{} and Fast-DRO~\ref{} by optimizing the same objective Equation~\ref{} both in the convex and non-convex setting. We provide the training accuracy curves both in terms of both running times and the number (\#) of processed samples. From the generalization perspective, we demonstrate the superb empirical performance of SCCMA by including more baselines and reporting the testing accuracy. RECOVER~\ref{} and ABSGD~\ref{} are optimizing the same objective Equation~\ref{} but with constant $\lambda$ and have been shown effective for addressing the large-scale data imbalance problem. 
%In addition, we also verify the robustness of SCCMA for the constraint parameter $\rho$.
\noindent\textbf{Datasets.}
We conduct experiments on four imbalanced datasets, namely CIFAR10-ST, CIFAR100-ST \citep{qi2020simple}, ImageNet-LT~\citep{liu2019large}, and iNaturalist2018~\citep{inaturalist18}. 
The original CIFAR10, CIFAR100 are balanced data, where CIFAR10 (resp. CIFAR100) has 10 (resp. 100) classes and each class has 5K (resp. 500) training images. For constructing CIFAR10-ST and CIFAR100-ST, we artificially construct imbalanced training data, where we only keep the last 100 images of each class for the first half classes, and keep other classes and the test data unchanged. ImageNet-LT is a long-tailed subset of the original ImageNet-2012 by sampling a subset following the Pareto distribution with the power value 6. It has 115.8K images from 1000 categories, which include 4980 for head class and 5 images for tail class. iNaturalist 2018 is a real-world dataset whose class-frequency follows a heavy-tail distribution. It contains 437K images from 8142 classes.

\noindent
\textbf{Models.}
For a non-convex setting (deep model), we learn ResNet20 for CIFAR10-ST, CIFAR100-ST, and ResNet50 for ImageNet-LT and iNaturalist2018, respectively.
On CIFAR10-ST, CIFAR100-ST, we optimize the network from scratch by different algorithms. For the large-scale ImageNet-LT and iNaturalist2018 datasets, we optimize the last block of the feature layers and the classifier weight with other layers frozen of a pretrained ResNet50 model. This is a common training strategy in the literature~\citep{kang2019decoupling, qi2020attentional}. For a convex setting (linear model), we freeze the feature layers of the pretrained models, and only fine-tune the last classifier weight.  The pretrained models for ImageNet-LT, CIFAR10-ST, CIFAR100-ST are trained from scratch by optimizing the standard cross-entropy (CE) loss using SGD with momentum 0.9 for 90 epochs. The pretrained ResNet50 model for iNaturalist2018 is from the released model by~\citet{kang2019decoupling}.  

\begin{figure*}[t]
\centering
    \includegraphics[width =0.23\linewidth]{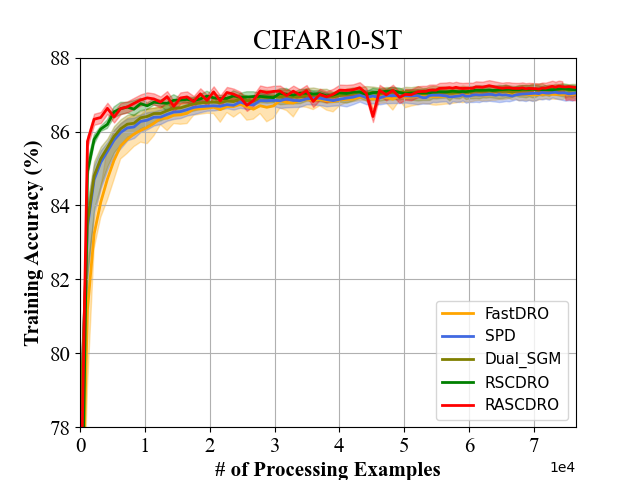}  \hspace*{-0.1in}  \
    \includegraphics[width =0.23\linewidth]{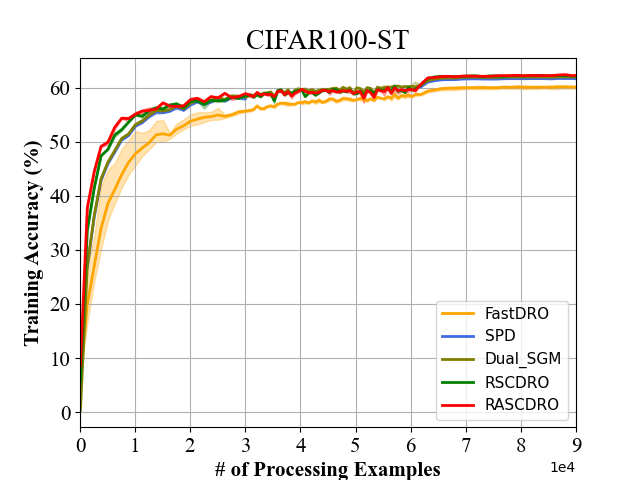} \hspace*{-0.1in}   \
    \includegraphics[width =0.23\linewidth]{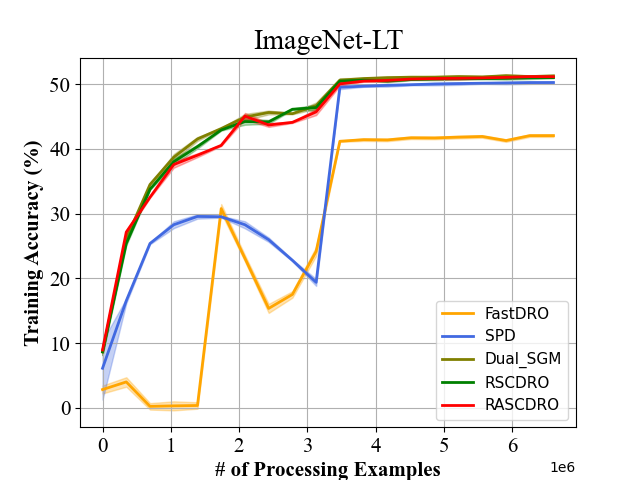}  \hspace*{-0.1in}   \
    \includegraphics[width =0.23\linewidth]{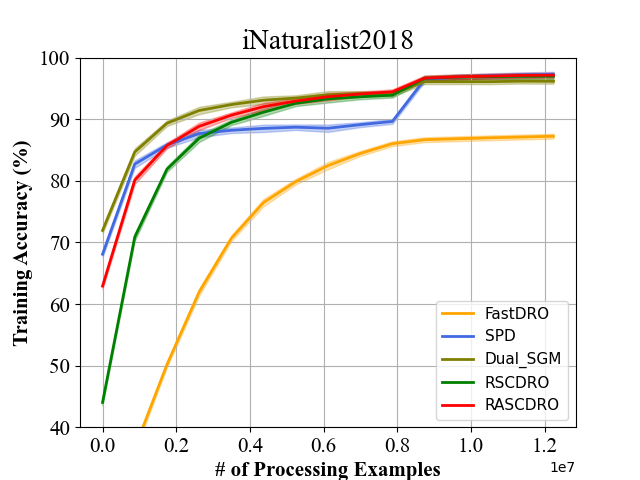}\hspace*{-0.1in}   \
    
   \includegraphics[width =0.23\linewidth]{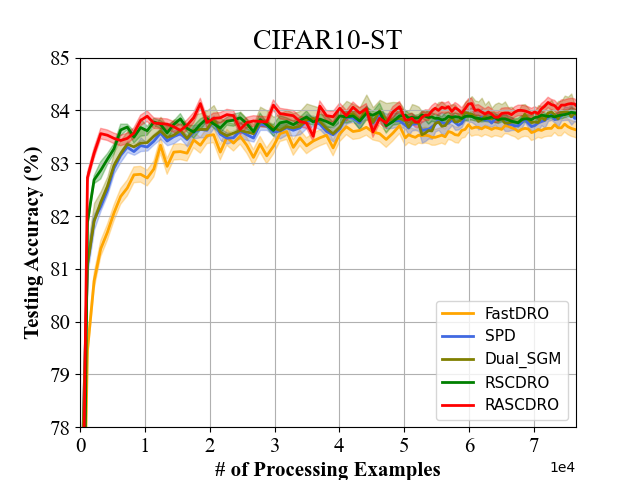}  \hspace*{-0.1in}  \ 
   \includegraphics[width =0.23\linewidth]{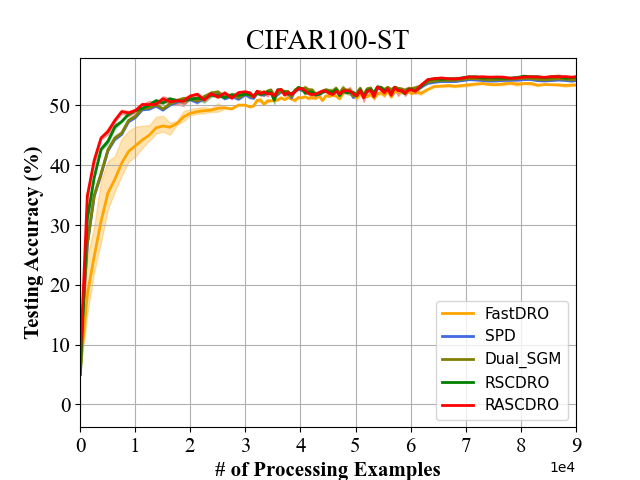} \hspace*{-0.1in}  \ 
    \includegraphics[width =0.23\linewidth]{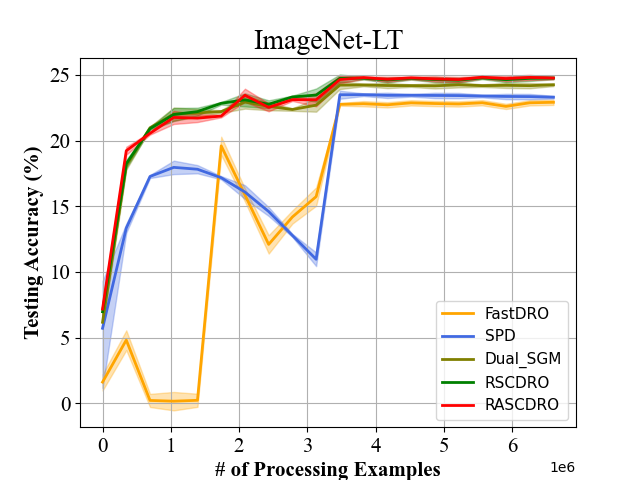}  \hspace*{-0.1in}   \
   \includegraphics[width =0.23\linewidth]{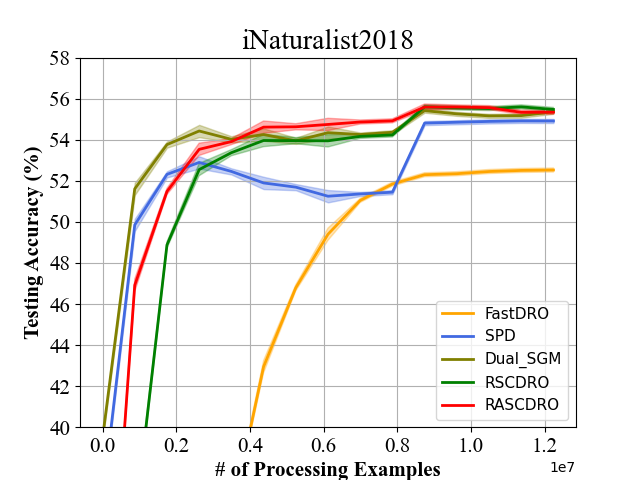} \hspace*{-0.1in}  \
    \caption{Training accuracy (\%) , Testing accuracy (\%)   vs $\#$ of processed training samples for the convex setting. $\rho$ is fixed to 0.5 on CIFAR10-ST and CIFAR100-ST, and 0.1 on ImageNet-LT and iNaturalist2018. The results are averaged over 5 independent runs.} 
    \label{fig:convex-figures}
     %\vspace{-0.05in}
%\end{figure*}
%\begin{figure*}[t]
\centering
  \includegraphics[width =0.23\linewidth]{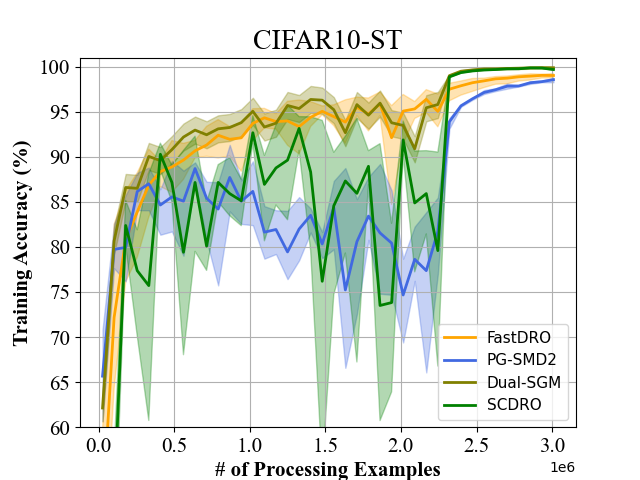}  \hspace*{-0.1in}  \
    \includegraphics[width =0.23\linewidth]{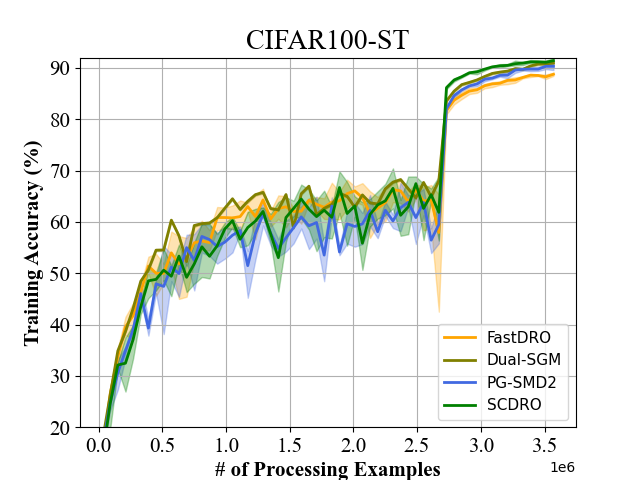}  \hspace*{-0.1in}  \
    \includegraphics[width =0.23\linewidth]{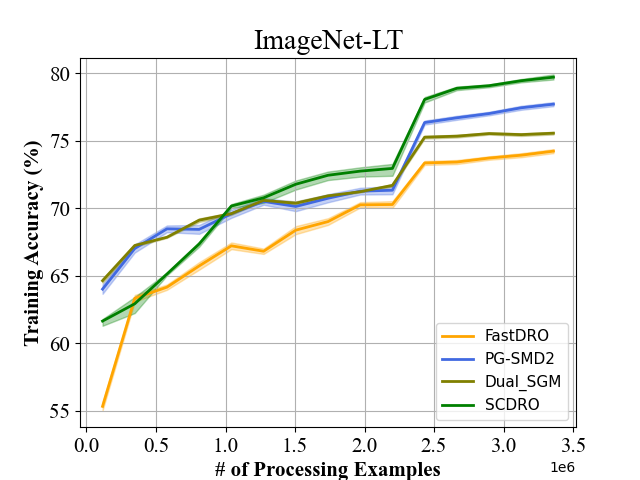}  \hspace*{-0.1in}  \
    \includegraphics[width =0.23\linewidth]{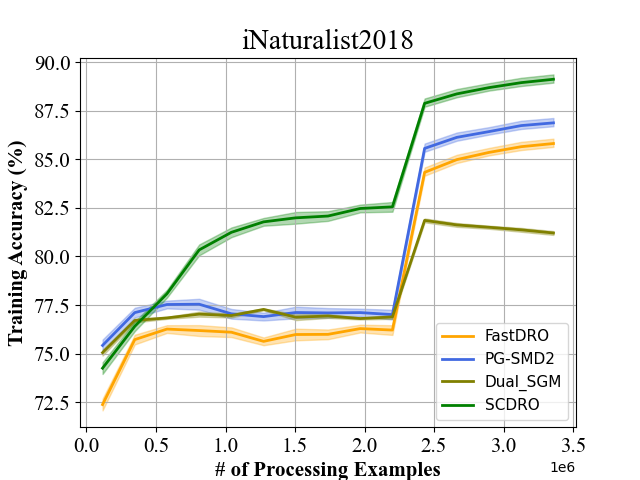}  \hspace*{-0.1in}  \
    
 \includegraphics[width =0.23\linewidth]{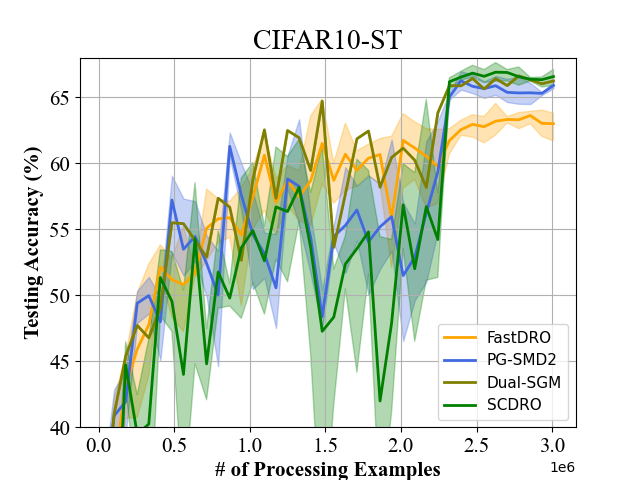}  \hspace*{-0.1in}  \
  \includegraphics[width =0.23\linewidth]{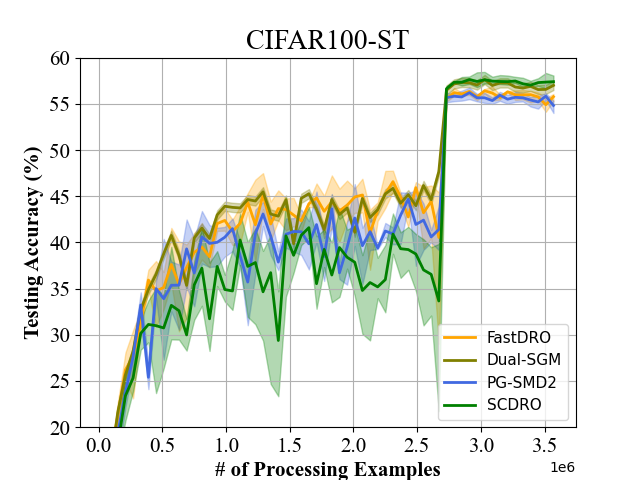}  \hspace*{-0.1in}  \
  \includegraphics[width =0.23\linewidth]{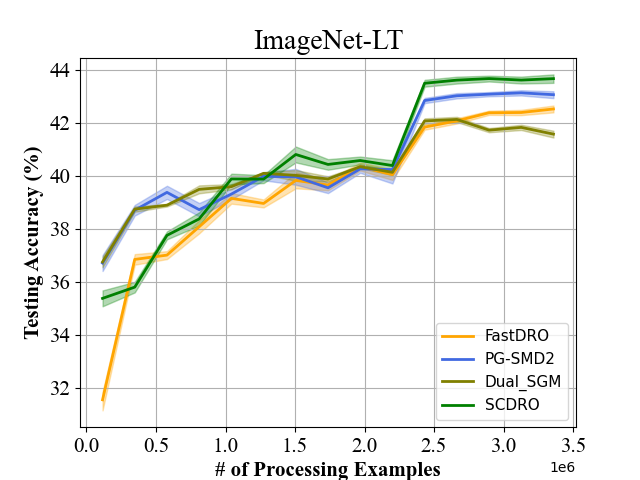}  \hspace*{-0.1in}  \
   \includegraphics[width =0.23\linewidth]{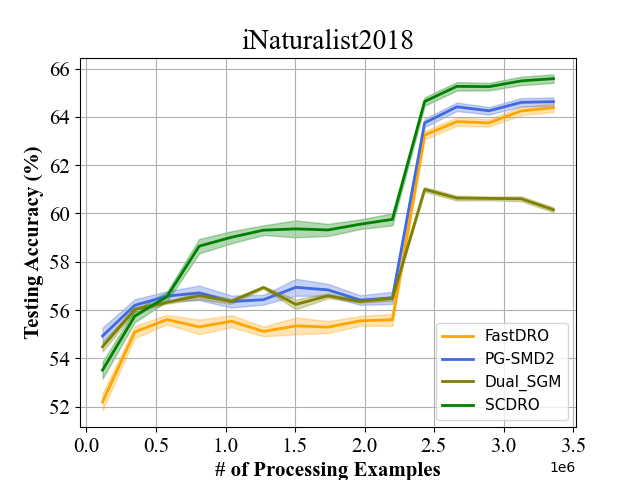}  \hspace*{-0.1in}  \

    \caption{Training accuracy (\%), Testing accuracy (\%)   vs $\#$ of processed training samples for the non-convex setting. $\rho$ is fixed to 0.5 on all datasets. The results are averaged over 5 independent runs. } 
    \label{fig:non-convex-figures}
    %\vspace{0.15in}
\end{figure*}
\noindent
\textbf{Parameters and Settings.}
For all experiments, the batch size is 128 for CIFAR10-ST and CIFAR100-ST, and 512 for ImageNet-LT and iNaturalist2018. The loss function is the CE loss. The $\lambda_0$ is set to $1e$-$3$. The (primal) learning rates for all methods are tuned in $\{0.01, 0.05, 0.1, 0.5, 1 \}$. The learning rate for updating the dual variable in PG$\_$SMD2 and SPD is tuned in $\{1e$-$5, 5e$-$5, 1e$-$4, 5e$-$4)\}$. The momentum parameter $\beta$ in our proposed algorithms and RECOVER are tuned $ \{0.1:0.1:0.9\}$. For RECOVER, the hyper-parameter $\lambda$ is tuned in $\{1, 50, 100\}$. The constrained parameter $\rho$ is tuned in $\{0.1,  0.5,  1\}$ for the comparison of generalization performance unless specified otherwise. The initial $\lambda$ and Larange multiplier in Dual SGM are both tuned in $\{0.1, 1, 10\}$. All our results are conducted on Tesla V100.

\noindent
\textbf{Convergence comparison between different baselines.}
In the convex setting,  we compare RSCDRO and RASCDRO with SPD, FastDRO and Dual SGM baselines. We report the training accuracy and testing accuracy in terms of the number ($\#$) of processing samples. We denote 1 pass of training data by 1 epoch. We run a total of 3 epochs for CIFAR10-ST and CIFAR100-ST and decay the learning rate by a factor of 10 at the end of 2nd epoch. Similarly, we run 60 epochs and decay the learning rate at the 30th epochs for the ImageNet-LT, and run 30 epochs and decay the learning rate at the 20th epoch for iNaturalist2018.  
In the nonconvex setting, we compare SCDRO with two baselines, PG-SMD2 and FastDRO. We run 120 epochs for CIFAR10-ST and CIFAR100-ST, and decay the learning rate by a factor of 10 at the 90th epoch. And we run 30 epochs for ImageNet-LT and iNaturalist2018, and decay the learning rate at the 20th epoch.

\noindent
\textbf{Results.} We first report the results for convex setting in Figures~\ref{fig:convex-figures}. It is obvious to see that RSCDRO and RASCDRO are consistently better than baselines on CIFAR10-ST, CIFAR100-ST, and ImageNet-LT. PD-SMD2 and Dual SGM have comparable results with our proposed algorithms on the iNaturalist2018 in terms of training accuracy, but is worse in terms of testing accuracy.
FastDRO has the worst performance on all the datasets. RSCDRO and RASCDRO achieve comparable results on all datasets, however, the stochastic estimator in RASCDRO requires two gradient computations per iteration, which incurs more computational cost than RSCDRO. Hence, in the non-convex setting, we focus on SCDRO. Figure~\ref{fig:non-convex-figures} reports the results for non-convex setting. We can see that SCDRO achieves the best performance on all the datasets. The margin increases on the large scale ImageNet-LT and iNaturalist2018 datasets. 
 For the three baselines, Dual SGM has better testing performance than FastDRO and PD-SGM2 on CIFAR10-ST and CIFAR100-ST. On the large scale data ImageNet-LT and iNaturalist2018, however, Dual SGM has the worst performance in terms of the testing accuracy. Furthermore, SCDRO is more stable than FastDRO and Dual SGM in different settings as the training of Dual SGM and FastDRO is comparable to SCDRO in convex settings and much worse than SCDRO in non-convex settings.

\noindent
{\bf Comparison with ERM and KL-regularized DRO.} Next, we compare our method for solving KL-constrained DRO (KL-CDRO) with 1) ERM+SGDM, and KL-regularized DRO (KL-RDRO) optimized by RECOVER, ABSGD in the non-convex setting 2) CVaR-constrained DRO, $\chi^2$-regularized DRO $\chi^2$-constrained DRO optimized by FastDRO in the convex setting. We conduct the experiments on the large-scale ImageNet-LT and iNaturalist2018 datasets. The results shown in Table~\ref{tab:convex-general} and~\ref{tab:non-convex-general} vividly demonstrate that our method for constrained DRO outperforms the ERM-based method and other popular $f$-divergence constrained/regularized DRO in different settings.

\begin{table}[t]
\centering
\caption{Testing Accuracy in Convex Setting } 
\label{tab:convex-general}
% \resizebox{0.9\textwidth}{!}{%
\begin{tabular}{c|c|c}
\toprule
        & ImageNet-LT & iNaturalist2018 \\ \midrule
%pretrained  & 41.84 &\\ \midrule
KL-Constraint + SCDRO     &  {\bf 24.08} ($\pm$ 0.01)           &     {\bf 55.63}  ($\pm$ 0.03)      \\  \midrule
CVaR-Constraint + FastDRO    &  17.23   ($\pm$ 0.03)      & 54.52     ($\pm$ 0.11)           \\  \midrule
$\chi^2$-Regularization  + FastDRO &        23.98 ($\pm$ 0.01)  &     55.03 ($\pm$ 0.03)   \\         \midrule  
$\chi^2$-Constraint + FastDRO  &    23.61 ($\pm$ 0.01)        & 53.71 ($\pm$ 0.05)    \\ \bottomrule        
\end{tabular}
%\end{table}
%\begin{table}[htbp]  
\centering
\caption{Testing Accuracy in Non-Convex Setting } 
\label{tab:non-convex-general}
% \resizebox{0.9\textwidth}{!}{
\begin{tabular}{c|c|c}
\toprule
        & ImageNet-LT & iNaturalist2018 \\ \midrule
        KL-Constraint + SCDRO   &    \textbf{43.74}        & \textbf{65.59} \\         \midrule  
%pretrained  & 41.84 &\\ \midrule
ERM+SGDM     &  43.36           & 64.42                \\  \midrule
KL-Regularization + RECOVER &        42.68    &     64.57  \\ \midrule
KL-Regularization + ABSGD &        43.44    &     65.01 \\
\bottomrule
\end{tabular}
%}
%\end{table}
\end{table}

\noindent{\bf Sensitivity to $\rho$.} We study the sensitivity of different methods to $\rho$. The results on CIFAR10-ST and CIFAR100-ST are shown in Table~\ref{tab:rho-sensitivity} in the supplement, which demonstrates that the testing performance is sensitive to $\rho$. However, our method SCDRO is better than baselines PG-SMD2 and FastDRO for different values of $\rho$.

\noindent
\begin{table}[ht]
\centering
\caption{Test accuracy (\%) of different methods for different constraint parameter $\rho$ in the non-convex setting. The results are averaged over 5 independent runs.}

\resizebox{1\textwidth}{!}{%
\begin{tabular}{c|c|ccccc}\toprule
                          &    $\rho$     & 0.01 & 0.05 & 0.1 & 0.5 & 1 \\ \hline
\multirow{3}{*}{CIFAR10-ST}  & PG-SMD2   &   67.09 ($\pm$ 0.59)     & 66.96 ($\pm$ 0.71)     & 67.12 ($\pm$ 0.61)    &  67.36 ($\pm$ 0.36)   &  67.10 ($\pm$ 0.61) \\
                          & FastDRO &   65.41 ($\pm$ 0.33)   &   66.15 ($\pm$ 0.09)   &   66.24 ($\pm$ 0.63)  &   65.98 ($\pm$ 0.45)  & 65.68 ($\pm$ 0.52)  \\
                          & SCDRO   &  \textbf{67.73} ($\pm$ 0.39)    &   \textbf{67.58} ($\pm$ 0.48)   &  \textbf{67.71} ($\pm$ 0.43)  &   \textbf{67.57} ($\pm$ 0.28)  &  \textbf{67.96} ($\pm$ 0.50) \\ \midrule
\multirow{3}{*}{CIFAR100-ST} & PG-SMD2   & 57.31 ($\pm$ 0.09)     &  56.44 ($\pm$ 0.17)    &   55.85 ($\pm$ 0.19)  & 52.68 ($\pm$ 0.40)    & 48.72 ($\pm$ 0.25)  \\
                          & FastDRO &  57.60 ($\pm$ 0.32)    &  57.20 ($\pm$ 0.42)    &  56.78 ($\pm$ 0.40)  &  55.58 ($\pm$ 0.62)   &52.39 ($\pm$ 0.31) \\
                          & SCDRO   &  \textbf{57.84} ($\pm$ 0.15)   &   \textbf{57.60} ($\pm$ 0.15)   &  \textbf{58.32} ($\pm$ 0.43)  &   \textbf{57.90} ($\pm$ 0.26)  & \textbf{57.71} ($\pm$ 0.24) \\
                          \bottomrule
\end{tabular}%
}
\label{tab:rho-sensitivity}
\end{table}

\section{Conclusions}
In this paper, we proposed dual-free stochastic algorithms for solving KL-constrained distributionally robust optimization problems for both convex and non-convex losses. The proposed algorithms have nearly optimal complexity in both settings. Empirical studies vividly demonstrate the effectiveness of the proposed algorithm for solving non-convex and convex constrained DRO problems.

\subsubsection*{Acknowledgments}
Q. Qi and T. Yang are partially supported by NSF Career Award \#1844403, NSF Grant \#2110545, and NSF-Amazon Joint Grant \#2147253.

% \usepackage{graphicx}

% \bibliography{cdro.bib}
% \bibliographystyle{icml_2022}
% 

%\bibliographystyle{}
\bibliographystyle{tmlr}
\bibliography{cdro.bib}

\begin{thebibliography}{58}
\providecommand{\natexlab}[1]{#1}
\providecommand{\url}[1]{\texttt{#1}}
\expandafter\ifx\csname urlstyle\endcsname\relax
  \providecommand{\doi}[1]{doi: #1}\else
  \providecommand{\doi}{doi: \begingroup \urlstyle{rm}\Url}\fi

\bibitem[Ahmadi-Javid(2012)]{ahmadi2012entropic}
Amir Ahmadi-Javid.
\newblock Entropic value-at-risk: A new coherent risk measure.
\newblock \emph{Journal of Optimization Theory and Applications}, 155:\penalty0
  1105--1123, 2012.

\bibitem[Alacaoglu et~al.(2022)Alacaoglu, Cevher, and
  Wright]{alacaoglu2022complexity}
Ahmet Alacaoglu, Volkan Cevher, and Stephen~J Wright.
\newblock On the complexity of a practical primal-dual coordinate method.
\newblock \emph{arXiv preprint arXiv:2201.07684}, 2022.

\bibitem[Arjevani et~al.(2019)Arjevani, Carmon, Duchi, Foster, Srebro, and
  Woodworth]{arjevani2019lower}
Yossi Arjevani, Yair Carmon, John~C Duchi, Dylan~J Foster, Nathan Srebro, and
  Blake Woodworth.
\newblock Lower bounds for non-convex stochastic optimization.
\newblock \emph{arXiv preprint arXiv:1912.02365}, 2019.

\bibitem[Ben-Tal et~al.(2013)Ben-Tal, Den~Hertog, De~Waegenaere, Melenberg, and
  Rennen]{ben2013robust}
Aharon Ben-Tal, Dick Den~Hertog, Anja De~Waegenaere, Bertrand Melenberg, and
  Gijs Rennen.
\newblock Robust solutions of optimization problems affected by uncertain
  probabilities.
\newblock \emph{Management Science}, 59\penalty0 (2):\penalty0 341--357, 2013.

\bibitem[Bertsimas et~al.(2018)Bertsimas, Gupta, and Kallus]{bertsimas2018data}
Dimitris Bertsimas, Vishal Gupta, and Nathan Kallus.
\newblock Data-driven robust optimization.
\newblock \emph{Mathematical Programming}, 167\penalty0 (2):\penalty0 235--292,
  2018.

\bibitem[Boyd et~al.(2004)Boyd, Boyd, and Vandenberghe]{boyd2004convex}
Stephen Boyd, Stephen~P Boyd, and Lieven Vandenberghe.
\newblock \emph{Convex optimization}.
\newblock Cambridge university press, 2004.

\bibitem[Chen \& Paschalidis(2018)Chen and Paschalidis]{chen2018robust}
Ruidi Chen and Ioannis~C Paschalidis.
\newblock A robust learning approach for regression models based on
  distributionally robust optimization.
\newblock \emph{Journal of Machine Learning Research}, 19\penalty0 (13), 2018.

\bibitem[Chen et~al.(2021)Chen, Sun, and Yin]{Chen_2021}
Tianyi Chen, Yuejiao Sun, and Wotao Yin.
\newblock Solving stochastic compositional optimization is nearly as easy as
  solving stochastic optimization.
\newblock \emph{{IEEE} Transactions on Signal Processing}, 69:\penalty0
  4937--4948, 2021.
\newblock \doi{10.1109/tsp.2021.3092377}.
\newblock URL \url{https://doi.org/10.1109%2Ftsp.2021.3092377}.

\bibitem[Chen et~al.(2020)Chen, Kornblith, Norouzi, and Hinton]{chen2020simple}
Ting Chen, Simon Kornblith, Mohammad Norouzi, and Geoffrey Hinton.
\newblock A simple framework for contrastive learning of visual
  representations.
\newblock In \emph{Proceedings of the 37th International Conference on Machine
  Learning}, pp.\  1597--1607, 2020.

\bibitem[Cutkosky \& Orabona(2019)Cutkosky and Orabona]{cutkosky2019momentum}
Ashok Cutkosky and Francesco Orabona.
\newblock Momentum-based variance reduction in non-convex sgd.
\newblock \emph{Advances in Neural Information Processing Systems},
  32:\penalty0 15236--15245, 2019.

\bibitem[Delage \& Ye(2010)Delage and Ye]{delage2010distributionally}
Erick Delage and Yinyu Ye.
\newblock Distributionally robust optimization under moment uncertainty with
  application to data-driven problems.
\newblock \emph{Operations research}, 58\penalty0 (3):\penalty0 595--612, 2010.

\bibitem[Deng et~al.(2020)Deng, Kamani, and Mahdavi]{deng2020distributionally}
Yuyang Deng, Mohammad~Mahdi Kamani, and Mehrdad Mahdavi.
\newblock Distributionally robust federated averaging.
\newblock \emph{Advances in Neural Information Processing Systems}, 33, 2020.

\bibitem[Dentcheva et~al.(2017)Dentcheva, Penev, and
  Ruszczynski]{RePEc:spr:aistmt:v:69:y:2017:i:4:d:10.1007_s10463-016-0559-8}
Darinka Dentcheva, Spiridon Penev, and Andrzej Ruszczynski.
\newblock Statistical estimation of composite risk functionals and risk
  optimization problems.
\newblock \emph{Annals of the Institute of Statistical Mathematics},
  69\penalty0 (4):\penalty0 737--760, 2017.
\newblock URL
  \url{https://EconPapers.repec.org/RePEc:spr:aistmt:v:69:y:2017:i:4:d:10.1007_s10463-016-0559-8}.

\bibitem[Duchi et~al.(2016)Duchi, Glynn, and Namkoong]{duchi2016statistics}
C.~John Duchi, W.~Peter Glynn, and Hongseok Namkoong.
\newblock Statistics of robust optimization: A generalized empirical likelihood
  approach.
\newblock \emph{Mathematics of Operations Research}, 2016.

\bibitem[Duchi \& Namkoong(2021)Duchi and Namkoong]{duchi2021learning}
John~C Duchi and Hongseok Namkoong.
\newblock Learning models with uniform performance via distributionally robust
  optimization.
\newblock \emph{The Annals of Statistics}, 49\penalty0 (3):\penalty0
  1378--1406, 2021.

\bibitem[Ghadimi et~al.(2020)Ghadimi, Ruszczynski, and Wang]{ghadimi2020single}
Saeed Ghadimi, Andrzej Ruszczynski, and Mengdi Wang.
\newblock A single timescale stochastic approximation method for nested
  stochastic optimization.
\newblock \emph{SIAM Journal on Optimization}, 30\penalty0 (1):\penalty0
  960--979, 2020.

\bibitem[Goel et~al.(2022)Goel, Bansal, Bhatia, Rossi, Vinay, and
  Grover]{goel2022cyclip}
Shashank Goel, Hritik Bansal, Sumit Bhatia, Ryan Rossi, Vishwa Vinay, and
  Aditya Grover.
\newblock Cyclip: Cyclic contrastive language-image pretraining.
\newblock \emph{Advances in Neural Information Processing Systems},
  35:\penalty0 6704--6719, 2022.

\bibitem[Hinton et~al.(2015)Hinton, Vinyals, and Dean]{HinVin15Distilling}
Geoffrey Hinton, Oriol Vinyals, and Jeff Dean.
\newblock Distilling the knowledge in a neural network.
\newblock \emph{arXiv preprint arXiv:1503.02531}, 2015.
\newblock URL \url{https://arxiv.org/abs/1503.02531v1}.

\bibitem[Hu et~al.(2021)Hu, Chen, and He]{NEURIPS2021_b986700c}
Yifan Hu, Xin Chen, and Niao He.
\newblock On the bias-variance-cost tradeoff of stochastic optimization.
\newblock In M.~Ranzato, A.~Beygelzimer, Y.~Dauphin, P.S. Liang, and J.~Wortman
  Vaughan (eds.), \emph{Advances in Neural Information Processing Systems},
  volume~34, pp.\  22119--22131. Curran Associates, Inc., 2021.
\newblock URL
  \url{https://proceedings.neurips.cc/paper/2021/file/b986700c627db479a4d9460b75de7222-Paper.pdf}.

\bibitem[Huang et~al.(2020)Huang, Gao, Pei, and Huang]{huang2020accelerated}
Feihu Huang, Shangqian Gao, Jian Pei, and Heng Huang.
\newblock Accelerated zeroth-order momentum methods from mini to minimax
  optimization.
\newblock \emph{arXiv e-prints}, pp.\  arXiv--2008, 2020.

\bibitem[{iNaturalist} 2018 competition dataset()]{inaturalist18}
{iNaturalist} 2018 competition dataset.
\newblock {iNaturalist} 2018 competition dataset.
\newblock ~\url{https://github.com/visipedia/inat_comp/tree/master/2018}, 2018.

\bibitem[Jin et~al.(2021)Jin, Zhang, Wang, and Wang]{jin2021non}
Jikai Jin, Bohang Zhang, Haiyang Wang, and Liwei Wang.
\newblock Non-convex distributionally robust optimization: Non-asymptotic
  analysis.
\newblock \emph{Advances in Neural Information Processing Systems}, 34, 2021.

\bibitem[Juditsky et~al.(2011)Juditsky, Nemirovski, and
  Tauvel]{juditsky2011solving}
Anatoli Juditsky, Arkadi Nemirovski, and Claire Tauvel.
\newblock Solving variational inequalities with stochastic mirror-prox
  algorithm.
\newblock \emph{Stochastic Systems}, 1\penalty0 (1):\penalty0 17--58, 2011.

\bibitem[Kang et~al.(2019)Kang, Xie, Rohrbach, Yan, Gordo, Feng, and
  Kalantidis]{kang2019decoupling}
Bingyi Kang, Saining Xie, Marcus Rohrbach, Zhicheng Yan, Albert Gordo, Jiashi
  Feng, and Yannis Kalantidis.
\newblock Decoupling representation and classifier for long-tailed recognition.
\newblock \emph{arXiv preprint arXiv:1910.09217}, 2019.

\bibitem[Levy et~al.(2020)Levy, Carmon, Duchi, and Sidford]{levy2020large}
Daniel Levy, Yair Carmon, John~C Duchi, and Aaron Sidford.
\newblock Large-scale methods for distributionally robust optimization.
\newblock \emph{Advances in Neural Information Processing Systems}, 33, 2020.

\bibitem[Li et~al.(2021{\natexlab{a}})Li, Selvaraju, Gotmare, Joty, Xiong, and
  Hoi]{li2021align}
Junnan Li, Ramprasaath Selvaraju, Akhilesh Gotmare, Shafiq Joty, Caiming Xiong,
  and Steven Chu~Hong Hoi.
\newblock Align before fuse: Vision and language representation learning with
  momentum distillation.
\newblock \emph{Advances in neural information processing systems},
  34:\penalty0 9694--9705, 2021{\natexlab{a}}.

\bibitem[Li et~al.(2020)Li, Beirami, Sanjabi, and Smith]{li2020tilted}
Tian Li, Ahmad Beirami, Maziar Sanjabi, and Virginia Smith.
\newblock Tilted empirical risk minimization.
\newblock In \emph{International Conference on Learning Representations}, 2020.

\bibitem[Li et~al.(2021{\natexlab{b}})Li, Beirami, Sanjabi, and
  Smith]{li2021tilted}
Tian Li, Ahmad Beirami, Maziar Sanjabi, and Virginia Smith.
\newblock On tilted losses in machine learning: Theory and applications.
\newblock \emph{arXiv preprint arXiv:2109.06141}, 2021{\natexlab{b}}.

\bibitem[Liu et~al.(2019)Liu, Miao, Zhan, Wang, Gong, and Yu]{liu2019large}
Ziwei Liu, Zhongqi Miao, Xiaohang Zhan, Jiayun Wang, Boqing Gong, and Stella~X
  Yu.
\newblock Large-scale long-tailed recognition in an open world.
\newblock In \emph{Proceedings of the IEEE/CVF Conference on Computer Vision
  and Pattern Recognition}, pp.\  2537--2546, 2019.

\bibitem[Luo et~al.(2020)Luo, Ye, Huang, and Zhang]{luo2020stochastic}
Luo Luo, Haishan Ye, Zhichao Huang, and Tong Zhang.
\newblock Stochastic recursive gradient descent ascent for stochastic
  nonconvex-strongly-concave minimax problems.
\newblock \emph{Advances in Neural Information Processing Systems}, 33, 2020.

\bibitem[Namkoong \& Duchi(2016)Namkoong and Duchi]{namkoong2016stochastic}
Hongseok Namkoong and John~C Duchi.
\newblock Stochastic gradient methods for distributionally robust optimization
  with f-divergences.
\newblock In \emph{NIPS}, volume~29, pp.\  2208--2216, 2016.

\bibitem[Namkoong \& Duchi(2017)Namkoong and Duchi]{namkoong2017variance}
Hongseok Namkoong and John~C Duchi.
\newblock Variance-based regularization with convex objectives.
\newblock In \emph{Advances in neural information processing systems}, pp.\
  2971--2980, 2017.

\bibitem[Nedi{\'c} \& Ozdaglar(2009)Nedi{\'c} and
  Ozdaglar]{nedic2009subgradient}
Angelia Nedi{\'c} and Asuman Ozdaglar.
\newblock Subgradient methods for saddle-point problems.
\newblock \emph{Journal of optimization theory and applications}, 142\penalty0
  (1):\penalty0 205--228, 2009.

\bibitem[Nemirovski et~al.(2009)Nemirovski, Juditsky, Lan, and
  Shapiro]{nemirovski2009robust}
Arkadi Nemirovski, Anatoli Juditsky, Guanghui Lan, and Alexander Shapiro.
\newblock Robust stochastic approximation approach to stochastic programming.
\newblock \emph{SIAM Journal on optimization}, 19\penalty0 (4):\penalty0
  1574--1609, 2009.

\bibitem[Nemirovsky \& Yudin(1983)Nemirovsky and Yudin]{nemirovski1983problem}
A.~S. Nemirovsky and D.~B. Yudin.
\newblock \emph{Problem Complexity and Method Efficiency in Optimization}.
\newblock A Wiley-Interscience publication. Wiley, 1983.
\newblock ISBN 9780471103455.
\newblock URL \url{https://books.google.com/books?id=6ULvAAAAMAAJ}.

\bibitem[Qi et~al.(2020{\natexlab{a}})Qi, Xu, Jin, Yin, and
  Yang]{qi2020attentional}
Qi~Qi, Yi~Xu, Rong Jin, Wotao Yin, and Tianbao Yang.
\newblock Attentional biased stochastic gradient for imbalanced classification.
\newblock \emph{arXiv preprint arXiv:2012.06951}, 2020{\natexlab{a}}.

\bibitem[Qi et~al.(2020{\natexlab{b}})Qi, Yan, Wu, Wang, and
  Yang]{qi2020simple}
Qi~Qi, Yan Yan, Zixuan Wu, Xiaoyu Wang, and Tianbao Yang.
\newblock A simple and effective framework for pairwise deep metric learning.
\newblock In \emph{Computer Vision--ECCV 2020: 16th European Conference,
  Glasgow, UK, August 23--28, 2020, Proceedings, Part XXVII 16}, pp.\
  375--391. Springer, 2020{\natexlab{b}}.

\bibitem[Qi et~al.(2021)Qi, Guo, Xu, Jin, and Yang]{qi2021online}
Qi~Qi, Zhishuai Guo, Yi~Xu, Rong Jin, and Tianbao Yang.
\newblock An online method for a class of distributionally robust optimization
  with non-convex objectives.
\newblock \emph{Advances in Neural Information Processing Systems}, 34, 2021.

\bibitem[Qiu et~al.(2023{\natexlab{a}})Qiu, Hu, Yuan, Zhou, Zhang, and
  Yang]{DBLP:journals/corr/abs-2305-11965}
Zi{-}Hao Qiu, Quanqi Hu, Zhuoning Yuan, Denny Zhou, Lijun Zhang, and Tianbao
  Yang.
\newblock Not all semantics are created equal: Contrastive self-supervised
  learning with automatic temperature individualization.
\newblock In \emph{Proceedings of International Conference on Machine
  Learning}, volume abs/2305.11965, 2023{\natexlab{a}}.
\newblock \doi{10.48550/arXiv.2305.11965}.
\newblock URL \url{https://doi.org/10.48550/arXiv.2305.11965}.

\bibitem[Qiu et~al.(2023{\natexlab{b}})Qiu, Hu, Yuan, Zhou, Zhang, and
  Yang]{qiu2023not}
Zi-Hao Qiu, Quanqi Hu, Zhuoning Yuan, Denny Zhou, Lijun Zhang, and Tianbao
  Yang.
\newblock Not all semantics are created equal: Contrastive self-supervised
  learning with automatic temperature individualization.
\newblock \emph{arXiv preprint arXiv:2305.11965}, 2023{\natexlab{b}}.

\bibitem[Radford et~al.(2021)Radford, Kim, Hallacy, Ramesh, Goh, Agarwal,
  Sastry, Askell, Mishkin, Clark, et~al.]{radford2021learning}
Alec Radford, Jong~Wook Kim, Chris Hallacy, Aditya Ramesh, Gabriel Goh,
  Sandhini Agarwal, Girish Sastry, Amanda Askell, Pamela Mishkin, Jack Clark,
  et~al.
\newblock Learning transferable visual models from natural language
  supervision.
\newblock In \emph{International conference on machine learning}, pp.\
  8748--8763. PMLR, 2021.

\bibitem[Rafique et~al.(2021)Rafique, Liu, Lin, and Yang]{rafique2021weakly}
Hassan Rafique, Mingrui Liu, Qihang Lin, and Tianbao Yang.
\newblock Weakly-convex--concave min--max optimization: provable algorithms and
  applications in machine learning.
\newblock \emph{Optimization Methods and Software}, pp.\  1--35, 2021.

\bibitem[Rahimian \& Mehrotra(2019)Rahimian and
  Mehrotra]{rahimian2019distributionally}
Hamed Rahimian and Sanjay Mehrotra.
\newblock Distributionally robust optimization: A review.
\newblock \emph{arXiv preprint arXiv:1908.05659}, 2019.

\bibitem[Rockafellar \& Wets(1998)Rockafellar and
  Wets]{rockafellar1998variational}
RT~Rockafellar and RJB Wets.
\newblock Variational analysis springer.
\newblock \emph{MR1491362}, 1998.

\bibitem[Song et~al.(2021)Song, Wright, and Diakonikolas]{song2021variance}
Chaobing Song, Stephen~J Wright, and Jelena Diakonikolas.
\newblock Variance reduction via primal-dual accelerated dual averaging for
  nonsmooth convex finite-sums.
\newblock In \emph{International Conference on Machine Learning}, pp.\
  9824--9834. PMLR, 2021.

\bibitem[Staib \& Jegelka(2019)Staib and Jegelka]{staib2019distributionally}
Matthew Staib and Stefanie Jegelka.
\newblock Distributionally robust optimization and generalization in kernel
  methods.
\newblock \emph{Advances in Neural Information Processing Systems},
  32:\penalty0 9134--9144, 2019.

\bibitem[Tran-Dinh et~al.(2020)Tran-Dinh, Liu, and Nguyen]{tran2020hybrid}
Quoc Tran-Dinh, Deyi Liu, and Lam~M Nguyen.
\newblock Hybrid variance-reduced sgd algorithms for minimax problems with
  nonconvex-linear function.
\newblock In \emph{NeurIPS}, 2020.

\bibitem[Udell et~al.(2014)Udell, Mohan, Zeng, Hong, Diamond, and
  Boyd]{udell2014convex}
Madeleine Udell, Karanveer Mohan, David Zeng, Jenny Hong, Steven Diamond, and
  Stephen Boyd.
\newblock Convex optimization in julia.
\newblock In \emph{2014 First Workshop for High Performance Technical Computing
  in Dynamic Languages}, pp.\  18--28. IEEE, 2014.

\bibitem[Wang et~al.(2021)Wang, Gao, and Xie]{wang2021sinkhorn}
Jie Wang, Rui Gao, and Yao Xie.
\newblock Sinkhorn distributionally robust optimization.
\newblock \emph{arXiv preprint arXiv:2109.11926}, 2021.

\bibitem[Wang et~al.(2017)Wang, Fang, and Liu]{wang2017stochastic}
Mengdi Wang, Ethan~X Fang, and Han Liu.
\newblock Stochastic compositional gradient descent: algorithms for minimizing
  compositions of expected-value functions.
\newblock \emph{Mathematical Programming}, 161\penalty0 (1-2):\penalty0
  419--449, 2017.

\bibitem[Xu et~al.(2019)Xu, Jin, and Yang]{xu2019non}
Yi~Xu, Rong Jin, and Tianbao Yang.
\newblock Non-asymptotic analysis of stochastic methods for non-smooth
  non-convex regularized problems.
\newblock In \emph{Proceedings of the 33rd International Conference on Neural
  Information Processing Systems}, pp.\  2630--2640, 2019.

\bibitem[Yan et~al.(2019)Yan, Xu, Lin, Zhang, and Yang]{yan2019stochastic}
Yan Yan, Yi~Xu, Qihang Lin, Lijun Zhang, and Tianbao Yang.
\newblock Stochastic primal-dual algorithms with faster convergence than
  $\mathcal{O}(1/\sqrt {T})$ for problems without bilinear structure.
\newblock \emph{arXiv preprint arXiv:1904.10112}, 2019.

\bibitem[Yan et~al.(2020)Yan, Xu, Lin, Liu, and Yang]{yan2020sharp}
Yan Yan, Yi~Xu, Qihang Lin, Wei Liu, and Tianbao Yang.
\newblock Optimal epoch stochastic gradient descent ascent methods for min-max
  optimization.
\newblock In \emph{Conference on Neural Information Processing Systems}, 2020.

\bibitem[Yuan et~al.(2022)Yuan, Wu, Qiu, Du, Zhang, Zhou, and
  Yang]{yuan2022provable}
Zhuoning Yuan, Yuexin Wu, Zi-Hao Qiu, Xianzhi Du, Lijun Zhang, Denny Zhou, and
  Tianbao Yang.
\newblock Provable stochastic optimization for global contrastive learning:
  Small batch does not harm performance.
\newblock In \emph{International Conference on Machine Learning}, pp.\
  25760--25782. PMLR, 2022.

\bibitem[Zhang \& Xiao(2019)Zhang and Xiao]{zhang2019stochastic}
Junyu Zhang and Lin Xiao.
\newblock A stochastic composite gradient method with incremental variance
  reduction.
\newblock In \emph{Advances in Neural Information Processing Systems}, pp.\
  9075--9085, 2019.

\bibitem[Zhang \& Lan(2021)Zhang and Lan]{Zhang2020OptimalAF}
Zhe Zhang and Guanghui Lan.
\newblock Optimal algorithms for convex nested stochastic composite
  optimization.
\newblock \emph{ArXiv e-prints}, arXiv:2011.10076, 2021.

\bibitem[Zhou et~al.(2019)Zhou, Wang, Ji, Liang, and Tarokh]{zhou2019momentum}
Yi~Zhou, Zhe Wang, Kaiyi Ji, Yingbin Liang, and Vahid Tarokh.
\newblock Momentum schemes with stochastic variance reduction for nonconvex
  composite optimization.
\newblock \emph{arXiv preprint arXiv:1902.02715}, 2019.

\bibitem[Zhu et~al.(2019)Zhu, Li, Wang, Gong, and Yang]{zhu2019robust}
Dixian Zhu, Zhe Li, Xiaoyu Wang, Boqing Gong, and Tianbao Yang.
\newblock A robust zero-sum game framework for pool-based active learning.
\newblock In \emph{The 22nd international conference on artificial intelligence
  and statistics}, pp.\  517--526. PMLR, 2019.

\end{thebibliography}
\newpage
\appendix
\onecolumn

\section{Preliminary Lemmas}
\begin{lemma}
\label{lem:smooth_f_lambda}
For $q \geq 1$, $f_\lambda (q) = \lambda \log(q) +\lambda\rho$ is $L_{f_\lambda}$-Lipschitz continuous and $L_{\nabla f_\lambda}$-smooth, where $L_{\nabla f_\lambda} = L_{f_\lambda} = \lambda$.
\end{lemma}
\textbf{Remark: } $g_i(\w,\lambda) = \exp(\frac{\ell_i(\w)}{\lambda}) \geq 1$ as $\lambda \geq \lambda_0 \in \R^+ $ and $\ell_i(\w) \geq 0$ in problem~(\ref{eqn:prob1}). Thus $ g(\x)=\frac{1}{n}\sum_{i=1}^n  g_i(\w,\lambda)\geq 1$. Then by this lemma we have $ \|\nabla f_\lambda (g(\x))\|\leq \lambda$ and $\|\nabla f_{\lambda}(g(\x_1)) - \nabla f_{\lambda}(g(\x_2))\|  \leq \lambda \|g(\x_1) -g(\x_2)\|$ for $\x, \x_1, \x_2 \in \X$.
\begin{proof}
For any $q\geq1$, we have
    $$\nabla f_\lambda (q) = \frac{\lambda}{q}\leq \lambda$$
And for any $q_1,q_2\geq1$, we have
    $$\|\nabla f_{\lambda}(q_1) - \nabla f_{\lambda}(q_2)\| \leq \left\|\frac{\lambda}{q_1} - \frac{\lambda}{q_2}\right\| \leq \left\|\frac{(q_1-q_2)\lambda}{q_1q_2}\right\| \leq \lambda \|q_1 -q_2\|$$
This complete the proof.

\end{proof}

\begin{lemma}
\label{lem:g_pty}
Let $L_A =  \exp(\frac{C}{\lambda_0})(\frac{G^2}{\lambda_0^2} + \frac{L}{\lambda_0} )$, $L_B= \exp(\frac{C}{\lambda_0})(\frac{CG}{\lambda_0^3} + \frac{G}{\lambda^2_0} )$, $L_C=  \exp(\frac{C}{\lambda_0})(\frac{CG+\lambda_0G}{\lambda_0^3})$ and $L_D= \exp(\frac{C}{\lambda_0})(\frac{C^2+2\lambda_0C}{\lambda_0^4})$. $g_i(\w,\lambda)$ is $L_{g}$-Lipschtz continuous and $L_{\nabla g}$-smooth in terms of $(\w,\lambda)$, where $L_{g} = \exp(\frac{C}{\lambda_0})(\frac{G}{\lambda_0} + \frac{C}{\lambda_0^2})$ and $L_{\nabla_g} =\sqrt{L_A^2+L_B^2 +L_C^2+L_D^2}$, 
\end{lemma}
\begin{proof}
The gradient of $g_i(\w,\lambda)$ is given as
\begin{align*}
    \nabla_{\w,\lambda} g_i(\w,\lambda)^\top &= (\nabla_{\w} g_i(\w,\lambda)^\top,\nabla_{\lambda} g_i(\w,\lambda)) \\&= \left(\exp\left(\frac{\ell_i(\w)}{\lambda}\right) \frac{\nabla_\w \ell_{i}(\w)}{\lambda}^\top, -\exp\left(\frac{\ell_i(\w)}{\lambda}\right)\frac{\ell_{i}(\w)}{\lambda^2}\right).
\end{align*}
Then by Assumption~\ref{ass:1}, we have 
\begin{align*}
    \|\nabla_{\w,\lambda} g_i(\w,\lambda)\| &\leq \exp\left(\frac{\ell_i(\w)}{\lambda}\right)\left(\left\| \frac{\nabla_\w \ell_{i}(\w)}{\lambda}\right\| + \frac{\ell_i(\w)}{\lambda^2}\right) \\&\overset{\lambda\geq \lambda_0}{\leq}\exp\left(\frac{C}{\lambda_0}\right)\left(\frac{G}{\lambda_0} + \frac{C}{\lambda_0^2}\right).
\end{align*}
Thus, $L_g =\exp\left(\frac{C}{\lambda_0}\right)\left(\frac{G}{\lambda_0} + \frac{C}{\lambda_0^2}\right) $. 

\noindent
For for all $(\w,\lambda), (\w',\lambda') \in \X$, we have
\begin{align*}
   & \left\|\nabla_{\w,\lambda} g_i(\w,\lambda) - \nabla_{\w,\lambda} g_i(\w',\lambda')\right\|^2 \\
    &\quad\leq  \left\|\exp\left(\frac{\ell_i(\w)}{\lambda}\right) \frac{\nabla_\w \ell_{i}(\w)}{\lambda} +\exp\left(\frac{\ell_i(\w')}{\lambda'}\right) \frac{\nabla_{\w} \ell_{i}(\w')}{\lambda'}\right\|^2\\\nonumber&\quad\quad+ 
    \left\|\exp\left(\frac{\ell_i(\w)}{\lambda}\right)\frac{\ell_{i}(\w)}{\lambda^2} -\exp\left(\frac{\ell_i(\w')}{\lambda'}\right)\frac{\ell_{i}(\w')}{\lambda'^2}\right\|^2 
    %\\&\quad=\left\|\exp\left(\frac{\ell_i(\w)}{\lambda}\right) \frac{\nabla_\w \ell_{i}(\w)}{\lambda} - \exp\left(\frac{\ell_i(\w')}{\lambda}\right) \frac{\nabla_{\w} \ell_{i}(\w')}{\lambda}  + \exp\left(\frac{\ell_i(\w')}{\lambda}\right) \frac{\nabla_{\w} \ell_{i}(\w')}{\lambda}  - \exp\left(\frac{\ell_i(\w')}{\lambda'}\right) \frac{\nabla_{\w} \ell_{i}(\w')}{\lambda'}\right\|^2 \\
    % \\&\quad\quad+  \left\|\exp\left(\frac{\ell_i(\w)}{\lambda}\right)\frac{\ell_{i}(\w)}{\lambda^2} - \exp\left(\frac{\ell_i(\w')}{\lambda}\right)\frac{\ell_{i}(\w')}{\lambda^2} + \exp\left(\frac{\ell_i(\w')}{\lambda}\right)\frac{\ell_{i}(\w'
    % )}{\lambda^2} - \exp\left(\frac{\ell_i(\w')}{\lambda'}\right)\frac{\ell_{i}(\w')}{\lambda'^2}\right\|^2
    \\
    &\quad\leq  \left\|\exp\left(\frac{\ell_i(\w)}{\lambda}\right) \frac{\nabla_\w \ell_{i}(\w)}{\lambda} - \exp\left(\frac{\ell_i(\w')}{\lambda}\right) \frac{\nabla_{\w} \ell_{i}(\w')}{\lambda} \right\|^2 \\\nonumber&\quad\quad+ \left\|\exp\left(\frac{\ell_i(\w')}{\lambda}\right) \frac{\nabla_{\w} \ell_{i}(\w')}{\lambda}  - \exp\left(\frac{\ell_i(\w')}{\lambda'}\right) \frac{\nabla_{\w} \ell_{i}(\w')}{\lambda'}\right\|^2 
    \\
    &\quad\quad+\left\|\exp\left(\frac{\ell_i(\w)}{\lambda}\right)\frac{\ell_{i}(\w)}{\lambda^2} - \exp\left(\frac{\ell_i(\w')}{\lambda}\right)\frac{\ell_{i}(\w')}{\lambda^2} \right\|^2\\\nonumber&\quad\quad+\left\| \exp\left(\frac{\ell_i(\w')}{\lambda}\right)\frac{\ell_{i}(\w')}{\lambda^2} - \exp\left(\frac{\ell_i(\w')}{\lambda'}\right)\frac{\ell_{i}(\w')}{\lambda'^2}\right\|^2.
 \end{align*}
To bound the first term, we first check the Lipschitz continuous of $\exp(\frac{\ell_i(\w)}{\lambda}) \frac{\nabla_\w \ell_{i}(\w)}{\lambda}$ with respect to $\w$,
 \begin{align*}
     &\left\|\frac{\nabla\left(\exp\left(\frac{\ell_i(\w)}{\lambda}\right) \frac{\nabla_\w \ell_{i}(\w)}{\lambda}\right)}{\nabla \w}\right\|  \\\nonumber&\leq \left\|\exp\left(\frac{\ell_i(\w)}{\lambda}\right)\left( \frac{\nabla_\w \ell_{i}(\w)}{\lambda}\right)\left( \frac{\nabla_\w \ell_{i}(\w)}{\lambda}\right)^{\top}\right\| +\left\|\exp\left(\frac{\ell_i(\w)}{\lambda}\right)\frac{\nabla^2_\w \ell_{i}(\w)}{\lambda}\right\|  \\&\overset{(a)}{=} \left\|\exp\left(\frac{\ell_i(\w)}{\lambda}\right)\left( \frac{\nabla_\w \ell_{i}(\w)}{\lambda}\right)^{\top}\left( \frac{\nabla_\w \ell_{i}(\w)}{\lambda}\right)\right\| +\left\|\exp\left(\frac{\ell_i(\w)}{\lambda}\right)\frac{\nabla^2_\w \ell_{i}(\w)}{\lambda}\right\| \\&\overset{(b)}{\leq} \exp\left(\frac{\ell_i(\w)}{\lambda}\right)\left\|\left( \frac{\nabla_\w \ell_{i}(\w)}{\lambda}\right)\right\|^2 +\left\|\exp\left(\frac{\ell_i(\w)}{\lambda}\right)\frac{\nabla^2_\w \ell_{i}(\w)}{\lambda}\right\|
    \\&\leq \exp\left(\frac{C}{\lambda_0}\right)\left(\frac{G^2}{\lambda_0^2} + \frac{L}{\lambda_0} \right):= L_A. 
 \end{align*}
 where equality (a) is due to the property of the norm of rank-one symmetric matrix and inequality (b) is due to Cauchy-Schwarz inequality.
 
 Therefore, we have
 $$\left\|\exp(\frac{\ell_i(\w)}{\lambda}) \frac{\nabla_\w \ell_{i}(\w)}{\lambda} - \exp(\frac{\ell_i(\w')}{\lambda}) \frac{\nabla_{\w} \ell_{i}(\w')}{\lambda} \right\|^2\leq L_A \left\|\w-\w'\right\|^2$$
 Furthermore, it holds that
  \begin{align*}
    \left\|\frac{\nabla\left(\exp\left(\frac{\ell_i(\w)}{\lambda}\right) \frac{\nabla_\w \ell_{i}(\w)}{\lambda}\right)}{\nabla \lambda}\right\|& = \left\| \exp\left(\frac{\ell_i(\w)}{\lambda}\right)\frac{\ell_i(\w)\nabla_\w \ell_i(\w) }{\lambda^3} +\exp\left(\frac{\ell_i(\w)}{\lambda}\right)\left(\frac{\nabla_{\w} \ell_i(\w)}{\lambda^2}\right) \right\| \\&\leq \exp\left(\frac{C}{\lambda_0}\right)\left(\frac{CG}{\lambda_0^3} + \frac{G}{\lambda^2_0} \right):= L_B \\
       \left\| \frac{\nabla\left(\exp\left(\frac{\ell_i(\w)}{\lambda}\right)\frac{\ell_i(\w) }{\lambda^2}\right)}{\nabla\w}\right\| &= \left\| \exp\left(\frac{\ell_i(\w)}{\lambda}\right)\frac{\ell_i(\w)\nabla_{\w} \ell_i(\w)}{\lambda^3} + \exp\left(\frac{\ell_i(\w)}{\lambda}\right)\frac{\nabla_{\w} \ell_i(\w)}{\lambda^2}\right\|\\&\leq \exp\left(\frac{C}{\lambda_0}\right)\left(\frac{CG+\lambda_0G}{\lambda_0^3}\right):= L_C\\
\left\| \frac{\nabla\left(\exp\left(\frac{\ell_i(\w)}{\lambda}\right)\frac{\ell_i(\w) }{\lambda^2}\right)}{\nabla\lambda}\right\| &= \left\| \exp\left(\frac{\ell_i(\w)}{\lambda}\right)\frac{ \ell^2_i(\w)}{\lambda^4} + \exp\left(\frac{\ell_i(\w)}{\lambda}\right)\frac{2 \ell_i(\w)}{\lambda^3}\right\|\\&\leq \exp\left(\frac{C}{\lambda_0}\right)\left(\frac{C^2+2\lambda_0C}{\lambda_0^4}\right) := L_D.
 \end{align*}
 As a result, we obtain
 \begin{align*}
  & \left\|\nabla_{\w,\lambda} g_i(\w,\lambda) - \nabla_{\w,\lambda} g_i(\w',\lambda')\right\|^2 \\&\leq L_A^2 \left\|\w-\w'\right\|^2 + L_B^2 \left\|\lambda - \lambda'\right\| ^2+ L_C^2\left\|\w-\w'\right\|^2 +L_D^2\left\|\lambda - \lambda'\right\|^2 \\
    &=   (L_A^2+L_C^2)\left\| \w-\w'\right\|^2 + (L_B^2+L_D^2)\left\|\lambda - \lambda'\right\|^2  \\
    &\leq (L_A^2+L_B^2 +L_C^2+L_D^2)\left\|(\w^\top,\lambda) - (\w'^\top,\lambda') \right\|^2.
 \end{align*}

 Thus $L_{\nabla_g} =\sqrt{L_A^2+L_B^2 +L_C^2+L_D^2}$.
 
\end{proof}

\begin{lemma}
\label{lem:F_pty}
% $\widetilde{\lambda}L_g^2+L_g^2+\widetilde{\lambda}L_{\nabla_g} + 1)$
 $F(\w,\lambda)$ is $L_F$-smooth, where $L_F=\tilde{\lambda}L_g^2+2L_g+\tilde{\lambda}L_{\nabla_g} + 1 +\tilde{\lambda}$.
\end{lemma}

\noindent
\textbf{Remark: }Lemma~\ref{lem:smooth_f_lambda}, \ref{lem:g_pty} and Lemma~\ref{lem:F_pty} imply that $L_{\nabla f_\lambda} =L_{ f_\lambda}\leq L_F, L_g\leq L_F $ and $L_F\geq 1$. 

\begin{proof}
For all $\x_1 = (\w_1^\top,\lambda_1)^\top, \x_2 = (\w_2^\top,\lambda_2)^\top \in \X$, and let
 $\d(\x) = (0,\cdots,0,\log(g(\x))+\rho)^\top \in\R^{d+1}$, by expansion we have
\begin{align*}
 &\|\nabla F(\x_1) -\nabla F(\x_2)\| \\&=  \|\nabla f_{\lambda_1}(g(\x_1))\nabla g(\x_1)+\d(\x_1)-\nabla f_{\lambda_2}(g(\x_2))\nabla g(\x_2) -\d(\x_2) \| \\
 &\leq \|\nabla f_{\lambda_1}(g(\x_1))\nabla g(\x_1)-\nabla f_{\lambda_2}(g(\x_2))\nabla g(\x_2) \| + |\log(g(\x_1)) - \log(g(\x_2))|\\
 &\leq  \|\nabla f_{\lambda_1}(g(\x_1))\nabla g(\x_1)-\nabla f_{\lambda_1}(g(\x_2))\nabla g(\x_1)\|+ \|\nabla f_{\lambda_1}(g(\x_2))\nabla g(\x_1)-\nabla f_{\lambda_2}(g(\x_2))\nabla g(\x_1)\|\\
 &\quad+\|\nabla f_{\lambda_2}(g(\x_2))\nabla g(\x_1)-\nabla f_{\lambda_2}(g(\x_2))\nabla g(\x_2) \| +|g(\x_1) - g(\x_2)|.
\end{align*}
Noting the Lipschtiz continuous of $g(x)$ and $\nabla g(x)$, we obtain
\begin{align*}
 &\|\nabla F(\x_1) -\nabla F(\x_2)\|\\&\leq (L_{\nabla_{f_{\lambda_1}}}L_g+1) | g(\x_1)-g(\x_2)| + \frac{\|\nabla g(\x_1)\|}{g(\x_2)}\| \lambda_1 - \lambda_2\| + L_{f_{\lambda_2}}\|\nabla g(\x_1) -\nabla g(\x_2)  \| \\
 &\overset{(a)}{\leq}  (L_{\nabla_{f_{\lambda_1}}}L^2_g+L_g)\|\x_1 - \x_2\| + \|\nabla g(\x_1)\|\|\lambda_1 - \lambda_2\| + L_{f_{\lambda_2}}L_{\nabla_g}\| \x_1 - \x_2\|  \\
 &\leq (L_{\nabla_{f_{\lambda_1}}}L^2_g + 2L_g+L_{f_{\lambda_2}}L_{\nabla_g})\|\x_1-\x_2\| \\
 &\overset{(b)}{\leq}(\tilde{\lambda}L_g^2+2L_g+\tilde{\lambda}L_{\nabla_g} + 1 +\tilde{\lambda})\| \x_1 - \x_2\|.
\end{align*}
where the inequality (a) is due to $g(\x_2)\geq 1$ and the inequality (b) is due to the upper bound of $\lambda$.
Thus, $L_F=\tilde{\lambda}L_g^2+2L_g+\tilde{\lambda}L_{\nabla_g} + 1 +\tilde{\lambda}$.
\end{proof}

\subsection{Proof of Lemma~\ref{lem:lambda_upper}}
\begin{proof}
Recall the primal problem:
\begin{align*}
p^*= \max_{\{\p\in\Delta_n, D(\p, 1/n)\leq \rho\}}\sum_{i=1}^n  p_i \ell_i(\w) +  \lambda_0 D(\p, 1/n).
 \end{align*}
 Invoking dual variable $\bar\lambda$, we obtain the dual problem:
 \begin{align}\label{le3:eq0}
q^*= \min_{\bar\lambda\geq 0}\max_{\p\in\Delta_n}\sum_{i=1}^n  p_i \ell_i(\w) - \bar\lambda (D(\p, 1/n) -\rho) - \lambda_0 D(\p, 1/n).
 \end{align}
 Set $\bar\p=(1/n,\dots,1/n)$, which is a Slater vector satisfying $D(\bar \p, 1/n) -\rho<0$.
Applying Lemma 3 in \citep{nedic2009subgradient}, we have 
  \begin{align*}
|\bar\lambda^*|\leq\frac{1}{\rho}\left(q^*-\sum_{i=1}^n  \bar p_i \ell_i(\w) -  \lambda_0 D(\bar\p, 1/n)\right).
 \end{align*}
 Since the primal problem is concave in term of $\p$ given $\w$, we have $p^*=q^*$. Therefore,
   \begin{align}\nonumber
|\bar\lambda^*|&\leq\frac{1}{\rho}\left(p^*-\sum_{i=1}^n  \bar p_i \ell_i(\w)\right)\\\nonumber&=\frac{1}{\rho}\left(\sum_{i=1}^n  \bar p^*_i \ell_i(\w) -  \lambda_0 D(\p^*, 1/n)-\sum_{i=1}^n \bar p_i \ell_i(\w)\right)\\ \label{le3:eq1}&\leq \frac{C}{\rho},
 \end{align}
where the last inequality is because $|\ell_i(\w)|\leq C$ for $\w \in \W$. 
Let $\lambda=\bar\lambda+\lambda_0$, we have $$q^*= \min_{\lambda\geq \lambda_0}\max_{\p\in\Delta_n}\sum_{i=1}^n p_i \ell_i(\w) - \lambda (D(\p, 1/n) -\rho) - \lambda_0 \rho.$$
Section~\ref{sec:derivation} will also show $$q^*= \min_{\lambda\geq \lambda_0}\lambda \log\left(\frac{1}{n}\sum_{i=1}^n  \exp\left(\frac{\ell_i(\w)}{\lambda}\right)\right) +\lambda (\rho-\rho_0).$$
By Eq.~(\ref{le3:eq1}), we have the optimal solution of above optimization problem 
$|\lambda^*|\leq|\bar\lambda^*|+\lambda_0\leq\lambda_0+\frac{C}{\rho}$, which complete the proof

 \end{proof}
 \section{Proofs in Section~\ref{sec:basic}}
 \subsection{Technical Lemmas}
 \begin{lemma}
\label{lem:recur-g}
Suppose Assumption~\ref{ass:2} holds and $i\sim \D$ and $s$ are initialized with $s_1 = \exp(\frac{\ell_i(\w_{1})}{\lambda_{1}})$. Then for every $ t\in\{1,\cdots T\}$ we have
\begin{align*}
    \E[\|g(\x_{t+1})-s_{t+1} \|^2]\leq \E\left[ (1-\beta)\|g(\x_t) -s_t \|^2 + \frac{2L_g^2\|\x_{t+1} - \x_t \|^2}{\beta} + \beta^2\sigma^2\right].
\end{align*}
Taking summation of $\E[\|g(\x_{t+1})-s_{t+1} \|^2]$ from $1$ to $T$, we have
\begin{equation}
\label{eqn:lemma9-2}
    \begin{aligned}
         \sum\limits_{t=1}^{T}\E[\|g(\x_{t})-s_{t} \|^2]\leq \E\left[ \frac{\|g(\x_1) -s_1 \|^2}{\beta} + \frac{2L_g^2}{\beta^2}\sum\limits_{t=1}^{T}\|\x_{t+1} - \x_t \|^2 + \beta T\sigma^2\right]. 
    \end{aligned}
\end{equation}
\end{lemma}
\begin{proof}
Note that $s_{t+1}=(1-\beta)s_t + \beta g_i(\x_{t+1})$ and $\E[g(\x_{t+1}) - g_i(\x_{t+1})]$=0, then by simple expansion we have
\begin{align}\nonumber
  & \E[\|g(\x_{t+1})-s_{t+1} \|^2] \\\nonumber
   &\quad= \E[\|\beta(g(\x_{t+1})-g_i(\x_{t+1}))+(1-\beta)(g(\x_{t+1})-s_t)\|^2] \\\nonumber
   &\quad= \E[\beta^2\|g(\x_{t+1}) - g_i{(\x_{t+1})}\|^2 + (1-\beta)^2\| g(\x_{t+1})-s_t\|^2 ] \\\nonumber
   &\quad\quad+ 2\underbrace{\E[\langle g(\x_{t+1}) - g_i(\x_{t+1}), g(\x_{t+1}) - s_t\rangle]}_{0} \\ \label{eq:le:gst} 
   &\quad=  \E[\beta^2\|g(\x_{t+1}) - g_i{(\x_{t+1})}\|^2 + (1-\beta)^2\| g(\x_{t+1})-g(\x_{t}) + g(\x_{t}) - s_t\|^2 ]. 
   \end{align}
 
 Invkoing Lemma~\ref{lem:g_pty} to Eq.~(\ref{eq:le:gst}) and recalling Assumption~\ref{ass:2} , we obtain 
\begin{align*}
  & \E[\|g(\x_{t+1})-s_{t+1} \|^2] \\
   &\quad\overset{(a)}{\leq} \E[\beta^2\|g(\x_{t+1}) - g_i{(\x_{t+1})}\|^2 + (1-\beta)^2(1+\beta)\|g(\x_t) -s_t \|^2 \\
   &\quad\quad+ (1+\frac{1}{\beta})(1-\beta)^2\|g(\x_{t+1}) - g(\x_t) \|^2\\
   &\quad\overset{(b)}{\leq}  \E\left[\beta^2\|g(\x_{t+1}) - g_i{(\x_{t+1})}\|^2 + (1-\beta)\|g(\x_t) -s_t \|^2 + \frac{2L_g^2\|\x_{t+1} - \x_t \|^2}{\beta}\right] \\
   &\quad\overset{(c)}{\leq} \E\left[ (1-\beta)\|g(\x_t) -s_t \|^2 + \frac{2L_g^2\|\x_{t+1} - \x_t \|^2}{\beta} + \beta^2\sigma^2\right]. 
   \end{align*}

where the inequality $(a)$ is due to $(a+b)^2\leq (1+\beta)a^2 + (1+\frac{1}{\beta})b^2$, the inequality $(b)$ is because of $(1-\beta)^2\leq 1$, $(1+\frac{1}{\beta})\leq \frac{2}{\beta}$ and the Lemma~\ref{lem:g_pty} and the inequality $(c)$ is from Assumption~\ref{ass:2}.
\end{proof}

% \begin{lemma}
% \label{lem:dist_diff_x}
% Suppose Assumption~\ref{ass:1} holds,  we have 
% \begin{equation}
% \label{eqn:Yi-1}
% \begin{aligned}
%     \sum\limits_{t=1}^{T}\frac{1-2\eta  L_F}{4\eta }\| \x_{t+1} -\x_t \|^2 \leq \Delta + \sum\limits_{t=1}^{T}\eta \|\z_t -\nabla F(\x_t)\|^2
% \end{aligned}
% \end{equation}
% \end{lemma}
\begin{lemma}
\label{lem:dist_lemma_1}
\noindent
Under Assumption~\ref{ass:1}, run Algorithm~\ref{alg:SCCMA} with $\eta L_F\leq 1/4$, and then
the output $\x_R$ of Algorithm~\ref{alg:SCCMA} satisfies
\begin{equation}
    \begin{aligned}
        \E_R[\dist (0, \hat{\partial}\bar F(\x_R))^2]
       &\leq \frac{2+40L_F\eta}{T} \sum\limits_{t=1}^T\|\z_t -\nabla F(\x_t)\|^2 + \frac{2\Delta}{\eta T} + \frac{40L_F\Delta}{T}.
    \end{aligned}
\end{equation}
\end{lemma}
\begin{proof}The proof of this lemma follow the proof of Theorem 2 in \citep{xu2019non}.

Recall the update of $\x_{t+1}$ is 
\begin{align*}
    \x_{t+1} &= \Pi_\X(\x_t - \eta\z_t)\\
    &=\argmin_{\x\in\R^{d+1}}\{\delta_\X(\x)+\Braket{\z_t,\x-\x_t}+\frac{1}{2\eta}\|\x-\x_t\|^2\}.
\end{align*}
then by Exercise $8.8$ and Theorem $10.1$ of \citep{rockafellar1998variational} we know
\begin{equation*}
    -\mathbf{z}_{t}-\frac{1}{\eta}\left(\mathbf{x}_{t+1}-\mathbf{x}_{t}\right) \in \hat{\partial} \delta_\X\left(\mathbf{x}_{t+1}\right),
\end{equation*}
which implies that
\begin{equation}\label{eq:partialPhi}
\nabla F\left(\mathbf{x}_{t+1}\right)-\mathbf{z}_{t}-\frac{1}{\eta}\left(\mathbf{x}_{t+1}-\mathbf{x}_{t}\right) \in \nabla F\left(\mathbf{x}_{t+1}\right)+\hat{\partial} \delta_\X\left(\mathbf{x}_{t+1}\right)=\hat{\partial} \bar F\left(\mathbf{x}_{t+1}\right).
\end{equation}
By the update of $\mathbf{x}_{t+1}$, we also have,
\begin{align*}
    \delta_\X(\x_{t+1}) +\langle\z_t,\x_{t+1} - \x_t \rangle + \frac{1}{2\eta }\| \x_{t+1} - \x_t\|^2\leq \delta_\X({\x_t}).
\end{align*}
Since $F(\x)$ is smooth with parameter $L_F$, then
\begin{align*}
    F(\x_{t+1}) \leq F(\x_t) + \langle \nabla F(\x_t), \x_{t+1} - \x_t\rangle + \frac{L_F}{2}\|\x_{t+1} - \x_t\|^2.
\end{align*}
Combing the above two inequalities, we get
\begin{align*}
    \langle \z_t - \nabla F(\x_t),\x_{t+1} - \x_t \rangle + \frac{1}{2}(1/\eta  - L)\|\x_{t+1} - \x_t \|^2 \leq \bar F(\x_t) - \bar F(\x_{t+1}).
\end{align*}
That is
\begin{align*}
    \frac{1}{2}(1/\eta  - L_F) \|\x_{t+1} - \x_t \|^2 & \leq \bar F(\x_t) - \bar F(\x_{t+1}) - \langle \z_t -\nabla F(\x_t), \x_{t+1} - \x_t\rangle \\
    &\leq \bar F(\x_t) - \bar F(\x_{t+1}) + \eta \|\z_t-\nabla F(\x_t) \|^2 + \frac{1}{4\eta } \|\x_{t} - \x_{t+1} \|^2,
\end{align*}
where the last inequality uses Young's inequality $\langle \a, 
\b\rangle \leq \| \a\|^2 + \frac{\|\b \|^2}{4}$. Then by rearranging the above inequality and summing it across $t=1,\cdots, T$, we have
\begin{align}\nonumber
 \sum\limits_{t=1}^{T}\frac{1-2\eta L_F}{4\eta }\|\x_{t+1} - \x_t \|^2 &\leq \bar F(\x_1) -\bar F(\x_{T+1}) + \sum\limits_{t=1}^{T}\eta \|\z_t - \nabla F(\x_t) \|^2 \\\nonumber
 &\leq \bar F(\x_1) - \inf_{\x\in\X}\bar F(\x)+ \sum\limits_{t=1}^{T}\eta \|\z_t - \nabla F(\x_t) \|^2 \\\label{eqn:Yi-1}
 &\leq \Delta + \sum\limits_{t=1}^{T}\eta \|\z_t - \nabla F(\x_t) \|^2. 
\end{align}
By the same method used in the proof of Theorem 2 in~\citet{xu2019non}, we have the following inequality,
\begin{equation}
\label{eqn:Yi-2}
    \begin{aligned}
       \sum\limits_{t=1}^{T}\|\z_t -\nabla F(\x_{t+1}) + \frac{1}{\eta}(\x_{t+1} - \x_t)\|^2 
      &\leq 2\sum\limits_{t=1}^{T}\| \z_t - \nabla F(\x_t)\|^2 + \frac{2\Delta}{\eta} \\&\quad+ (2L_F^2 + \frac{3L_F}{\eta})\sum\limits_{t=1}^T\|\x_{t+1} -\x_t\|^2.
    \end{aligned}
\end{equation}
Recalling $\eta L_F\le\frac{1}{4}$ and combining Eq.~(\ref{eqn:Yi-1}) and Eq.~(\ref{eqn:Yi-2}), we obtain
    \begin{align}\nonumber
      & \sum\limits_{t=1}^{T}\|\z_t -\nabla F(\x_{t+1}) + \frac{1}{\eta}(\x_{t+1} - \x_t)\|^2 \\\nonumber
      &{\overset{(a)}{\leq}}2\sum\limits_{t=1}^{T}\| \z_t - \nabla F(\x_t)\|^2 + \frac{2\Delta}{\eta} + \frac{5L_F}{\eta}\left(\frac{1}{1/4-\eta_1 L_F/2}\left(\eta_1\Delta + \eta_1\sum\limits_{t=1}^{T}\eta_t\|\z_t -\nabla F(\x_t)\|^2\right)\right) \\\label{eqn:Yi-3}
      &\overset{(b)}{\leq} 2\sum\limits_{t=1}^{T}\| \z_t - \nabla F(\x_t)\|^2 + \frac{2\Delta}{\eta} + 40L_F\Delta + 40\eta L_F \sum\limits_{t=1}^{T}\|\z_t -\nabla F(\x_t)\|^2.
    \end{align}
where inequality (a) is due to $(2L_F^2 + \frac{3L_F}{\eta})\leq\frac{5L_F}{\eta}$ and inequality (b) is due to $\frac{1}{1/4-\eta L_F/2}\leq 8$.

Recalling Eq.~(\ref{eq:partialPhi}) and the output rule of Algorithm~\ref{alg:SCCMA}, we have
\begin{equation}\label{eq:outputxR}
     \E_R[\dist (0, \hat{\partial}\bar F(\x_R))^2]\leq\frac{1}{T} \sum\limits_{t=1}^{T}\|\z_t -\nabla F(\x_{t+1}) + \frac{1}{\eta}(\x_{t+1} - \x_t)\|^2.
\end{equation}
Then by combining Eqs.~(\ref{eqn:Yi-3},\ref{eq:outputxR}) together we have the Lemma.
\end{proof}

\begin{lemma}
\label{lem:dist_cum_var_lemma_2}
Under Assumption~\ref{ass:1},~\ref{ass:2}, run Algorithm~\ref{alg:SCCMA} with $\eta\leq \frac{\beta}{4L_F\sqrt{4+20L_g^2} }\leq \frac{1}{4L_F}$, and then we have
\begin{align*}
 \frac{1}{T}\sum\limits_{t=1}^{T}  \E[\|\z_t-\nabla F(\x_t)\|^2]&\leq \frac{2\E[\|\z_1-\nabla F(\x_1) \|^2]}{\beta T} + \frac{\Delta}{\eta T } + \frac{20L_F\E[\|g(\x_{1})-s_{1} \|^2]}{\beta T}  + 24\beta L_{F}^2\sigma^2.
\end{align*}
\end{lemma}
\begin{proof}
To facilitate our proof statement, we define the following notations:
\begin{equation*}
\label{eqn:def_grad}
 \begin{aligned}
&\nabla F(\x_t)^\top = (\nabla_{\w} F(\x_t)^{\top}, \nabla_{\lambda}F(\x_t))= (\nabla f_{\lambda_t} (g(\x_t))\nabla_{\w} g(\x_t)^\top,   \nabla f_{\lambda_t}  (g(\x_t))\nabla_{\lambda} g(\x_t) + \log(g(\x_t)) + \rho)\\
&\widetilde{\nabla} F(\x_t)^\top = (\nabla f_{\lambda_t} (g(\x_t))\nabla_{\w} g_i(\x_t)^\top,   \nabla f_{\lambda_t}  (g(\x_t))\nabla_{\lambda} g_i(\x_t) + \log(g(\x_t)) + \rho)\\
&G(\x_t)^\top = (G_{\w_t}(\x_t)^\top, G_{\lambda_t}(\x_t))=(\nabla f_{\lambda_{t}}(s_{t})\nabla_{\w} g_i(\x_t)^\top, \nabla f_{\lambda_{t}}(s_{t})\nabla_{\lambda} g_i(\x_t)  +\log(s_{t}) + \rho).
\end{aligned}   
\end{equation*}
It is worth to notice that $\E[\widetilde{\nabla}F(\x_t)] = \nabla F(\x_t)$.

For every iteration $t$, by simple expansion we have
\begin{equation*}
    \begin{aligned}
    I_t & = \E[\|\nabla F(\x_t)-\z_t \|^2] \\
    & =  \E[\|\nabla F(\x_t)-(1- \beta)\z_{t-1} - \beta G(\x_t) \|^2]\\
    &=  \E[\|(1- \beta)(\nabla F(\x_t) - \nabla F(\x_{t-1}))+(1- \beta) \nabla F(\x_{t-1})- (1- \beta)\z_{t-1} +  \beta \nabla F(\x_t)- \beta G(\x_t) \|^2]\\
    & =  \E[\|(1- \beta)\underbrace{(\nabla F(\x_t)  - \nabla F(\x_{t-1})}_{A})+ (1- \beta)\underbrace{(\nabla F(\x_{t-1})-\z_{t-1})}_{B}\|^2] \\&\quad+ \E[\|\beta(\underbrace{ \widetilde{\nabla} F(\x_t)- G(\x_t)}_{C}) + \beta\underbrace{(\nabla F(\x_{t}) -\widetilde{\nabla}F(\x_t))}_{D}\|^2]\\
& = \E[ (1-\beta)^2\|A\|^2+ (1-\beta)^2\|B\|^2 + \beta^2\|C\|^2+ \beta^2\|D\|^2+2(1-\beta)(1-\beta)\langle A, B\rangle\\&\quad+ 2\beta(1-\beta)\langle A, C\rangle +2\beta(1-\beta)\langle A, D\rangle + 2(1-\beta)\beta\langle B, C\rangle + 2(1-\beta)\beta\langle B, D\rangle + 2\beta^2\langle C, D\rangle]\\
&\overset{(a)}{=} \E[ (1-\beta)^2\|A\|^2+ (1-\beta)^2\|B\|^2 + \beta^2\|C\|^2+ \beta^2\|D\|^2\\
&\quad+2(1-\beta)^2\langle A,B \rangle + 2(1-\beta)\beta\langle C, B\rangle+ 2\beta(1-\beta)\langle A, C\rangle  +  2\beta^2\langle C, D\rangle],
\end{aligned}
\end{equation*}
where the equality $(a)$ is due to $\E\langle \nabla F(\x_t) - \nabla F(\x_{t-1}), \nabla F(\x_t) - \widetilde{\nabla}F(\x_t)\rangle = 0$ and $\E\langle  \z_{t-1}-\nabla F(\x_{t-1}) , \nabla F(\x_t) - \widetilde{\nabla}F(\x_t) \rangle = 0 $. 

By Young's inequality, we have $(1-\beta)^2\langle A,B\rangle \leq (1-\beta)\langle A,B\rangle \leq \frac{2}{\beta}\|A\|^2 + \frac{(1-\beta)^2\beta}{8}\|B\|^2$, $2\beta(1-\beta)\langle C, B\rangle \leq \frac{(1-\beta)^2\beta}{2}\|B \|^2 + 2\beta\| C\|^2$,
$2\beta(1-\beta)\langle A, C\rangle\leq(1-\beta)^2\|A \|^2  +\beta^2\| C\|^2 $ and $ 2\beta^2\langle C, D\rangle\leq \beta^2\| C\|^2 +\beta^2\| D\|^2$. Therefore, noting $(1-\beta)< 1$ and $1/\beta> 1$, we can obtain
    \begin{align}\nonumber
    I_t&\leq \E[ (1-\beta)^2\|A\|^2+ (1-\beta)^2\|B\|^2 + \beta^2\|C\|^2+ \beta^2\|D\|^2\\\nonumber
& \quad+\frac{2}{\beta}\| A\|^2+ \frac{(1-\beta)^2\beta}{2}\| B\|^2+2\beta\| C\|^2+
\frac{(1-\beta)^2\beta}{2}\| B\|^2\\\nonumber&\quad+
(1-\beta)^2\|A \|^2  +\beta^2\| C\|^2  +  \beta^2\| C\|^2 +\beta^2\| D\|^2]\\\label{eq:I_t}
&\leq\E[(1-\beta)\|B\|^2+\frac{4}{\beta}\|A\|^2 +5\beta \| C\|^2 +2\beta^2\|D\|^2].
\end{align}
Thus recalling the defintion of $G(\x_t), \widetilde{\nabla}F(\x_t), \nabla F(\x_t)$ and applying the smoothness and Lipschitz continuity of $f_\lambda$ and $g$, we have
  \begin{align}\nonumber
    C &= \|\widetilde{\nabla} F(\x_t)- G(\x_t)\|^2 \\ \nonumber
    &=\|\nabla f_{\lambda_t} (g(\x_t))\nabla_{\w} g_i(\x_t) - \nabla f_{\lambda_{t}}(s_{t})\nabla_{\w_{t}} g_i(\x_t)\|^2 \\\nonumber
    &\quad+\|  \nabla f_{\lambda_t}  (g(\x_t))\nabla_{\lambda} g_i(\x_t) + \log(g(\x_t))
- \nabla f_{\lambda_{t}}(s_{t})\nabla_{\lambda} g_i(\x_t) - \log(s_{t})\| ^2\\\nonumber
&\leq \|\nabla f_{\lambda_t} (g(\x_t))\nabla_{\w} g_i(\x_t) - \nabla f_{\lambda_{t}}(s_{t})\nabla_{\w_{t}} g_i(\x_t)\|^2 +2\|  \nabla f_{\lambda_t}  (g(\x_t))\nabla_{\lambda} g_i(\x_t)  
- \nabla f_{\lambda_{t}}(s_{t})\nabla_{\lambda} g_i(\x_t) \| ^2 \\\nonumber
    &\quad+ 2\|\log(g(\x_t)) - \log(s_{t}) \|^2\\\nonumber
&\overset{(a)}{\leq} 2L^2_{g} L^2_{\nabla f_{\lambda_t}}\|s_t -g(\x_t)\|^2 + 2\|s_t - g(\x_t)\|^2\\\label{eqn:C}
&\overset{(b)}{\leq} 2L_F^2\|s_t-g(\x_t) \|^2,
\end{align}  
where the inequality $(a)$ is due to $|\log(g(\x_t)) - \log(s_t)| \leq |s_t -g(\x_t) |$ since $g(\x_t)\geq 1 ,s_t\geq 1$ for all $t = \{1,\cdots, T\}$ by the definition and initialzation of $g_i(\x_t), s_t$, and the inequality $(b)$ is due to $L^2_{g} L^2_{\nabla f_{\lambda_t}}+1\leq L_F^2$.

And by the similar method, we also have
  \begin{align}\nonumber
    D & = \|\nabla F(\x_{t}) -\widetilde{\nabla}F(\x_t)\|^2 \\ \nonumber
    &= \|\nabla f_{\lambda_t} (g(\x_t))\nabla_{\w} g(\x_t) - \nabla f_{\lambda_t} (g(\x_t))\nabla_{\w} g_i(\x_t)\|^2\\ \nonumber
  & \quad+ \|   \nabla f_{\lambda_t}  (g(\x_t))\nabla_{\lambda} g(\x_t) + \log(g(\x_t)) + \rho -   \nabla f_{\lambda_t}  (g(\x_t))\nabla_{\lambda} g_i(\x_t) - \log(g(\x_t)) - \rho\|^2\\\nonumber
  & =  \|\nabla f_{\lambda_t} (g(\x_t))\nabla_{\w} g(\x_t) - \nabla f_{\lambda_t} (g(\x_t))\nabla_{\w} g_i(\x_t)\|^2\\\nonumber
  & \quad+ \|   \nabla f_{\lambda_t}  (g(\x_t))\nabla_{\lambda} g(\x_t)  -   \nabla f_{\lambda_t}  (g(\x_t))\nabla_{\lambda} g_i(\x_t) \|^2\\\label{eqn:D}
  &\leq L_{f_{\lambda_t}}^2 \|\nabla g(\x_t) - \nabla g_i(\x_t) \|^2\leq L_F^2  \|\nabla g(\x_t) - \nabla g_i(\x_t) \|^2. 
\end{align}
Thus combining the Eqs.~(\ref{eq:I_t}, \ref{eqn:C}, \ref{eqn:D}) and applying Assumption~\ref{ass:2}, we can obtain
\begin{align*}
  &\E[\| \z_{t}-\nabla F(\x_{t}) \|^2]\\ &  =    \E[(1-\beta)\| \z_{t-1}-\nabla F(\x_{t-1}) \|^2 + \frac{4}{\beta}\|\nabla F(\x_t) - \nabla F(\x_{t-1}) \|^2\\ 
  &\quad+ 5\beta \|\widetilde{\nabla} F(\x_t)- G(\x_t) \|^2  +2\beta^2
  \|\nabla F(\x_{t}) -\widetilde{\nabla}F(\x_t)\|^2]\\&\leq\E[(1-\beta)\| \z_{t-1}-\nabla F(\x_{t-1}) \|^2 + \frac{4}{\beta}L_F^2\|\x_{t}-\x_{t-1}\|^2 +10L_F^2\beta\|g(\x_{t}) - s_t\|^2] + 2\beta^2 L_F^2 \sigma^2.
\end{align*}
Taking summation of $\E[\|\z_{t+1}-\nabla F(\x_{t+1})\|^2]$ from $1$ to $T$ and invoking Lemma~\ref{lem:recur-g}, we have
\begin{align*}
  & \sum\limits_{t=1}^{T}  \E[\|\z_t-\nabla F(\x_t)\|^2]\\
    &\leq \frac{\E[\|\nabla F(\x_{1}) - \z_{1}\|^2]}{\beta} + \frac{4L_F^2}{\beta^2}\sum\limits_{t=1}^T\E[\|\x_{t+1}-\x_{t}\|^2]+10L_F^2\beta\sum\limits_{t=1}^T\E[\|g(\x_{t}) - s_t\|^2] + 2\beta^2 L_F\sigma^2\\
      &\leq
     \frac{\E[\|\nabla F(\x_{1}) - \z_{1}\|^2]}{\beta} + \frac{4L_F^2}{\beta^2}\sum\limits_{t=1}^T\E[\|\x_{t+1}-\x_{t}\|^2]\\
     &\quad+10L_F^2\left (\E\left[ \frac{\|g(\x_1) -s_1 \|^2}{\beta} + \frac{2L_g^2}{\beta^2}\sum\limits_{t=1}^{T}\|\x_{t+1} - \x_t \|^2\right] + \beta T\sigma^2\right)+ 2\beta L^2_FT\sigma^2.
\end{align*}
Taking Eq.~(\ref{eqn:Yi-1}) into the above inequality, we have
\begin{align}\nonumber
  & \sum\limits_{t=1}^{T}  \E[\|\z_t-\nabla F(\x_t)\|^2]
     \\\nonumber&\leq\frac{\E[\|\nabla F(\x_{1}) - \z_{1}\|^2]}{\beta} + (\frac{4L_F^2}{\beta^2}+\frac{20L_F^2L_g^2}{\beta^2})\left ( \frac{\eta}{1/4- \eta L_F/2} \left(\Delta  + \eta\sum\limits_{t=1}^{T}\E[\|\z_t -\nabla F(\x_t)\|^2]\right)\right )\\\nonumber
     &\quad+10L_F^2\left(\frac{\E[\|g(\x_{1}) -s_{1} \|^2]}{\beta} + \beta T\sigma^2 \right)+ 2\beta L^2_FT\sigma^2\\\nonumber
     &\overset{(a)}{\leq}\frac{\E[\|\nabla F(\x_{1}) - \z_{1}\|^2]}{\beta} + (\frac{4L_F^2}{\beta^2}+\frac{20L_F^2L_g^2}{\beta^2})\left( 8\eta \left(\Delta  + \eta\sum\limits_{t=1}^{T}\E[\|\z_t -\nabla F(\x_t)\|^2]\right)\right)\\\nonumber
     &\quad+10L_F^2\left(\frac{\E[\|g(\x_{1}) -s_{1} \|^2]}{\beta} + \beta T\sigma^2 \right)+ 2\beta L^2_FT\sigma^2\\\nonumber
    &\overset{(b)}{\leq} \frac{\E[\|\z_1-\nabla F(\x_1) \|^2]}{\beta} + \frac{\Delta}{2\eta} +\frac{1}{2}\sum\limits_{t=1}^{T}\E[\|\z_t -\nabla F(\x_t)\|^2]\\\label{eq:arrage}&\quad+10L_F^2\left (\frac{\E[\|g(\x_{1}) -s_{1} \|^2]}{\beta}+ \beta T\sigma^2 \right)+ 2\beta L_{F}^2T\sigma^2,
\end{align}
where the inequality (a) is due to $\eta L_F\leq1/4$ and the inequality (b) is due to $8(4L_F^2 + 20L_F^2L_g^2)\eta^2\leq \frac{\beta^2}{2}$.

Rearranging terms and dividing $T$ on both sides of Eq.~(\ref{eq:arrage}), we compelte the proof.
\end{proof}
 \subsection{Proof of Theorem~\ref{thm:dist_main}}
 \begin{proof}
Since $\eta=\frac{\beta}{20L_F^2}$, $L_F\geq1$ and $L_F\leq L_g$, it holds that $\eta \leq \frac{\beta}{4L_F\sqrt{4+20L_g^2} }\leq \frac{1}{4L_F}$ which satisfy the assumptions of $\eta$ in Lemma~\ref{lem:dist_lemma_1} and Lemma~\ref{lem:dist_cum_var_lemma_2}. Therefore, combining Lemma~\ref{lem:dist_lemma_1} and Lemma~\ref{lem:dist_cum_var_lemma_2}, we have
\begin{align}\nonumber
    & \E[\dist (0, \hat{\partial}\bar F(\x_R))^2]\\\nonumber
       &\leq \frac{2+40L_F\eta}{T} \sum\limits_{t=1}^T\E[\|\z_t -\nabla F(\x_t)\|^2] + \frac{2\Delta}{\eta T} + \frac{40L_F\Delta}{T}\\\nonumber
     &\leq\frac{12}{T} \sum\limits_{t=1}^T\E[\|\z_t -\nabla F(\x_t)\|^2] + \frac{2\Delta}{\eta T} + \frac{20L_F\Delta}{T}\\\label{eqn:thm1eq1}
     &\leq \frac{24\E[\|\z_1-\nabla F(\x_1) \|^2]}{\beta T} + \frac{12\Delta}{\eta T } + \frac{240L^2_F\E[\|g(\x_{1})-s_{1} \|^2]}{\beta T}  +288L^2_F\beta\sigma^2 + \frac{2\Delta}{\eta T} + \frac{20L^2_F\Delta}{T}.
\end{align}
By the definition of $s_1$ and Assumption~\ref{ass:2}, it holds that
\begin{equation}\label{eqn:thm1eq2}
    \E[\|s_{1} - g(\x_{1})\|^2] \leq  \E[\|g_i(\x_1) - g(\x_{1}) \|^2] \leq \sigma^2.
\end{equation}
Since $L_g^2L^2_{\nabla f_{\lambda_1}}\leq L_F^2$ and $2L^2_{ f_{\lambda_1}}\leq L_F^2$, we have
\begin{align}\nonumber
&\E[\|\z_1-\nabla F(\x_1) \|^2] \\\nonumber&= \|\nabla f_{\lambda_1}(g_i(\x_1))\nabla g_i(\x_1) - \nabla f_{\lambda_1}(g(\x_{1}))\nabla g(\x_{1})  \| ^2 \\\nonumber
&= \|\nabla f_{\lambda_1}(g_i(\x_1)\nabla g_i(\x_1))-\nabla f_{\lambda_1}(g(\x_{1}))\nabla g_i(\x_1)+ \nabla f_{\lambda_1}(g(\x_{1}))\nabla g_i(\x_1)- \nabla f_{\lambda_1}(g(\x_{1})\nabla g(\x_{1})) \| ^2 \\\nonumber
&\overset{(a)}{\leq}2 \|\nabla f_{\lambda_1}(g_i(\x_1))-\nabla f_{\lambda_1}(g(\x_{1}))\|^2\|\nabla g_i(\x_1)\|^2+2 \|\nabla f_{\lambda_1}(g_i(\x_1))\|^2 \|\nabla g_i(\x_1)- \nabla g(\x_{1})\|^2 \\\label{eqn:thm1eq3}
&\leq (2L_g^2L^2_{\nabla f_{\lambda_1}} + 2L^2_{ f_{\lambda_1}})\sigma^2\leq 4L_F^2\sigma^2,
\end{align}
where the inequality $(a)$ is due to $\|\a+\b\|^2\leq 2\|\a \|^2 + 2\|\b \|^2$.

Combining Eqs.~(\ref{eqn:thm1eq1},\ref{eqn:thm1eq2},\ref{eqn:thm1eq3}), we obtain
\begin{align*}
    & \E[\dist (0, \hat{\partial}\bar F(\x_R))^2]\\
     &\leq \frac{24\E[\|\z_1-\nabla F(\x_1) \|^2]}{\beta T} + \frac{12\Delta}{\eta T } + \frac{240L^2_F\E[\|g(\x_{1})-s_{1} \|^2]}{\beta T}  +288L^2_F\beta\sigma^2 + \frac{2\Delta}{\eta T} + \frac{20L^2_F\Delta}{T} \\
     &\leq \frac{96L_F^2\sigma^2}{\beta T}+ \frac{12\Delta}{\eta T } + \frac{240L^2_F\sigma^2}{\beta T}  +288L^2_F\beta\sigma^2   + \frac{2\Delta}{\eta T} + \frac{20L^2_F\Delta}{T}\\
     &\leq\frac{96L_F^2\sigma^2}{\sqrt{T}}+ \frac{240\Delta L_F^2}{\sqrt{T} }  +\frac{528L^2_F\sigma^2}{\sqrt{T}}+  \frac{40\Delta L_F^2}{\sqrt{T}} + \frac{20L^2_F\Delta}{T} \\
     &\leq (624\sigma^2 +280\Delta)\frac{L_F^2}{\sqrt{T}} + \frac{20L_F^2\Delta}{T}.
\end{align*}
This complete the proof.
\end{proof}
 
 \section{Proofs in Section~\ref{sec:acceleration}}
 \subsection{Technical Lemmas}
 \begin{lemma}Let $\z_t = \nabla f_{\lambda_t}(s_t) \q_t + \q_{\lambda_t}$, where $\q_t =  (\v_t^\top, u_t)^\top$, $\q_{\lambda_t} = (\textbf{0}^\top,\log(s_{t})+\rho)^\top$ and $\textbf{0}\in\R^{d}$. Let $\|\varkappa_t\|^2 = \|s_t - g(\x_t) \|^2 + \|\v_t - \nabla_\w g(\x_t) \|^2 + |u_t - \nabla_\lambda g(\x_t)|^2$. Under Assumption~\ref{ass:1}, run Algorithm~\ref{alg:STORM-CCMA}, and then for every $ t\in\{1,\cdots T\}$  we have
\label{lem:storm_1}
\begin{align*}
  \| \z_t-\nabla F(\x_t) \|^2\leq 4L_F^2 \|\varkappa_t\|^2.
\end{align*}
\end{lemma}
\begin{proof}
By simple expansion, it holds that
\begin{align}\nonumber
  & \| \z_t- \nabla F(\x_t) \|^2 \\\nonumber
   &= \|\nabla f_{\lambda_t}(g(\x_t)) \nabla_\w g(\x_t) - \nabla f_{\lambda_t}(s_t) \v_t\|^2 \\\nonumber&\quad+  \|\nabla f_{\lambda_t}(g(\x_t)) \nabla_\lambda g(\x_t) - \nabla f_{\lambda_t}(s_t) \v_t + \log(g(\x_t)) - \log(s_t)\|^2 \\\nonumber
   &\overset{(a)}{\leq} 2\|\nabla f_{\lambda_t}(g(\x_t)) \nabla_\w g(\x_t) - \nabla f_{\lambda_t}(s_t) \v_t\|^2 + 2 \|\nabla f_{\lambda_t}(g(\x_t)) \nabla_\lambda g(\x_t) - \nabla f_{\lambda_t}(s_t)  u_t\|^2\\\nonumber&\quad + 2\|g(\x_t) - s_t\|^2 \\\label{stormle1:eq1}
   &= 2\|\nabla f_{\lambda_t}(g(\x_t)) \nabla g(\x_t) -  \nabla f_{\lambda_t}(s_t) \q_t \|^2 + 2\| g(\x_t) - s_t\|^2,
\end{align}
where the inequality $(a)$ is because  $\|\a+\b\|^2\leq 2\|\a \|^2 + 2\|\b \|^2$, and $|\log (x) - \log(y)|\leq |x-y|$ for all $x,y\geq 1$.

Applying the smoothness and Lipschitz continuity of $f_\lambda$ and $g$, we obtain
\begin{align}\nonumber
   &\|\nabla f_{\lambda_t}(g(\x_t)) \nabla g(\x_t) -  \nabla f_{\lambda_t}(s_t) \q_t \|^2  \\\nonumber
   &=  \|\nabla f_{\lambda_t}(g(\x_t)) \nabla g(\x_t) -\nabla f_{\lambda_t}(s_t) \nabla g(\x_t)  + \nabla f_{\lambda_t}(s_t)  \nabla g(\x_t)  -  \nabla f_{\lambda_t}(s_t) \q_t \|^2 \\\nonumber
   &\leq  2\|\nabla f_{\lambda_t}(g(\x_t)) \nabla g(\x_t) -\nabla f_{\lambda_t}(s_t) \nabla g(\x_t)\|^2 + 2\|\nabla f_{\lambda_t}(s_t)  \nabla g(\x_t)  -  \nabla f_{\lambda_t}(s_t) \q_t \|^2\\\label{stormle2:eq2}
   &\leq 2L^2_gL^2_{\nabla f_{\lambda_t}}\|s_t - g(\x_t) \|^2 + 2L_{f_{\lambda_t}}\|\q_t - \nabla g(\x_t) \|^2 + 2\| g(\x_t) - s_t\|^2.
\end{align}
  Noting $\|\q_t - \nabla g(\x_t)\|^2] =  \|\v_t - \nabla_\w g(\x_t) \|^2 + |u_t - \nabla_\lambda g(\x_t) |^2$ and combining Eqs.~(\ref{stormle1:eq1}, \ref{stormle2:eq2}), we have
\begin{align*}
  & \| \z_t-\nabla F(\x_t)\|^2 \\
   &\leq (4L^2_gL^2_{\nabla f_{\lambda_t}} + 2)\|s_t - g(\x_t) \|^2 + 4L^2_{f_{\lambda_t}}\|\q_t - \nabla g(\x_t) \|^2 \\
   &\leq 4L_F^2\|s_t - g(\x_t) \|^2 + 4L_F^2\|\q_t - \nabla g(\x_t) \|^2 \\
   &= 4L_F^2(\|s_t - g(\x_t) \|^2 +\|\v_t - \nabla_\w g(\x_t) \|^2 + |u_t - \nabla_\lambda g(\x_t) |^2 ).
\end{align*}
This complete the proof.
\end{proof}

\begin{lemma}\label{lem:storm_3} Under Assumption~\ref{ass:1},~\ref{ass:2}, run Algorithm~\ref{alg:STORM-CCMA}, and then for every $ t\in\{1,\cdots T\}$  we have
\begin{align*}
    \E[\|\varkappa_{t+1}\|^2] \leq (1-\beta_t)^2\E[\|\varkappa_t\|^2]  + 8(1-\beta_t)^2L_F^2\E[\|\x_{t+1} - \x_t\|^2]+ 6\beta_t^2\sigma^2.
\end{align*}
\end{lemma}
\begin{proof}
Since $s_{t+1} =  (g_i(\x_{t+1}) + (1-\beta)(s_t - g_i(\x_{t}))$, it holds that
\begin{align}\nonumber
&\E[\|s_{t+1} - g(\x_{t+1}) \|^2]\\\nonumber
&= \E[\|g_i(\x_{t+1}) + (1-\beta_t)(s_{t} - g_i(\x_{t}))-g(\x_{t+1})\|^2] \\\nonumber
&\leq \E[\|(1-\beta_t)(s_t - g(\x_t)) + \beta_t(g_i(\x_{t+1}) - g(\x_{t+1}))\\\nonumber&\quad+ (1-\beta_t)(g_i(\x_{t+1}) - g_i(\x_t) - (g(\x_{t+1}) - g(\x_t)))\|^2] \\\label{stormle3:eq1}
&=\E[(1-\beta_t)^2\|s_t - g(\x_{t})\|^2]+ \E[\| \beta_t(g_i(\x_{t+1}) - g(\x_{t+1})) \\\nonumber&\quad+ (1-\beta_t)(g_i(\x_{t+1}) - g_i(\x_t) - (g(\x_{t+1}) - g(\x_t)))\|^2],
\end{align}
where the last inequality is due to $\E[g_i(\x_{t+1}) - g(\x_{t+1})]=0$.

Noting $\E[\Braket{g_i(\x_{t+1}) - g_i(\x_{t+1}),g(\x_{t+1}) - g(\x_t)}]=\E[\|(g(\x_{t+1}) - g(\x_t))\|^2]$ and applying the Lipschitz continuty of $g_i(\x)$, we have
\begin{align}\nonumber
    &\E[\|g_i(\x_{t+1}) - g_i(\x_{t+1}) - (g(\x_{t+1}) - g(\x_t))\|^2]\\\nonumber
    &=\E[\|(g_i(\x_{t+1}) - g_i(\x_{t+1})\|^2+ \|(g(\x_{t+1}) - g(\x_t))\|^2-2\Braket{g_i(\x_{t+1}) - g_i(\x_{t+1}),g(\x_{t+1}) - g(\x_t)}]
    \\\nonumber
    &=\E[\|(g_i(\x_{t+1}) - g_i(\x_{t+1})\|^2- \|(g(\x_{t+1}) - g(\x_t))\|^2]\\\nonumber
    &\leq\E[\|(g_i(\x_{t+1}) - g_i(\x_{t+1})\|^2]
    \\\label{stormle3:eq2}
    &\leq L_g^2\E[\| \x_{t+1} - \x_t\|^2].
\end{align}
Combining Eqs.~(\ref{stormle3:eq1}, \ref{stormle3:eq2}) and invoking the Lipschitz continuty of $g_i(\x)$, under Assumption~\ref{ass:2}, we have
\begin{align}\nonumber
&\E[\|s_{t+1} - g(\x_{t+1}) \|^2]\\\nonumber
&\leq (1-\beta_t)^2\E[\|s_t - g(\x_{t})\|^2]\\\nonumber&\quad + 2\beta_t^2\E[\|g_i(\x_{t+1}) - g(\x_t)\|^2]+ 2(1-\beta_t)^2\E[\|g_i(\x_{t+1}) - g_i(\x_{t+1}) - (g(\x_{t+1}) - g(\x_t))\|^2] \\\label{stormle3:g}
&\leq (1-\beta_t)^2\E[\|s_t - g(\x_{t})\|^2] +2\beta_t^2\sigma^2+ 2(1-\beta_t)^2L_g^2\E[\| \x_{t+1} - \x_t\|^2].
\end{align}
In the same way, we also have
\begin{align}\label{stormle3:v}
&\E[\|\v_{t+1} - \nabla_\w g(\x_{t+1}) \|^2]\leq (1-\beta_t)^2\E[\|\v_t -\nabla_\w g(\x_t)\|^2]+ 2\beta_t^2\sigma^2+2(1-\beta_t)^2L_{\nabla g}^2\E[\| \x_{t+1} - \x_t\|^2], \\\label{stormle3:u}
&\E[|u_{t+1} - \nabla_\lambda g(\x_{t+1}) |^2]\leq (1-\beta_t)^2\E[|u_t -\nabla_\lambda g(\x_t)|^2] + 2\beta_t^2\sigma^2 + 2(1-\beta_t)^2L_{\nabla g}^2\E[\| \x_{t+1} - \x_t\|^2].
\end{align}
Therefore, combining Eqs.~(\ref{stormle3:g}, \ref{stormle3:u}, \ref{stormle3:v}), we obtain
\begin{align*}
\E[\|\varkappa_{t+1}\|^2]&\leq (1-\beta_t)^2\E[\|\varkappa_t\|^2] + 6\beta_t^2\sigma^2 + 4(1-\beta_t)^2(L^2_{\nabla g} +L_g^2)\|\x_{t+1} - \x_t\|^2)\\
&\leq (1-\beta_t)^2\E[\|\varkappa_t\|^2]  + 8(1-\beta_t)^2L_F^2\E[\|\x_{t+1} - \x_t\|^2]+ 6\beta_t^2\sigma^2,
\end{align*}
where the last inequality applies $(L^2_{\nabla g} +L_g^2)\leq 2L_F^2$.\
This complete the proof.
\end{proof}
\begin{lemma}\label{lem:storm_4} Under Assumption~\ref{ass:1} and~\ref{ass:2}, for any $\alpha>1$, let $k = \frac{\alpha\sigma^{2/3}}{L_F}$, $w = \max(2\sigma^2, (16L^2_Fk)^3)$ and $c =\frac{\sigma^2}{14L_Fk^3}+130L_F^4$. Then with $\eta_t = \frac{k}{(w+t\sigma^2)^{1/3}}$, $\beta_t = c\eta_t^2$ and after running  $T$ iterations, Algrithm~\ref{alg:STORM-CCMA} satisfies
\begin{align*}
    4L_F^4\sum\limits_{t=1}^{T}\eta_t\E[\|\varkappa_t\|^2]&\leq \frac{\E[\|\varkappa_1\|^2]}{\eta_0} -  \frac{\E[\|\varkappa_{T+1}\|^2]}{\eta_T} +\sum\limits_{t=1}^T6c^2\eta_t^3\sigma^2 + 64L_F^2\Delta.
\end{align*}
\end{lemma}
\begin{proof}

Since $w\geq (16L^2_Fk)^3$, it is easy to note that $$\eta_t \leq \eta_0 \leq \frac{1}{16L^2_F}\leq \frac{1}{4L_F}.$$

In addition, 
\begin{align*}
    \beta_t = c\eta_t^2&\leq c\eta_0^2 \leq (\frac{\sigma^2}{14L_Fk^3} + 130L_F^4) \frac{1}{256L_F^4} \\&=  \frac{\sigma^2L_F^3}{14L_F\alpha^3\sigma^2} \frac{1}{256L_F^4}+\frac{65}{128} = \frac{1}{14\alpha^3}\frac{1}{2556L_F^2}+ \frac{65}{128}\leq 1.
\end{align*}
With $\eta_t = \frac{k}{(w+t\sigma^2)^{1/3}}$, we obtain
\begin{align*}
    \frac{1}{\eta_t} - \frac{1}{\eta_{t-1}} & = \frac{(w+t\sigma^2)^{1/3} -(w+(t-1)\sigma^2)^{1/3} }{k}\overset{(a)}{\leq} \frac{\sigma^2}{3k(w+(t-1)\sigma^2)^{2/3}}\\
    &\overset{(b)}{\leq} \frac{\sigma^2}{3k(w/2+ t\sigma^2)^{2/3}}\leq \frac{\sigma^2}{3k(w/2+ t\sigma^2/2)^{2/3}} = \frac{2^{2/3}\sigma^2}{3k(w+t\sigma^2)^{2/3}}\\
    &= \frac{2^{2/3}\sigma^2}{3k^3}\eta_t^2\overset{(c)}{\leq} \frac{2^{2/3}}{12L_Fk^3}\eta_t \leq \frac{\sigma^2}{7Lk^3}\eta_t,
\end{align*}
where the inequality (a) uses the inequality $(x+y)^{1/3} - x^{1/3} \leq \frac{yx^{-2/3}}{3}$, the inequality (b) is due to $w\geq 2\sigma^2$, and the inequality (c) is due to $\eta_t\leq \frac{1}{4L_F}$.

Noting $\beta_t = c\eta_t^2$ and $0\leq(1-\beta_t)\leq 1$, by Lemma~\ref{lem:storm_3} we have
\begin{align}\nonumber
    &\frac{\E[\|\varkappa_{t+1}\|^2]}{\eta_t} - \frac{\E[\|\varkappa_{t}\|^2]}{\eta_{t-1}}\\\nonumber&\leq (\frac{(1-\beta_t)^2}{\eta_t} - \frac{1}{\eta_{t-1}})\E[\|\varkappa_t\|^2] + 6c^2\eta_t^3\sigma^2 + \frac{8(1-\beta_t)^2L_F^2}{\eta_t}\E[\|\x_{t+1} - \x_t\|^2] \\\nonumber
    &\leq (\eta_t^{-1} - \eta_{t-1}^{-1} - 2c\eta_t)\E[\|\varkappa_{t}\|^2] + 6c^2\eta_t^3\sigma^2 + \frac{8(1-\beta_t)^2L_F^2}{\eta_t}\E[\|\x_{t+1} - \x_t\|^2]
    \\\label{eqn:beforesum}
    &\leq -260L_F^4\eta_t\E[\|\varkappa_{t}\|^2] + 6c^2\eta_t^3\sigma^2 + \frac{8(1-\beta_t)^2L_F^2}{\eta_t}\E[\|\x_{t+1} - \x_t\|^2],
\end{align}
where the last inequality is due to $\eta_t^{-1} - \eta_{t-1}^{-1} - 2c\eta_t \leq \frac{\sigma^2}{7L_Fk^3}\eta_t-2(\frac{\sigma^2}{14L_Fk^3}+130L_F^4)\eta_t \leq -260L_F^4\eta_t$.

Taking  summation of Eq.~(\ref{eqn:beforesum}) from $1$ to $T$, we have
\begin{align}\label{eqn:kappa1}
260L_F^4\sum\limits_{t=1}^{T}\eta_t\E[\|\varkappa_t\|^2]&\leq \frac{\E[\|\varkappa_1\|^2]}{\eta_0} -  \frac{\E[\|\varkappa_{T+1}\|^2]}{\eta_T} +\sum\limits_{t=1}^T6c^2\eta_t^3\sigma^2 + 8L_F^2\sum\limits_{t=1}^{T}\frac{1}{\eta_t}\E[\|\x_{t+1} -\x_t \|^2].
\end{align}
In the same way with Eq.~(\ref{eqn:Yi-1}) and $\eta_t\leq \eta_1, \forall t\geq 1$, we could also have
\begin{align}\label{eq:eta_tx_t}
 \frac{1-2\eta_1 L_F}{4}   \sum\limits_{t=1}^{T}\frac{1}{\eta_t} \| \x_{t+1} -\x_t \|^2 \leq     \sum\limits_{t=1}^{T}\frac{1-2\eta_t L_F}{4\eta_t} \| \x_{t+1} -\x_t \|^2 & \leq \Delta + \sum\limits_{t=1}^{T}\eta_t\|\z_t -\nabla F(\x_t)\|^2.
\end{align}
Noting $\eta_1L_F\leq \frac{1}{4}$ and invoking Lemma~\ref{lem:storm_1}, we obtain
\begin{align}\nonumber
        \sum\limits_{t=1}^{T}\frac{1}{\eta_t}\E[ \| \x_{t+1} -\x_t \|^2]&  \leq \frac{4}{1-2\eta_1L_F}(\Delta + \sum\limits_{t=1}^{T}\eta_t\E[\|\z_t -\nabla F(\x_t)\|^2])\\\nonumber
        &\leq8\Delta +8 \sum\limits_{t=1}^{T}\eta_t\E[\|\z_t -\nabla F(\x_t)\|^2]\\\label{eqn:kappa2}
        &\leq8\Delta + 32L_F^2\sum\limits_{t=1}^{T}\eta_t\E[\|\varkappa_t\|^2].
\end{align}
Combining Eqs.~(\ref{eqn:kappa1}, \ref{eqn:kappa2}), we have
\begin{align}\label{eqstormlem3}
    4L_F^4\sum\limits_{t=1}^{T}\eta_t\E[\|\varkappa_t\|^2]&\leq \frac{\E[\|\varkappa_1\|^2]}{\eta_0} -  \frac{\E[\|\varkappa_{T+1}\|^2]}{\eta_T} +\sum\limits_{t=1}^T6c^2\eta_t^3\sigma^2 + 64L_F^2\Delta.
\end{align}
This complete the proof.
\end{proof}
\subsection{Proof of Theorem~\ref{thm:storm-ncx}}
\begin{proof}
Noting the monotonity of $\eta_t$ and dividing $\frac{\eta_1}{1/4-\eta_1 L_F/2}$ on both sides of Eq.~(\ref{eq:eta_tx_t}), we have
\begin{align}\label{thmstorm:eq1}
        \sum\limits_{t=1}^{T}\| \x_{t+1} -\x_t \|^2 & \leq\frac{1}{1/4-\eta_1 L_F/2}\left(\eta_1\Delta + \eta_1\sum\limits_{t=1}^{T}\eta_t\|\z_t -\nabla F(\x_t)\|^2\right).
\end{align}
By the same method used in the proof of Theorem 2 in~\citet{xu2019non}, we have the following inequality,
\begin{equation}
\label{eqn:thmstorm:eq2}
    \begin{aligned}
      \|\z_t -\nabla F(\x_{t+1}) + \frac{1}{\eta_t}(\x_{t} - \x_{t+1})\|^2 
      &\leq 2\| \z_t - \nabla F(\x_t)\|^2 + \frac{2\left(\bar F(\x_{t+1})-\bar F(\x_{t})\right)}{\eta_t} \\\nonumber&\quad+ (2L_F^2 + \frac{3L_F}{\eta_t})\|\x_{t+1} -\x_t\|^2. 
    \end{aligned}
\end{equation}
Multiplying $\eta_t$ on both sides of the above inequality and taking  summation from $1$ to $T$, we have
    \begin{align}\nonumber
      & \sum\limits_{t=1}^{T}\eta_t\|\z_t -\nabla F(\x_{t+1}) + \frac{1}{\eta_t}(\x_{t+1} - \x_t)\|^2 \\\nonumber
      &\overset{(a)}{\leq}2\sum\limits_{t=1}^{T}\eta_t\| \z_t - \nabla F(\x_t)\|^2 + 2\Delta + 5L_F\left(\frac{1}{1/4-\eta_1 L_F/2}\left(\eta_1\Delta + \eta_1\sum\limits_{t=1}^{T}\eta_t\|\z_t -\nabla F(\x_t)\|^2\right)\right) \\\label{thmstorm:eq3}
      &\overset{(b)}{\leq} 12\sum\limits_{t=1}^{T}\eta_t\| \z_t - \nabla F(\x_t)\|^2 + 12\Delta,
    \end{align}
where inequality (a) is due to $(2L_F^2 + \frac{3L_F}{\eta_t})\leq\frac{5L_F}{\eta_t}$, inequality (b) is due to $\eta_1 L_F\leq\frac{1}{4}$ and $\frac{1}{1/4-\eta_1 L_F/2}\leq 8$.

Combining Eqs.~(\ref{eqstormlem3}, \ref{thmstorm:eq3}) and invoking Lemma~\ref{lem:storm_1} we have
    \begin{align}\nonumber
      & \sum\limits_{t=1}^{T}\eta_t\|\z_t -\nabla F(\x_{t+1}) + \frac{1}{\eta_t}(\x_{t+1} - \x_t)\|^2 \\\nonumber
& \leq 48L_F^2\sum\limits_{t=1}^{T}\eta_t\E[\|\varkappa_t\|^2] + 12\Delta
        \\ \label{thmstorm:eq3.5}& \leq 12\left(\frac{\E[\|\varkappa_1\|^2]}{\eta_0} -  \frac{\E[\|\varkappa_{T+1}\|^2]}{\eta_T} +\sum\limits_{t=1}^T6c^2\eta_t^3\sigma^2 + 64L_F^2\Delta\right) + 12\Delta.
    \end{align}

Noting the monotonity of $\eta_t$ and dividing $T\eta_T$ on both sides of Eq.~(\ref{thmstorm:eq3.5}), we obtain
    \begin{align}\nonumber
      &\frac{1}{T}\sum\limits_{t=1}^{T}\|\z_t -\nabla F(\x_{t+1}) + \frac{1}{\eta_t}(\x_{t+1} - \x_t)\|^2 \\\label{thmstorm:eq4}&\leq 12\left(\frac{\E[\|\varkappa_1\|^2]}{T\eta_T\eta_0} -  \frac{\E[\|\varkappa_{T+1}\|^2]}{T\eta_T^2} +\frac{1}{T\eta_T}\sum\limits_{t=1}^T6c^2\eta_t^3\sigma^2 + \frac{64L_F^2\Delta}{T\eta_T}\right) + \frac{12\Delta}{T\eta_T}.
    \end{align}
Combining Eqs.~(\ref{eq:outputxR}, \ref{thmstorm:eq4}) and noting $\sum_{t=1}^T\eta_t^3\leq\mathcal{O}(\log T)$, we get the conclusion that
\begin{align*}
    \E[\dist (0, \hat{\partial}\bar F(\x_R))^2]&\leq \frac{1}{T}\sum\limits_{t=1}^{T}\E[\|\z_t -\nabla F(\x_{t+1}) + \frac{1}{\eta_t}(\x_{t+1} - \x_t)\|^2]\\&
      \leq 12\left(\frac{\E[\|\varkappa_1\|^2]}{T\eta_T\eta_0} +\frac{1}{T\eta_T}\sum\limits_{t=1}^T6c^2\eta_t^3\sigma^2 + \frac{64L_F^2\Delta}{T\eta_T}\right) + \frac{12\Delta}{T\eta_T}\\&\leq\mathcal{O}\left(\frac{\log T}{T^{2/3}}\right).
\end{align*}
This complete the proof.
\end{proof}

\section{Proofs in Section~\ref{sec:convex}}
 \subsection{Technical Lemmas}
 \begin{lemma}~\label{lem:joint-convex}
 If $\ell_i(\w)$ is convex for all $i$, we can show that $F(\w,\lambda)$ is jointly convex in terms of $(\w,\lambda)$.
\end{lemma}
 \begin{proof}
 We have
  \begin{align*}
  F(\w, \lambda) = \max_{\p\in\Delta_n}\underbrace{\sum_{i=1}^n p_i \ell_i(\w) - \lambda (\sum_{i=1}^np_i\log(np_i) -  \rho) - \lambda_0 \rho}\limits_{G(\w, \lambda,\p)}.
  \end{align*}
Since $G(\w, \lambda,\p)$ is jointly convex in terms of $(\w, \lambda)$ for every fixed $\p$, $F(\w, \lambda)$ is jointly convex in terms of $(\w, \lambda)$.
  \end{proof}
  \begin{lemma}
\label{lem:Fmulemma1}
Under Assumption~\ref{ass:1},~\ref{ass:2},  run Algorithm~\ref{alg:SCCMA} with $\eta\leq \frac{\beta}{4L_F\sqrt{9+20L_g^2} }\leq \frac{1}{6L_F}$ and apply SCDRO to the new objective $\bar F_\mu(\x)$ by adding $\mu\x_t$ to $(\nabla f_{\lambda_{t}}(s_{t})\nabla_{\w}g_i(\x_{t})^\top,\nabla f_{\lambda_{t}}(s_{t})\nabla_\lambda g_i(\x_{t})  +\log(s_{t}) + \rho)^\top$ in Eq.~(\ref{eqn:uv}) of  Algorithm~\ref{alg:SCCMA}, where $\mu$ is a small constant to be determined later. Without loss of the generality, we assume  $0<\mu\leq \frac{1}{2}$  and then we have
\begin{align*}
 \frac{1}{T}\sum\limits_{t=1}^{T}  \E[\|\z_t-\nabla F_{\mu}(\x_t)\|^2]&\leq \frac{2\E[\|\z_1-\nabla F_{\mu}(\x_1) \|^2]}{\beta T} + \frac{\Delta_\mu}{\eta T } + \frac{20L_F\E[\|g(\x_{1})-s_{1} \|^2]}{\beta T}  + 24\beta L_{F}^2\sigma^2.
\end{align*}
\end{lemma}
\begin{proof}
To facilitate our proof statement, we define the following notations:
\begin{equation*}
 \begin{aligned}
&\nabla F_{\mu}(\x_t)^\top = (\nabla_{\w} F_{\mu}(\x_t)^{\top}, \nabla_{\lambda}F_{\mu}(\x_t))\\\nonumber&\quad= (\nabla f_{\lambda_t} (g(\x_t))\nabla_{\w} g(\x_t)^\top+\mu\w_t^\top,   \nabla f_{\lambda_t}  (g(\x_t))\nabla_{\lambda} g(\x_t) + \log(g(\x_t)) + \rho+\mu\lambda_t)\\
&\widetilde{\nabla} F_{\mu}(\x_t)^\top \\\nonumber&\quad= (\nabla f_{\lambda_t} (g(\x_t))\nabla_{\w} g_i(\x_t)^\top+\mu\w_t^\top,   \nabla f_{\lambda_t}  (g(\x_t))\nabla_{\lambda} g_i(\x_t) + \log(g(\x_t)) + \rho+\mu\lambda_t)\\
&G_\mu(\x_t)^\top \\\nonumber&\quad= (G_{\w_t}(\x_t)^\top, G_{\lambda_t}(\x_t))=(\nabla f_{\lambda_{t}}(s_{t})\nabla_{\w} g_i(\x_t)^\top+\mu\w_t^\top, \nabla f_{\lambda_{t}}(s_{t})\nabla_{\lambda} g_i(\x_t)  +\log(s_{t}) + \rho+\mu\lambda_t).
\end{aligned}   
\end{equation*}
It is worth to notice that $\E[\widetilde{\nabla}F_{\mu}(\x_t)] = \nabla F_{\mu}(\x_t)$.

Since $F(\x)$ is $L_F$-smooth, then we have $F_\mu(\x)$ is $L_{F_\mu}$-smooth, where $L_{F_\mu}=(L_F+\mu)$. Noting $L_F>1$ and $\mu\leq\frac{1}{2}$, we obtain $L_F+\mu\leq\frac{3}{2}L_F$.
For every iteration $t$, by simple expansion we have
\begin{equation*}
    \begin{aligned}
    I_t & = \E[\|\nabla F_{\mu}(\x_t)-\z_t \|^2] \\
    & =  \E[\|\nabla F_{\mu}(\x_t)-(1- \beta)\z_{t-1} - \beta G_\mu(\x_t) \|^2]\\
    &=  \E[\|(1- \beta)(\nabla F_{\mu}(\x_t) - \nabla F_{\mu}(\x_{t-1}))+(1- \beta) \nabla F_{\mu}(\x_{t-1})- (1- \beta)\z_{t-1} +  \beta \nabla F_{\mu}(\x_t)- \beta G_\mu(\x_t) \|^2]\\
    & =  \E[\|(1- \beta)(\nabla F_{\mu}(\x_t)  - \nabla F_{\mu}(\x_{t-1}))+ (1- \beta)(\nabla F_{\mu}(\x_{t-1})-\z_{t-1})\|^2] \\&\quad+ \E[\|\beta( \widetilde{\nabla} F_{\mu}(\x_t)- G_\mu(\x_t)) + \beta(\nabla F_{\mu}(\x_{t}) -\widetilde{\nabla}F_{\mu}(\x_t))\|^2]\\
        & =  \E[\|(1- \beta)\underbrace{(\nabla F_{\mu}(\x_t)  - \nabla F_{\mu}(\x_{t-1})}_{A})+ (1- \beta)\underbrace{(\nabla F(\x_{t-1})-\z_{t-1})}_{B}\|^2] \\&\quad+ \E[\|\beta(\underbrace{ \widetilde{\nabla} F(\x_t)- G(\x_t)}_{C}) + \beta\underbrace{(\nabla F(\x_{t}) -\widetilde{\nabla}F(\x_t))}_{D}\|^2].
\end{aligned}
\end{equation*}
The above inequality shows that the only difference between $I_t$ in the proof of  Lemma~\ref{lem:dist_cum_var_lemma_2} and $I_t$ in the proof of Lemma~\ref{lem:Fmulemma1} is term $A$.

Therefore, by the same method used in the proof of Lemma~\ref{lem:dist_cum_var_lemma_2}, we have
\begin{align}\nonumber
  & \sum\limits_{t=1}^{T}  \E[\|\z_t-\nabla F_{\mu}(\x_t)\|^2]
     \\\nonumber&\leq\frac{\E[\|\nabla F_{\mu}(\x_{1}) - \z_{1}\|^2]}{\beta} + (\frac{4L_{F_\mu}^2}{\beta^2}+\frac{20L_{F}^2L_g^2}{\beta^2})\bigg ( \frac{\eta}{1/4- \eta L_{F_\mu}/2} (\Delta_\mu  + \eta\sum\limits_{t=1}^{T}\E[\|\z_t -\nabla F_{\mu}(\x_t)\|^2])\bigg ) \\\nonumber
     &\quad+10L_F^2\bigg (\frac{\E[\|g(\x_{1}) -s_{1} \|^2]}{\beta} + \beta T\sigma^2 \bigg)+ 2\beta L^2_FT\sigma^2.
\end{align}
By $L_{F_\mu}\leq \frac{3}{2}L_F $ and $\eta L_{F_\mu}\leq\frac{3}{2}\eta L_F\leq1/4$, it holds that
\begin{align}\nonumber
  & \sum\limits_{t=1}^{T}  \E[\|\z_t-\nabla F_{\mu}(\x_t)\|^2]
    \\\nonumber
     &\leq\frac{\E[\|\nabla F_{\mu}(\x_{1}) - \z_{1}\|^2]}{\beta} + (\frac{9L_F^2}{\beta^2}+\frac{20L_F^2L_g^2}{\beta^2})\left( 8\eta \left(\Delta_\mu  + \eta\sum\limits_{t=1}^{T}\E[\|\z_t -\nabla F_{\mu}(\x_t)\|^2]\right)\right)\\\nonumber
     &\quad+10L_F^2\bigg (\frac{\E[\|g(\x_{1}) -s_{1} \|^2]}{\beta} + \beta T\sigma^2 \bigg)+ 2\beta L^2_FT\sigma^2\\\nonumber
    &\leq \frac{\E[\|\z_1-\nabla F_{\mu}(\x_1) \|^2]}{\beta} + \frac{\Delta_\mu}{2\eta} +\frac{1}{2}\sum\limits_{t=1}^{T}\E[\|\z_t -\nabla F_{\mu}(\x_t)\|^2]\\\label{eq:arragemu}&\quad+10L_F^2\bigg (\frac{\E[\|g(\x_{1}) -s_{1} \|^2]}{\beta}+ \beta T\sigma^2 \bigg)+ 2\beta L_{F}^2T\sigma^2,
\end{align}
 where the last inequality is due to $8(9L_F^2 + 20L_F^2L_g^2)\eta^2\leq \frac{\beta^2}{2}$.

Rearranging terms and dividing $T$ on both sides of Eq.~(\ref{eq:arragemu}), we complete the proof of this Lemma.
\end{proof}
\begin{lemma}
\label{lem:storm_2} At the $k$-th stage of RASCDRO, let $\beta_k = c\eta_k^2$ and $c = 512L_F^4$ we have
\begin{align}\label{eqn:storm_2_1}
    \frac{1}{8L_F^2T_k}\sum\limits_{t=1}^{T_k}\E[\|  \z_t - \nabla F_{\mu}(\x_t) \|^2]\leq \frac{\E[\|\varkappa_{k}\|^2]}{\beta_kT_k}   
+6\beta_k\sigma^2+ \frac{64L_F^2\E[\Delta^\mu_{k}]\eta_k}{\beta_kT_k},
\end{align}
where $\Delta^\mu_{k}=F_{\mu}(\x_k)-\inf_{\x\in\X}F_{\mu}(\x)$.
\end{lemma}
\begin{proof}
Recall the definition of $ \|\varkappa_t\|^2$ and by the same proof of Lemma~\ref{lem:storm_1} we have
\begin{equation}\label{eq:mukappa}
    \| \z_t-\nabla F_\mu(\x_t) \|^2 \leq 4L_F^2 \|\varkappa_t\|^2.
\end{equation}
 Denote $\varkappa_{t}$ at $k$th-stage as  $\varkappa_k^{t}$, and by Lemma~\ref{lem:storm_3}, at the $k$th-stage in RASCDRO we have 
\begin{align}\nonumber
%&\frac{1}{4L_F^2}\E[\| \z_k^{t+1}-\nabla F_{\mu}(\x_k^{t+1}) \|^2]\\&\leq 
\E[\|\varkappa_k^{t+1}\|^2]
    &\leq (1-\beta_k)^2\|\varkappa_k^{t}\|^2 + 6\beta_k^2\sigma^2 + 8L_F^2(1-\beta_k)^2\|\x_{t+1} - \x_t\|^2\\\nonumber
    &\leq (1-\beta_k)^{2t}\|\varkappa_{k}\|^2 + 6\beta_k^2\sigma^2\sum\limits_{i=1}^t (1-\beta_k)^{2(t-i)} \\\nonumber&\quad+8L_F^2(1-\beta_k)^2\sum\limits_{i=1}^t (1-\beta_k)^{2(t-i)}\| \x_{i+1}- \x_{i}\|^2
    \\\label{eq:mukappa2}
    &\leq (1-\beta_k)^{2t}\|\varkappa_{k}\|^2 + 6\beta_k\sigma^2 \\\nonumber&\quad+8L_F^2(1-\beta_k)^2\sum\limits_{i=1}^t (1-\beta_k)^{2(t-i)}\| \x_{i+1}- \x_{i}\|^2.
\end{align}
Combining Eqs.~(\ref{eq:mukappa},\ref{eq:mukappa2}), we obtain
\begin{align*}
    &\frac{1}{4L_F^2T_k}\sum\limits_{t=1}^{T_k}\E[\|  \z_t - \nabla F_{\mu}(\x_t) \|^2]\\
    &\leq\frac{1}{T_k}\sum\limits_{t=1}^{T_k}\E[(1-\beta_k)^2\|\varkappa_k^t\|^2+6\beta_k^2\sigma^2 + 8L_F^2(1-\beta_k)^2\|\x_{t+1} - \x_t\|^2]\\
    &\leq \frac{1}{T_k}\sum\limits_{t=1}^{T_k}(1-\beta_k)^{2t-2}\E[\|\varkappa_{k}\|^2] +6\beta_k\sigma^2  +\frac{8L_F^2(1-\beta_k)^2}{T_k}  \sum\limits_{t=1}^{T_k}\sum\limits_{i=1}^{t-1} (1-\beta_k)^{2(t-i)}\E[\| \x_{i+1}- \x_{i}\|^2].
\end{align*}
Noting $\sum_{t=1}^{T_k}(1-\beta_k)^{2t-2}\leq1/\beta_k$ and invoking Eq.~(\ref{thmstorm:eq1}), we have 
\begin{align*}
    &\frac{1}{4L_F^2T_k}\sum\limits_{t=1}^{T_k}\E[\|  \z_t - \nabla F_{\mu}(\x_t) \|^2]\\
    &\leq\frac{\E[\|\varkappa_{k}\|^2]}{\beta_kT_k}   +6\beta_k\sigma^2+ \frac{8L_F^2(1-\beta_k)^2}{\beta_kT_k}\sum\limits_{t=1}^{T_k}\E[\|\x_{t+1}-\x_t \|^2]\\
      &\leq\frac{\E[\|\varkappa_{k}\|^2]}{\beta_kT_k}   
+6\beta_k\sigma^2+ \frac{8L_F^2(1-\beta_k)^2}{\beta_kT_k}\left(\frac{\eta_{k}}{1/4-\eta_{k}L_{F_\mu}/2}\left(\E[\Delta^\mu_{k}] + \eta_{k}\sum\limits_{t=1}^{T_{k}}\E[\|\z_t-\nabla F_{\mu}(\x_t)\|^2]\right)\right)\\
    &\leq \frac{\E[\|\varkappa_{k}\|^2]}{\beta_kT_k}   
+6\beta_k\sigma^2+ \frac{64L_F^2\E[\Delta^\mu_{k}]\eta_k}{\beta_kT_k}+\frac{64L_F^2\eta_k^2}{\beta_kT_k}\sum\limits_{t=1}^{T_k}\|\z_t -\nabla F_{\mu}(\x_t)\|^2],
\end{align*}
where the last inequality is due to $1/(1/4-\eta_{k}L_{F_\mu}/2)\leq 8, (1-\beta_k)^2\leq 1, L^2_{\nabla g} +L_g^2 \leq 2L_F^2$. 

Invoking $\beta_k = c\eta_k^2$ and $c = 576L_F^4$ to above inequality, we  get the conclusion that
\begin{equation*}
    \frac{1}{8L_F^2T_k}\sum\limits_{t=1}^{T_k}\E[\|  \z_t - \nabla F_{\mu}(\x_t) \|^2]\leq \frac{\E[\|\varkappa_{k}\|^2]}{\beta_kT_k}   
+6\beta_k\sigma^2+ \frac{64L_F^2\E[\Delta^\mu_{k}]\eta_k}{\beta_kT_k}.
\end{equation*}
%Then multiply $8L_F^2$ on both sides, we finish the lemma.
\end{proof}

  \subsection{Proof of Lemma~\ref{le:kl}}
  \begin{proof}
 Since $\ell_i(\w)$ is convex for all $i$, by Lemma~\ref{lem:joint-convex} we know $F(\x)$ is convex.
  And thus by the definition of $\bar F_\mu(\mathbf{x})$ we have $\bar F_\mu(\mathbf{x})$ is a strongly convex function. Then by strong convexity, we have
$$
\bar F_\mu(\mathbf{y}) \geq \bar F_\mu(\mathbf{x})+\mathbf{v}^{\top}(\mathbf{y}-\mathbf{x})+\frac{\mu}{2}\|\mathbf{y}-\mathbf{x}\|^{2}, \forall \mathbf{x}, \mathbf{y} \in \X, \mathbf{v} \in \partial \bar F_\mu(\mathbf{x}).
$$
Then
$$
\begin{aligned}
\inf_{\x\in\X}\bar F_\mu\left(\mathbf{x}\right) & \geq \min _{\mathbf{y} \in \X} \bar F_\mu(\mathbf{x})+\mathbf{v}^{\top}(\mathbf{y}-\mathbf{x})+\frac{\mu}{2}\|\mathbf{y}-\mathbf{x}\|^{2} \\&\geq \min _{\mathbf{y}} \bar F_\mu(\mathbf{x})+\mathbf{v}^{\top}(\mathbf{y}-\mathbf{x})+\frac{\mu}{2}\|\mathbf{y}-\mathbf{x}\|^{2} \\
&=\bar F_\mu(\mathbf{x})-\frac{\|\mathbf{v}\|^{2}}{2 \mu}, \quad \forall \mathbf{v} \in \partial \bar F_\mu(\mathbf{x}).
\end{aligned}
$$Hence, $\frac{\|\mathbf{v}\|^{2}}{2 \mu} \geq \bar F_\mu(\mathbf{x})-\inf_{\x\in\X}\bar F_\mu\left(\mathbf{x}\right), \forall \mathbf{v} \in \partial \bar F_\mu(\mathbf{x})$, which implies
$$ 
\operatorname{dist}(0, \partial \bar F_\mu(\mathbf{x}))^{2} \geq 2 \mu\left(\bar F_\mu(\mathbf{x})-\bar F_\mu\left(\mathbf{x}_{*}\right)\right).$$
  \end{proof}
\subsection{Proof of Lemma~\ref{lem:stage_variance}}
\begin{proof}
We use inductions to prove $\E[\|\z_k - \nabla F_{\mu}(\x_k)\|^2]\leq \mu\epsilon_k/4$, $\E[\|g(\x_k) - s_{k}\|^2]\leq \mu\epsilon_k/4$ and $\E[F_{\mu}(\x_{k}) -\inf\limits_{\x \in\X}F_{\mu}(\x)]\leq  \epsilon_{k}$.
Let's consider the first stage in the beginning.

Let $\epsilon_1= \Delta_\mu$, thus $\E[F_{\mu}(\x_{1}) -\inf_{\x \in\X}F_{\mu}(\x)]\leq  \epsilon_{1}$. And we can use a batch size of $4/\mu\epsilon_1$ for initialization.to make sure $\E[\|\nabla F_{\mu}(\x_{1}) - \z_{1}\|^2]\leq \mu\epsilon_{1}/4$,$\E[\|s_{1} - g(\x_{1})\|^2]\leq \mu\epsilon_{1}/4$. 

% Without the loss of generality, we consider the case $\Delta^\mu \geq  \frac{\sigma^2}{\mu }$. The case that $\Delta^\mu \leq \frac{\sigma^2}{\mu }$ can be simply covered by the same analysis as well.
% $\eta_1 = \frac{1}{1300L_F^4}\leq  \frac{1}{1300L_F}\leq\frac{1}{4L_F}$, $L_F\geq 1$, $\beta_1 =\frac{1}{96L_F^2}$, and $8(9L_F^2 + 20L_F^2L_g^2)\eta_1^2\leq 168L_F^4\eta_1^2\leq \beta^2_1$, $T_1 =\max\{\frac{15360 \Delta^\mu L_F^4}{\sigma^2}, \frac{15360 L_F^4}{\mu }\}$

% \begin{equation}
% \label{eqn:thm15-1}
%     \begin{aligned}
% \E[\|\z_{k}-\nabla F_{\mu}(\x_k)\|^2]&= \frac{1}{T_1}\sum\limits_{t=1}^{T_1}  \E[\|\z_t-\nabla F_{\mu}(\x_t)\|^2]\\&\leq \frac{\E[2\|\nabla F_{\mu}(\x_1)-\z_1 \|^2]}{\beta_1 T_1} + \frac{\Delta^\mu }{\eta_1 T_1 } + \frac{20L_F\E[\|g(\x_1)-s_1 \|^2]}{\beta_1 T_1}  + 24\beta_1 L_{F}^2\sigma^2 \\
%  &\leq \frac{2}{\beta_1T_1}\frac{\epsilon_1}{2} + \frac{\Delta^\mu }{\eta_1T_1} + \frac{20L_F}{\beta_1T_1}\cdot\frac{\epsilon_1}{2} + 24\beta_1L_F^2\sigma^2\\
%  &\leq \frac{\mu \epsilon_1}{160L^2_F} + \frac{\sigma^2}{44} + \frac{\mu \epsilon_1}{16L_F} + \frac{\sigma^2}{4}\\
%  &\overset{L_F\geq 1}{\leq}\frac{\mu \epsilon_1}{160} +\frac{\mu\epsilon_1}{44} + \frac{\mu \epsilon_1}{16} + \frac{\mu \epsilon_1}{4}\\
%  &\leq \frac{\mu \epsilon_1}{2}
%     \end{aligned}
% \end{equation}

Suppose that $\E[\|g(\x_{k-1})-s_{k-1}\|^2]\leq \mu \epsilon_{k-1}/4$, $\E[\|\z_{k-1}-\nabla F_{\mu}(\x_{k-1})\|^2]\leq\mu \epsilon_{k-1}/4$ and 
      $\E[F_{\mu}(\x_{k-1}) -\inf_{\x \in\X}F_{\mu}(\x)]\leq  \epsilon_{k-1}$
. By setting $\beta_{k-1} =\min\{ \frac{\mu\epsilon_{k-1}}{384L_F^2\sigma^2}, \frac{1}{384L_F^2}\}, \eta_{k-1} =\min\{ \frac{\mu\epsilon_{k-1}}{4608L_F^4\sigma^2}, \frac{1}{4608L_F^4}\}$ and $T_{k-1} =\max\{\frac{147456L_F^4\sigma^2}{\mu^2\epsilon_{k-1}}, \frac{147456 L_F^4}{\mu}\}$,  it is easy to obtain that $\eta_{k-1}\leq \frac{\beta_{k-1}}{4L_F\sqrt{9+20L_g^2} }$. Therefore, invoking Lemma~\ref{lem:Fmulemma1} we have
\begin{align*}
    &\E[\|\z_{k}-\nabla F_{\mu}(\x_k)\|^2]\\&\leq \frac{1}{T_{k-1}}\sum\limits_{t=1}^{T_{k-1}}  \E[\|\z_t-\nabla F_{\mu}(\x_t)\|^2]\\
    &\leq \frac{\E[2\|\z_{k-1}-\nabla F_{\mu}(\x_{k-1}) \|^2]}{\beta_{k-1} T_{k-1}} + \frac{\E[\Delta^\mu_{k-1}]}{\eta_{k-1} T_{k-1} } + \frac{20L_F\E[\|g(\x_{k-1})-s_{k-1}\|^2]}{\beta_{k-1} T_{k-1}}  + 24\beta_k L_{F}^2\sigma^2 \\
        &\leq \frac{\mu\epsilon_{k-1}}{2\beta_{k-1} T_{k-1}} + \frac{\epsilon_{k-1}}{\eta_{k-1} T_{k-1} } + \frac{5L_F\mu\epsilon_{k-1}}{\beta_{k-1} T_{k-1}}  + 24\beta_{k-1} L_{F}^2\sigma^2.
\end{align*}
 Without  loss of the generality, we consider the case $\mu\epsilon_{k-1}/\sigma^2\leq 1$.  By definition  have $\beta_{k-1}=\mu\epsilon_{k-1}/(384L_F^2\sigma^2)$, $\eta_{k-1}=\mu\epsilon_{k-1}/(4608L_F^4\sigma^2)$ and $T_{k-1}=147456L_F^4\sigma^2/(\mu^2\epsilon_{k-1})$, which imply$$ \frac{1}{\beta_{k-1} T_{k-1}}\leq\frac{\mu}{384L_F^2},\ \ \frac{1}{\eta_{k-1} T_{k-1} }\leq\frac{\mu}{32}\text{ and } 24\beta_{k-1} L_{F}^2\sigma^2\leq\frac{\mu\epsilon_{k-1}}{16}.$$
Then, note $L_F\geq 1$, $\mu<1$ and $\epsilon_k = \epsilon_{k-1}/2$ we have
\begin{align*}
    \E[\|\z_{k}-\nabla F_{\mu}(\x_k)\|^2]&\leq\frac{\mu^2\epsilon_{k-1}}{768L_F^2 } + \frac{\mu\epsilon_{k-1}}{16} + \frac{5\mu^2\epsilon_{k-1}}{384L_F} + \frac{\mu\epsilon_{k-1}}{8} \\
   &\leq\frac{\mu\epsilon_{k-1}}{768} + \frac{\mu\epsilon_{k-1}}{32} + \frac{5\mu\epsilon_{k}}{192} + \frac{\mu\epsilon_{k-1}}{16} \\
   &=\frac{\mu\epsilon_{k}}{192} + \frac{\mu\epsilon_k}{16} + \frac{\mu\epsilon_k}{40} + \frac{\mu\epsilon_k}{8}
   \\&\leq 
   \frac{\mu\epsilon_k}{4}.
\end{align*}
%where the inequality (a) is due to $L_F\geq 1$ and $\mu<1$ and the equality (b) is due to $\epsilon_k = \epsilon_{k-1}/4$.
Next we need to show $\E[\|g(\x_k) - s_{k}\|^2]\leq \mu\epsilon_k/4$ under the assumption that $\E[\|g(\x_{k-1}) - s_{k-1}\|^2]\leq \mu\epsilon_{k-1}/4$.

By Lemma~\ref{lem:recur-g}, we have
\begin{equation*}
    \begin{aligned}
     &\E[\|g(\x_{k}) - s_k\|^2]\\&=\frac{1}{T_{k-1}}    \sum\limits_{t=1}^{T_{k-1}}\E[\|g(\x_{t})-s_{t} \|^2]\\
     &\leq \frac{\E[ \|g(\x_{k-1}) -s_{k-1} \|^2]}{\beta_{k-1}T_{k-1}} + \frac{2L_g^2}{\beta_{k-1}^2T_{k-1}}\sum\limits_{t=1}^{T_{k-1}}\E[\|\x_{t+1} - \x_t \|^2] + \beta_{k-1} \sigma^2\ \\
     &\leq\frac{\mu \epsilon_{k-1}}{4\beta_{k-1}T_{k-1}}  + \frac{2L_g^2}{\beta_{k-1}^2T_{k-1}}\left(\frac{\eta_{k-1}}{1/4-\eta_{k-1}L_{F_\mu}/2}\left(\E[\Delta^\mu_{k-1}] + \eta_{k-1}\sum\limits_{t=1}^{T_{k-1}}\E[\|\z_t-\nabla F_{\mu}(\x_t)\|^2]\right)\right)\\\nonumber&\quad+ \beta_{k-1} \sigma^2,
    \end{aligned}
\end{equation*}
where $\Delta^\mu_{k-1}=F_{\mu}(\x_{k-1})-\inf_{\x\in\X}F_{\mu}(\x)$.
With $1/(1/4-\eta_{k-1}L_{F_\mu}/2)\leq 8 $, $\E[\|g(\x_{k-1})-s_{k-1}\|^2]\leq \mu \epsilon_{k-1}/4$ and 
      $\E[F_{\mu}(\x_{k-1}) -\inf_{\x \in\X}F_{\mu}(\x)]\leq  \epsilon_{k-1}$, it holds that
\begin{equation*}
    \begin{aligned}
     \E[\|g(\x_{k}) - s_k\|^2]
   &\leq\frac{\mu \epsilon_{k-1}}{2\beta_{k-1}T_{k-1}} + \frac{16L_g^2\eta_{k-1} \epsilon_{k-1}}{\beta_{k-1}^2T_{k-1}}+ \frac{4L_g^2\eta^2_{k-1}\mu\epsilon_{k-1} }{\beta_{k-1}^2}+ \beta_{k-1} \sigma^2
    \\&\leq\frac{\mu \epsilon_{k}}{384L_F^2} + \frac{L_g^2\mu \epsilon_{k-1}}{288L_F^4}+ \frac{L_g^2\mu\epsilon_{k-1} }{36L_F^4}+ \frac{\mu\epsilon_{k-1}}{192L_F^2}\\&\leq  \frac{\mu\epsilon_k}{2}.
    \end{aligned}
\end{equation*}

Invoking Lemma~\ref{lem:dist_lemma_1}, at $(k-1)$-th stage ($k> 1$) we have 
\begin{align*}
    & \E[\dist (0, \hat{\partial}\bar F_\mu(\x_k))^2]\\
       &\leq \frac{2+40L_{F_\mu}\eta_{k-1}}{T_{k-1}}\sum\limits_{t=1}^{T_{k-1}}\E[\|\z_t -\nabla F_{\mu}(\x_t)\|^2] + \frac{2\E[\Delta^\mu_{k-1}]}{\eta_{k-1} T_{k-1}} + \frac{40L_{F_\mu}\E[\Delta^\mu_{k-1}]}{T_{k-1}}\\
        &\leq \frac{(2+40L_{F_\mu}\eta_{k-1})\mu\epsilon_{k-1}}{4}+ \frac{2\epsilon_{k-1}}{\eta_{k-1} T_{k-1}} + \frac{40L_{F_\mu}\epsilon_{k-1}}{T_{k-1}}
         \\&\leq \frac{197\mu\epsilon_{k}}{192}+ \frac{\mu\epsilon_{k}}{8} + \frac{40L_{F_\mu}\mu\epsilon_{k-1}}{147456L_F^4}\\&\leq 2\mu\epsilon_k,
\end{align*}
where the second inequality is due to $L_{F_\mu}\eta_{k-1}\leq(3/2)L_F\eta_{k-1}\leq1/1536$.
% \begin{align*}
%     & \E[\dist (0, \hat{\partial}\bar F_\mu(\x_k))^2]\\
%       &\leq \frac{197\eta_{k-1}}{T_{k-1}}\sum\limits_{t=1}^{T_{k-1}}\E[\|\z_t -\nabla F(\x_t)\|^2] + \frac{2\E[\Delta^\mu_{k-1}]}{\eta_{k-1} T_{k-1}} + \frac{40L_{F_\mu}\E[\Delta^\mu_{k-1}]}{T_{k-1}}\\
%         &\leq \frac{(2+40L_{F_\mu}\eta_{k-1})\mu\epsilon_{k}}{2}+ \frac{2\E[\Delta^\mu_{k-1}]}{\eta_{k-1} T_{k-1}} + \frac{40L_{F_\mu}\E[\Delta^\mu_{k-1}]}{T_{k-1}}\\
%      &\overset{\eta_{k-1} L_F \leq 1/1300}{\leq} \frac{21}{10T_{k-1}}\sum\limits_{t=1}^{T_{k-1}}\E[\|\z_t -\nabla F_{\mu}(\x_t)\|^2] + \frac{2\E[\Delta^\mu_{k-1}]}{\eta_{k-1} T_{k-1}} +
%      \frac{20L_F\E[\Delta^\mu_{k-1}]}{T_{k-1}}\\
%      &= \frac{21\E[\|\z_{k} - \nabla F_{\mu}(\x_{k}) \|^2]}{10} + \frac{2\E[\Delta^\mu_{k-1}]}{\eta_{k-1} T_{k-1}} +
%      \frac{20L_F\E[\Delta^\mu_{k-1}]}{T_{k-1}} \\
%      &\overset{Lemma~\ref{lem:stage_variance}}{\leq} \frac{21\mu\epsilon_{k}}{20} + \frac{2\mu\epsilon_{k-1}}{11} + \frac{\mu^2\epsilon_k\epsilon_{k-1}}{768L_F^3}\\
%      &\overset{\epsilon_k = \epsilon_{k-1}/4}{\leq}\frac{21\mu\epsilon_k}{20} + \frac{2\mu\epsilon_k}{5} +\frac{\mu\epsilon_k}{768L_F^2}\leq 2\mu\epsilon_k
% \end{align*}

Since $F_{\mu}(\x_k)\leq \bar F_{\mu}(\x_k)$ and $ \inf_{\x \in\X}F_{\mu}(\x)=\inf_{\x \in\X} \bar F_{\mu}(\x)$, applying Lemma~\ref{le:kl}  we have
\begin{equation*}
    \begin{aligned}
      \E[F_{\mu}(\x_k) - \inf\limits_{\x \in\X}F_{\mu}(\x)]\leq \E[ \bar F_{\mu}(\x_k) -\inf\limits_{\x \in\X}  \bar F_{\mu}(\x)]\leq \frac{1}{2\mu}\E[\dist (0, \hat{\partial} \bar F_{\mu}(\x_k))^2]\leq \frac{2\mu\epsilon_k}{2\mu} = \epsilon_k.
    \end{aligned}
\end{equation*}
This complete the proof of this Lemma.
\end{proof}
\subsection{Proof of Theorem~\ref{thm:KL-RSCCMA}}
\begin{proof}
Invoking Lemma~\ref{lem:stage_variance}, then after $K = \mathcal O(\log_2(\epsilon_1/\epsilon))$ stages, we have
\begin{align*}
    \E[F_{\mu}(\x_K) - \inf\limits_{\x \in\X}F_{\mu}(\x)]\leq \epsilon_K = \frac{\epsilon_1}{2^{K-1}} = \epsilon.
\end{align*}
Since $\sum_{k=1}^K 2^k=\mathcal O (1/\epsilon)$, the overall oracle complexity is
\begin{equation*}
\begin{aligned}
 \sum\limits_{k=1}^{K}T_k +\frac{4}{\mu\epsilon_1}
&\leq36864\sigma^2L_F^4\sum\limits_{k=2}^K\frac{1}{\mu^2\epsilon_k}+\frac{4}{\mu\epsilon_1}\\
&\leq \frac{36864\sigma^2L_F^4}{\mu^2\epsilon}\sum\limits_{k=1}^K \frac{1}{2^k} +\frac{4}{\mu\epsilon_1} \\
&\leq \mathcal O(\frac{1}{\mu^2\epsilon}).
\end{aligned}
\end{equation*}
\end{proof}

\subsection{Proof of Corollary~\ref{cor:CX-RSCCMA}}\label{app:cor1}
It is easy to note that $F_\mu(\x_K) - F_\mu(\x_*)\leq F_\mu(\x_K) - \inf_{\x\in\X} F_\mu(\x)$, where $\x_*=\argmin_{\x\in\X}F(\x)$. Therefore, if after $K$ stages it holds that $\E[F_\mu(\x_K) - \inf_{\x\in\X} F_\mu(\x)]\leq\epsilon/2$ with an oracle complexity of $\mathcal O(1/\mu^2\epsilon)$, we have  $\E[F_\mu(\x_K) - F_\mu(\x_*)]\leq\epsilon/2$ , i.e., $\E[F(\x_K) + \mu\| \x_K\|^2/2 - F(\x_*) - \mu\|\x_* \|^2/2]\leq \epsilon/2$. By Assumption~\ref{ass:1}(a) $\mathcal W$ is bounded by $R$, and then by setting $\mu = \epsilon/(2 (R^2+\tilde\lambda^2))$, with $\|\x\|^2\leq (R^2+\tilde\lambda^2)$ we have
\begin{align*}
  \E[F(\x_K) - F(\x_*)] \leq \frac{\epsilon}{2} + (2 (R^2+\tilde\lambda^2))\frac{\mu}{2} \leq \frac{\epsilon}{2} + \frac{\epsilon}{2} \leq \epsilon
\end{align*}
with an oracle complexity of $\mathcal O(1/\epsilon^3)$.

\subsection{Proof of Lemma~\ref{lem:storm-var}}
\begin{proof}
We use inductions to prove $\E[\|\varkappa_{k}\|^2]\leq \mu\epsilon_k/16L_F^2$ and $\E[F_{\mu}(\x_{k}) -\inf\limits_{\x \in\X}F_{\mu}(\x)]\leq  \epsilon_{k}$.
Let's consider the first stage in the beginning.

Let $\epsilon_1= \Delta_\mu$, thus $\E[F_{\mu}(\x_{1}) -\inf_{\x \in\X}F_{\mu}(\x)]\leq  \epsilon_{1}$. And we can use a batch size of $48L_F^2/\mu\epsilon_1$ for initialization to make sure $\E[\|\varkappa_{1}\|^2] = \E[\|s_1 - g(\x_1) \|^2 + \|\v_1 - \nabla_\w g(\x_1) \|^2 + |u_1 - \nabla_\lambda g(\x_1)|^2]\leq \mu\epsilon_1/16L_F^2$.

Suppose that $\E[\|\varkappa_{k-1}\|^2]\leq \mu\epsilon_{k-1}/16L_F^2$ and 
      $\E[F_{\mu}(\x_{k-1}) -\inf_{\x \in\X}F_{\mu}(\x)]\leq  \epsilon_{k-1}$. By setting $\beta_{k-1} =\min\{ \frac{\mu\epsilon_{k-1}}{768L_F^2\sigma^2}, \frac{1}{768L_F^2}\}, \eta_{k-1} =\min\{ \frac{\sqrt{\mu\epsilon_{k-1}}}{18432L_F^3\sigma^2}, \frac{1}{18432L_F^4}\}$ and $T_{k-1} =\max\{\frac{147456L_F^3\sigma}{\mu^{3/2}\sqrt{\epsilon_{k-1}}},\frac{147456L_F^4\sigma^2}{\mu\epsilon_{k-1}}, \frac{147456 L_F^4}{\mu}\}$. 

Then following the above Lemma~\ref{lem:storm_2}, for $k\geq 1$,
\begin{align*}
    \E[\|\varkappa_{k}\|^2]&\leq\frac{1}{4L_F^2T_{k-1}}\sum\limits_{t=1}^{T_{k-1}}\E[\|  \z_t - \nabla F_{\mu}(\x_t) \|^2]\\&\leq \frac{2\E[\|\varkappa_{k-1}\|^2]}{\beta_{k-1}T_{k-1}}   
+12\beta_{k-1}\sigma^2+ \frac{128L_F^2\E[\Delta^\mu_{k-1}]\eta_{k-1}}{\beta_{k-1}T_{k-1}}\\&\leq \frac{\mu\epsilon_{k-1}}{4L_F^2\beta_{k-1}T_{k-1}}   
+12\beta_{k-1}\sigma^2+ \frac{128L_F^2\epsilon_{k-1}\eta_{k-1}}{\beta_{k-1}T_{k-1}}.
\end{align*}
Without  loss of the generality, we consider the case $\mu\epsilon_{k-1}/\sigma^2\leq 1$.  By definition we have $\beta_{k-1}=\mu\epsilon_{k-1}/(768L_F^2\sigma^2)$, $\eta_{k-1}=\sqrt{\mu\epsilon_{k-1}}/(9216L_F^3\sigma)$, which imply$$ \frac{1}{\beta_{k-1} T_{k-1}}\leq\frac{1}{96L_F^2},\ \ \frac{1}{\eta_{k-1} T_{k-1} }\leq\frac{\mu}{8}\text{ and } 12\beta_k\sigma^2\leq\frac{\mu\epsilon_{k-1}}{64L_F^2}.$$

Then, noting $L_F\geq 1$, $\mu<1$ and $\epsilon_k = \epsilon_{k-1}/2$ we have
\begin{align*}
    \E[\|\varkappa_{k}\|^2]&\leq\frac{\mu\epsilon_{k-1}}{384L_F^4 } + \frac{\mu\epsilon_{k-1}}{64L_F^2}  + \frac{\mu\epsilon_{k-1}}{6912L_F^4} \\
   &\leq\frac{\mu\epsilon_{k}}{192L_F^2 } + \frac{\mu\epsilon_{k}}{32L_F^2}  + \frac{\mu\epsilon_{k}}{3456L_F^2} \\&\leq 
   \frac{\mu\epsilon_k}{16L_F^2}.
\end{align*}
Then by Eq.~(\ref{eq:mukappa}), we have $\|\z_k -\nabla F_{\mu}(\x_k) \|^2\leq 4L_F^2\|\varkappa_{k} \|^2  \leq \mu\epsilon_k/4$.
%where the inequality (a) is due to $L_F\geq 1$ and $\mu<1$ and the equality (b) is due to $\epsilon_k = \epsilon_{k-1}/4$.
Invoking Lemma~\ref{lem:dist_lemma_1}, at $(k-1)$-th stage ($k> 1$) we have 
\begin{align*}
    & \E[\dist (0, \hat{\partial}\bar F_\mu(\x_k))^2]\\
       &\leq \frac{2+40L_{F_\mu}\eta_{k-1}}{T_{k-1}}\sum\limits_{t=1}^{T_{k-1}}\E[\|\z_t -\nabla F_{\mu}(\x_t)\|^2] + \frac{2\E[\Delta^\mu_{k-1}]}{\eta_{k-1} T_{k-1}} + \frac{40L_{F_\mu}\E[\Delta^\mu_{k-1}]}{T_{k-1}}\\
        &\leq \frac{(2+40L_{F_\mu}\eta_{k-1})\mu\epsilon_{k-1}}{2}+ \frac{2\epsilon_{k-1}}{\eta_{k-1} T_{k-1}} + \frac{40L_{F_\mu}\epsilon_{k-1}}{T_{k-1}}
         \\&\leq \frac{773\mu\epsilon_{k}}{768}+ \frac{\mu\epsilon_{k}}{2} + \frac{40L_{F_\mu}\mu\epsilon_{k-1}}{73728L_F^4}\\&\leq 2\mu\epsilon_k,
\end{align*}
where the second inequality is due to $L_{F_\mu}\eta_{k-1}\leq(3/2)L_F\eta_{k-1}\leq1/3072$.
% \begin{align*}
%     & \E[\dist (0, \hat{\partial}\bar F_\mu(\x_k))^2]\\
%       &\leq \frac{197\eta_{k-1}}{T_{k-1}}\sum\limits_{t=1}^{T_{k-1}}\E[\|\z_t -\nabla F(\x_t)\|^2] + \frac{2\E[\Delta^\mu_{k-1}]}{\eta_{k-1} T_{k-1}} + \frac{40L_{F_\mu}\E[\Delta^\mu_{k-1}]}{T_{k-1}}\\
%         &\leq \frac{(2+40L_{F_\mu}\eta_{k-1})\mu\epsilon_{k}}{2}+ \frac{2\E[\Delta^\mu_{k-1}]}{\eta_{k-1} T_{k-1}} + \frac{40L_{F_\mu}\E[\Delta^\mu_{k-1}]}{T_{k-1}}\\
%      &\overset{\eta_{k-1} L_F \leq 1/1300}{\leq} \frac{21}{10T_{k-1}}\sum\limits_{t=1}^{T_{k-1}}\E[\|\z_t -\nabla F_{\mu}(\x_t)\|^2] + \frac{2\E[\Delta^\mu_{k-1}]}{\eta_{k-1} T_{k-1}} +
%      \frac{20L_F\E[\Delta^\mu_{k-1}]}{T_{k-1}}\\
%      &= \frac{21\E[\|\z_{k} - \nabla F_{\mu}(\x_{k}) \|^2]}{10} + \frac{2\E[\Delta^\mu_{k-1}]}{\eta_{k-1} T_{k-1}} +
%      \frac{20L_F\E[\Delta^\mu_{k-1}]}{T_{k-1}} \\
%      &\overset{Lemma~\ref{lem:stage_variance}}{\leq} \frac{21\mu\epsilon_{k}}{20} + \frac{2\mu\epsilon_{k-1}}{11} + \frac{\mu^2\epsilon_k\epsilon_{k-1}}{768L_F^3}\\
%      &\overset{\epsilon_k = \epsilon_{k-1}/4}{\leq}\frac{21\mu\epsilon_k}{20} + \frac{2\mu\epsilon_k}{5} +\frac{\mu\epsilon_k}{768L_F^2}\leq 2\mu\epsilon_k
% \end{align*}

Since $F_{\mu}(\x_k)\leq \bar F_{\mu}(\x_k)$ and $ \inf_{\x \in\X}F_{\mu}(\x)=\inf_{\x \in\X}  \bar F{\mu}(\x)$, applying Lemma~\ref{le:kl}  we have
\begin{equation*}
    \begin{aligned}
      \E[F_{\mu}(\x_k) - \inf\limits_{\x \in\X}F_{\mu}(\x)]\leq \E[ \bar F_{\mu}(\x_k) -\inf\limits_{\x \in\X}  \bar F_{\mu}(\x)]\leq \frac{1}{2\mu}\E[\dist (0, \hat{\partial} \bar F_{\mu}(\x_k))^2]\leq \frac{2\mu\epsilon_k}{2\mu} = \epsilon_k.
    \end{aligned}
\end{equation*}
This complete the proof of this Lemma.
\end{proof}

\subsection{Proof of Theorem~\ref{thm:storm-KL}}
\begin{proof}
Invoking Lemma~\ref{lem:storm-var}, then after $K = \mathcal O(\log_2(\epsilon_1/\epsilon))$ stages, we have
\begin{align*}
    \E[F_{\mu}(\x_K) - \inf\limits_{\x \in\X}F_{\mu}(\x)]\leq \epsilon_K = \frac{\epsilon_1}{2^{K-1}} = \epsilon.
\end{align*}
Since $\sum_{k=1}^K 2^k=\mathcal O (1/\epsilon)$, the overall oracle complexity is
\begin{equation*}
\begin{aligned}
 \sum\limits_{k=1}^{K}T_k +\frac{48L_F^2}{\mu\epsilon_1}
&\leq\mathcal{O}\left(\sum\limits_{k=2}^K\max\left(\frac{1}{\mu\epsilon_k},\frac{1}{\mu^{3/2}\sqrt{\epsilon_k}}\right)\right)+\frac{48L_F^2}{\mu\epsilon_1}\\
&\leq\mathcal{O}\left(\sum\limits_{k=2}^K\max\left(\frac{2^k}{\mu},\frac{\sqrt{2}^k}{\mu^{3/2}\sqrt{\epsilon_k}}\right)\right) +\frac{48L_F^2}{\mu\epsilon_1} \\
&\leq\mathcal{O}\left(\max\left(\frac{1}{\mu\epsilon},\frac{1}{\mu^{3/2}\sqrt{\epsilon}}\right)\right).
\end{aligned}
\end{equation*}
This complete the proof.
\end{proof}
\section{Derivation of the Compositional Formulation}\label{sec:derivation}
Recall the original KL-constrained DRO  problem: \begin{align*}
\min_{\w \in \mathcal W} \max_{\{\p\in\Delta_n:D(\p, \mathbf 1/n)\leq  \rho\} }\sum_{i=1}^np_i\ell_i(\mathbf w) - \lambda_0 D(\p, 1/n),
\end{align*}
where $\Delta_n=\{\mathbf p\in\mathbb R^n: \sum_{i=1}^n p_i=1, 0\leq p_i\leq 1\}$, $D(\p, 1/n)$ is the KL divergence and $\lambda_0$ is a small positive constant.

In order to tackle this problem, let us first consider the robust loss 
\begin{align*}
 \max_{\{\p\in\Delta_n:D(\p, 1/n)\leq \rho\}}\sum_{i=1}^n p_i \ell_i(\w) -  \lambda_0 D(\p, 1/n).
 \end{align*}
And then we invoke the dual variable $\lambda$ to transform this primal problem to the following form
 \begin{align*}
 \max_{\p\in\Delta_n}\min_{\bar\lambda\geq 0}\sum_{i=1}^n p_i \ell_i(\w) - \bar\lambda (D(\p, 1/n) -\rho) - \lambda_0 D(\p, 1/n).
 \end{align*}
Since this problem is concave in term of $\p$ given $\w$, by strong duality theorem, we have
\begin{align*}
 &\max_{\p\in\Delta_n}\min_{\bar\lambda\geq 0}\sum_{i=1}^n p_i \ell_i(\w) - \bar\lambda (D(\p, 1/n) -\rho) - \lambda_0 D(\p, 1/n)\\
 &= \min_{\bar\lambda\geq 0}\max_{\p\in\Delta_n}\sum_{i=1}^n p_i \ell_i(\w) - \bar\lambda (D(\p, 1/n) -\rho) - \lambda_0 D(\p, 1/n).
 \end{align*}
 Let $\lambda=\bar\lambda+\lambda_0$, we have
 \begin{align*}
 &\min_{\bar\lambda\geq 0}\max_{\p\in\Delta_n}\sum_{i=1}^n p_i \ell_i(\w) - \bar\lambda (D(\p, 1/n) -\rho) - \lambda_0 D(\p, 1/n)\\
    &= \min_{\lambda\geq \lambda_0}\max_{\p\in\Delta_n}\sum_{i=1}^n p_i \ell_i(\w) - \lambda (D(\p, 1/n) -\rho) - \lambda_0 \rho.
 \end{align*}
Then the original problem is equivalent to the following problem 
\begin{equation*}
%\label{eqn:prob11}
    \begin{aligned}
     \min_{\w \in \mathcal W} \min\limits_{\lambda \geq \lambda_0}\max_{\p\in\Delta_n}\sum_{i=1}^n p_i \ell_i(\w) - \lambda (D(\p, 1/n) -\rho) - \lambda_0 \rho,
    \end{aligned}
\end{equation*}
Next we  fix $\x=(\w^\top,\lambda)^\top$  and derive an optimal solution $\p^*(\mathbf x)$ which depends on $\x$ and solves the inner maximization problem.  We consider the  following problem 
\begin{align*}
\min_{\p\in\Delta_n}-\sum_{i=1}^n p_i \ell_i(\w) + \lambda D(\p, 1/n).
\end{align*}
which has the same optimal solution $\p^*(\mathbf x)$ with our problem.

 There are three constraints to handle, i.e., $p_i\geq 0, \forall i$ and $p_i\leq 1, \forall i$ and $\sum_{i=1}^n p_i=1$.  Note that the constraint $p_i\geq 0$ is enforced by the term $p_i\log(p_i)$, otherwise the above objective will become infinity. As a result, the constraint $p_i<1$ is automatically satisfied due to $\sum_{i=1}^np_i=1$ and $p_i\geq0$. Hence, we only need to explicitly tackle the constraint $\sum_{i=1}^n p_i=1$. To this end, we define the following Lagrangian function
\begin{align*}
 L_{\mathbf x}(\mathbf p, \mu) = -\sum_{i=1}^n p_i\ell_i(\mathbf w)  +  \lambda\left(\log n + \sum_{i=1}^n  p_i \log (p_i)\right) + \mu(\sum_{i=1}^n  p_i - 1), 
\end{align*}
where $\mu$ is the Lagrangian multiplier for the constraint $\sum_{i=1}^n  p_i=1$. The optimal solutions satisfy the KKT conditions:  
\begin{align*}
    &- \ell_i(\mathbf w)  +  \lambda \left(\log (p^*_i(\mathbf x)) + 1\right) + \mu  = 0 \text{ and }\sum_{i=1}^n  p^*_i(\mathbf x)=1.
\end{align*}
From the first equation, we can derive $p^*_i(\mathbf x) \propto \exp(\ell_i(\mathbf w)/\lambda)$. Due to the second equation, we can conclude that $p^*_i(\mathbf x) = \frac{\exp(\ell_i(\mathbf w)/\lambda)}{\sum_{i=1}^n  \exp(\ell_i(\mathbf w)/\lambda)}$. Plugging this optimal $\p^*(\w)$ into the inner maximization problem, we have 
\begin{align*}
     \sum_{i=1}^np^*_i(\mathbf x)\ell_i(\mathbf w)  - \lambda\left(\log n + \sum_{i=1}^n  p_i^*(\mathbf w) \log (p_i^*(\mathbf w)) \right)  = \lambda \log \left(\frac{1}{n}\sum_{i=1}^n  \exp\left(\frac{\ell_i(\mathbf w)}{\lambda}\right)\right) , 
\end{align*}
% Therefore, combining Lemma~\ref{lem:lambda_upper} we get the following equivalent problem to the original problem
% \begin{equation*}
%     \begin{aligned}
%       %\min_{\w \in \mathcal W}\min\limits_{\lambda_0 \leq \lambda\leq \tilde\lambda}\underbrace{\lambda %\log\frac{1}{n}\sum_{i=1}^n  \exp(\frac{\ell_i(\w)}{\lambda}) +\lambda \rho}_{F(\w,\lambda)}.
%       \min_{\w \in \mathcal W}\min\limits_{\lambda_0 \leq \lambda\leq \tilde\lambda}\lambda \log\left(\frac{1}{n}\sum_{i=1}^n \exp\left(\frac{\ell_i(\w)}{\lambda}\right)\right) +\lambda \rho.
%     \end{aligned}
% \end{equation*}
% which is Eq.~(\ref{eqn:prob1}) in the paper. 
Therefore, we get the following equivalent problem to the original problem
\begin{equation*}
    \begin{aligned}
      %\min_{\w \in \mathcal W}\min\limits_{\lambda_0 \leq \lambda\leq \tilde\lambda}\underbrace{\lambda %\log\frac{1}{n}\sum_{i=1}^n  \exp(\frac{\ell_i(\w)}{\lambda}) +\lambda \rho}_{F(\w,\lambda)}.
      \min_{\w \in \mathcal W}\min\limits_{\lambda \geq \lambda_0}\lambda \log\left(\frac{1}{n}\sum_{i=1}^n \exp\left(\frac{\ell_i(\w)}{\lambda}\right)\right) +\lambda \rho.
    \end{aligned}
\end{equation*}
which is Eq.~(\ref{eq:cdro3}) in the paper. 

\end{document}